\documentclass{article}

\usepackage{a4}
\usepackage{algorithm}
\usepackage{algorithmicx}
\usepackage{algpseudocode}
\usepackage{amsfonts}
\usepackage{amsopn}
\usepackage{amssymb}
\usepackage{amsthm}
\usepackage{bbm}
\usepackage{blindtext}
\usepackage[font=small,labelfont=bf,tableposition=top]{caption}
\usepackage{color}
\usepackage{booktabs}
\usepackage{enumitem}
\usepackage[T1]{fontenc}
\usepackage{glossaries}
\usepackage{pifont}
\usepackage{siunitx}
\usepackage{subfig}
\usepackage{tikz}
\usetikzlibrary{positioning}
\usetikzlibrary{matrix,chains,positioning,decorations.pathreplacing,arrows}
\usepackage{todonotes}
\usepackage{multirow}
\usepackage{multicol}
\usepackage{neuralnetwork}
\usepackage{wrapfig}
\usepackage{xcolor}

\usepackage{amsmath}
\usepackage[colorlinks]{hyperref}
\hypersetup{
	colorlinks,
	linkcolor={red!80!black},
	citecolor={blue!80!black},
	urlcolor={blue!80!black}
}
\usepackage{cleveref}
\usepackage{authblk}
\newlist{todolist}{itemize}{2}
\setlist[todolist]{label=$\square$}
\newcommand{\bea}{
	\begin{eqnarray}
		}
		\newcommand{\eea}{
	\end{eqnarray}
}
\newcommand{\<}{\langle}
\renewcommand{\>}{\rangle}

\def\eps{{\varepsilon}}

\def\<{\langle}
\def\>{\rangle}

\def\argmin{{\rm arg\,min}}

\newcommand{\GIFT}{GIFT\xspace}
\newcommand{\finetuning}{fine-tuning\xspace}
\newcommand{\R}{\mathbbm{R}}
\newcommand{\X}{Z_1}
\newcommand{\Y}{Z_2}
\newcommand{\Z}{Z_3}
\newcommand{\ENO}{\mathcal{J}}
\newcommand{\true}{t}
\newcommand{\ft}{f}
\newcommand{\errorvar}{R}
\newcommand{\sd}{s}

\newcommand{\interior}[1]{%
	{\kern0pt#1}^{\mathrm{o}}%
}

\newacronym{ADC}{ADC}{Analog--Digital Conversion}
\newacronym{AI}{AI}{Artificial Intelligence}
\newacronym{AWGN}{AWGN}{Additive White Gaussian Noise}
\newacronym{AO}{AO}{All-Optical}
\newacronym{CMOS}{CMOS}{Complementary Metal Oxide Semiconductor}
\newacronym{CNN}{CNN}{Convolutional Neural Network}
\newacronym{DAC}{DAC}{Digital--Analog Conversion}
\newacronym{DFA}{DFA}{Direct Feedback Alignment}
\newacronym{EDA}{EDA}{Electronic Design Automation}
\newacronym{E/O/E}{E/O/E}{Electrical/Optical/Electrical}
\newacronym{FB}{FB}{Feedforward--Backpropagation}
\newacronym{FDM}{FDM}{Finite Difference Method}
\newacronym{GA}{GA}{Genetic Algorithm}
\newacronym{GD}{GD}{Gradient Descent}
\newacronym{IID}{IID}{Independent and Identically Distributed}
\newacronym{LLN}{LNN}{Law of Large Numbers}
\newacronym{MAC}{MAC}{Multiply--Accumulate operation}
\newacronym{ML}{ML}{Machine Learning}
\newacronym{MLP}{MLP}{Multi-Layer Perceptrons}
\newacronym{MNIST}{MNIST}{Modified National Institute of Standards and Technology}
\newacronym{MRM}{MRM}{Micro-Ring Modulator}
\newacronym{MZI}{MZI}{Mach--Zehnder Interferometer}
\newacronym{MSE}{MSE}{Mean Squared Error}
\newacronym{NAS}{NAS}{Neural Architecture Search}
\newacronym{NN}{NN}{Neural Network}
\newacronym{ODE}{ODE}{Ordinary Differential Equation}
\newacronym{ONN}{ONN}{Optical Neural Network}
\newacronym{O/E/O}{O/E/O}{Optical/Electrical/Optical}
\newacronym{PDE}{PDE}{Partial Differential Equation}
\newacronym{PDK}{PDK}{Process Design Kit}
\newacronym{PIC}{PIC}{Photonic Integrated Circuit}
\newacronym{PL}{PL}{Polyak--\L{}ojasiewicz}
\newacronym{PSO}{PSO}{Particle Swarm Optimization}
\newacronym{SGD}{SGD}{Stochastic Gradient Descent}
\newacronym{SNN}{SNN}{Spiking Neural Network}
\newacronym{SNR}{SNR}{Signal-to-Noise Ratio}
\newacronym{SOA}{SOA}{Semiconductor Optical Amplifier}
\newacronym{ReLU}{ReLU}{Rectified Linear Unit}
\newacronym{TIA}{TIA}{Transimpedance Amplifier}
\newacronym{WDM}{WDM}{Wavelength-Division Multiplexing}

\newtheorem*{assumptionsNoNum}{Assumptions}

\newtheorem{example}{Example}
\newtheorem{lemma}{Lemma}

\newtheorem{proposition}{Proposition}
\newtheorem{theorem}{Theorem}
\newtheorem{result}{Result}

\newlist{mytodolist}{itemize}{2}
\setlist[mytodolist]{label=$\square$}

\usepackage{dsfont} 

\newcommand{\vect}[1]{ \boldsymbol{#1} }

\newcommand{\naturalNumbersPlus}{ \mathbb{N}_{+} }

\newcommand{\QuodEratDemonstrandum}{\hfill \ensuremath{\Box}}

\def\eqcom#1{\overset{\textnormal{(#1)}}}

\def\({{\Bigl(}}
        \def\){{\Bigr)}}

    \newcommand{\ea}{\end{array}}

\newcommand{\xdeleted}[1]{\deleted{}} 

\crefname{equation}{}{}

\IfFileExists{includeonly.tex}{%
    \includeonly{%
    abstract,
    introduction,
    related_literature,
    preliminaries,
    analysis_of_gift,
    numerical_results,
    conclusion,
    acknowledgments,
    technical_details,
    additional_plots,
    misspec_bound,
} 
}{%
    \typeout{File includeonly.tex not found. Skipping inclusion.} 
}

\title{\emph{In situ} \finetuning of \emph{in silico} trained\\ Optical Neural Networks}
\author[1,2]{Gianluca Kosmella}
\author[1]{Ripalta Stabile}
\author[2]{Jaron Sanders}

\affil[1]{Department of Electrical Engineering, Eindhoven University of Technology, The Netherlands}
\affil[2]{Department of Mathematics \& Computer Science, Eindhoven University of Technology, The Netherlands}
\date{}

\begin{document}

\maketitle

\begin{abstract}
    \glspl{ONN} promise significant advantages over traditional electronic neural networks, including ultrafast computation, high bandwidth, and low energy consumption, by leveraging the intrinsic capabilities of photonics.
    However, training \glspl{ONN} poses unique challenges, notably the reliance on simplified \emph{in silico} models whose trained parameters must subsequently be mapped to physical hardware.
    This process often introduces inaccuracies due to discrepancies between the idealized digital model and the physical \gls{ONN} implementation, particularly stemming from noise and fabrication imperfections.

    In this paper, we analyze how noise misspecification during \emph{in silico} training impacts \gls{ONN} performance and we introduce \emph{Gradient-Informed Fine-Tuning} (\GIFT), a lightweight algorithm designed to mitigate this performance degradation.
    \GIFT uses gradient information derived from the noise structure of the \gls{ONN} to adapt pretrained parameters directly \emph{in situ}, without requiring expensive retraining or complex experimental setups.
    \GIFT comes with formal conditions under which it improves \gls{ONN} performance.

    We also demonstrate the effectiveness of \GIFT via simulation on a five-layer feed forward \gls{ONN} trained on the \gls{MNIST} digit classification task.
    \GIFT achieves up to $28\%$ relative accuracy improvement compared to the baseline performance under noise misspecification, without resorting to costly retraining.
    Overall, \GIFT provides a practical solution for bridging the gap between simplified digital models and real-world \gls{ONN} implementations.
\end{abstract}

\glsresetall
\section{Introduction}
\label{sec:introduction}

\gls{AI} has become a transformative technology across a wide range of fields ranging from healthcare \cite{mennella2024healthcareai} and finance \cite{cao2022aifinance,dakalbab2024financeartificial} to scientific discovery \cite{wang2023scientific} and autonomous systems \cite{reda2024autonomousdriving}.
Deep learning, in particular, has enabled breakthroughs in tasks such as image recognition \cite{archana2024deep}, natural language processing \cite{otter2020survey}, and reinforcement learning \cite{li2017deep}.
As models continue to grow in size and complexity, their capabilities improve significantly, but this progress comes at the cost of increasing computational demands.

Empirically, \gls{NN} performance has been seen to follow a power law relationship that depends on the dataset size and model scale \cite{hestness2017deep, kaplan2020scaling, rosenfeld2019constructive, henighan2020scaling}.
Consequently, improving performance requires even larger datasets and more complex models.
These, in turn, demand significantly greater computational resources.
However, advancements in compute performance are plateauing, driving a surge of research into alternative computing hardware solutions.
\glspl{ONN} offer the promise of ultra-fast and energy-efficient computing due to the high bandwidth and low latency of light \cite{cheng2020silicon, miscuglio2020photonic, shastri2021photonics, ning2024photonic, mcmahon2023physics}.

Despite these advantages, \glspl{ONN} face several unique challenges that complicate their deployment.
Noise, in particular, is a fundamental issue, arising from sources such as signal encoding, transmission, and processing, which can degrade network performance.
Additionally, device-level imperfections, such as phase errors in \gls{MZI} meshes, thermal crosstalk, and fabrication inconsistencies, introduce further variability that affects system performance \cite{oikonomou2022robust, passalis2021training, kosmella2023higher, shi2021noise, gu2020roq}.
The high-precision and high-speed requirements of \glspl{DAC} and \glspl{ADC} add significant overhead in mixed analog--digital architectures \cite{kirtas2022quantization, gu2020roq}.
Furthermore, state--of--the--art backpropagation algorithms would require both a forward pass, which computes an activation function, and a backward pass, which computes its derivative.
However, nonreciprocal optical neurons, which would allow for such passes, are not yet feasible in the optical domain.
Moreover, the lack of efficient optical on-chip memory presents a further challenge, as the backward pass in backpropagation relies on the values previously generated during the forward pass.

Given these limitations, \emph{in silico} training is widely adopted for \glspl{ONN}.
This approach involves training a digital model that simulates the physical hardware, allowing backpropagation to be performed computationally before mapping the optimized parameters to the actual optical hardware \cite{shi2019deep, ashtiani2022chip, wright2022deep}, where only inference can be executed.
Effective \emph{in silico} training requires accurate modeling of the \gls{ONN}, including the activation functions (e.g.\ sigmoidal \cite{mourgias2019all} or sinusoidal \cite{passalis2019training}), initialization schemes to avoid saturation \cite{passalis2020initializing}, and compensation for quantization and noise effects in \glspl{PIC} \cite{oikonomou2022robust, passalis2021training, kirtas2022learning}.
Nevertheless, \emph{in silico} models are prone to model misspecification.
Even small discrepancies between the simulated model and the physical hardware can accumulate and become amplified as they propagate through the network’s layers \cite{de2019noise}.

To address this, previous work has focused on adapting training algorithms to the unique properties of \glspl{ONN} \cite{mourgias2019all, passalis2019training, gu2020roq, passalis2020initializing, oikonomou2022robust, passalis2021training, kirtas2022learning}. Furthermore, hybrid \emph{in situ} \finetuning approaches have been proposed to mitigate model misspecification by performing forward passes directly on the \gls{ONN} while retaining backpropagation in a digital model \cite{zhou2021large, wright2022deep, spall2022hybrid}.
While this method reduces errors from model inaccuracies, it requires a complex interface to integrate feedback from physical hardware into the training process.
\Cref{fig:training_approaches}(left, middle) illustrate the \emph{in silico} and hybrid training approaches, respectively.

In this paper, we propose a novel \emph{in situ} \finetuning method to mitigate errors from model misspecification in \emph{in silico}-trained \glspl{ONN}, nicknamed \GIFT, which is largely agnostic to the specific model, ensuring reliable performance in real-world applications.
Unlike previous approaches, our method does not require incorporating hardware feedback into the digital model.
Instead, we focus on directly \finetuning the weights on the physical hardware to correct for discrepancies between digital model and physical hardware, as illustrated in \Cref{fig:training_approaches}(right).
This approach is motivated by the observation that parameters learned through \emph{in silico} training are often close to the optimal parameters for the physical system, removing the need for a sophisticated retraining algorithm.
Using our method, fair minimizers are first identified through standard \emph{in silico} training, and then fine-tuned towards better minimizers directly on the \gls{ONN} platform, \emph{in situ}.
To the best of our knowledge, this form of device-level \finetuning is novel and offers a practical, computationally efficient solution for real-world \gls{ONN} deployment.

\begin{figure}[h]
	\centering
	\includegraphics[width=.99\linewidth]{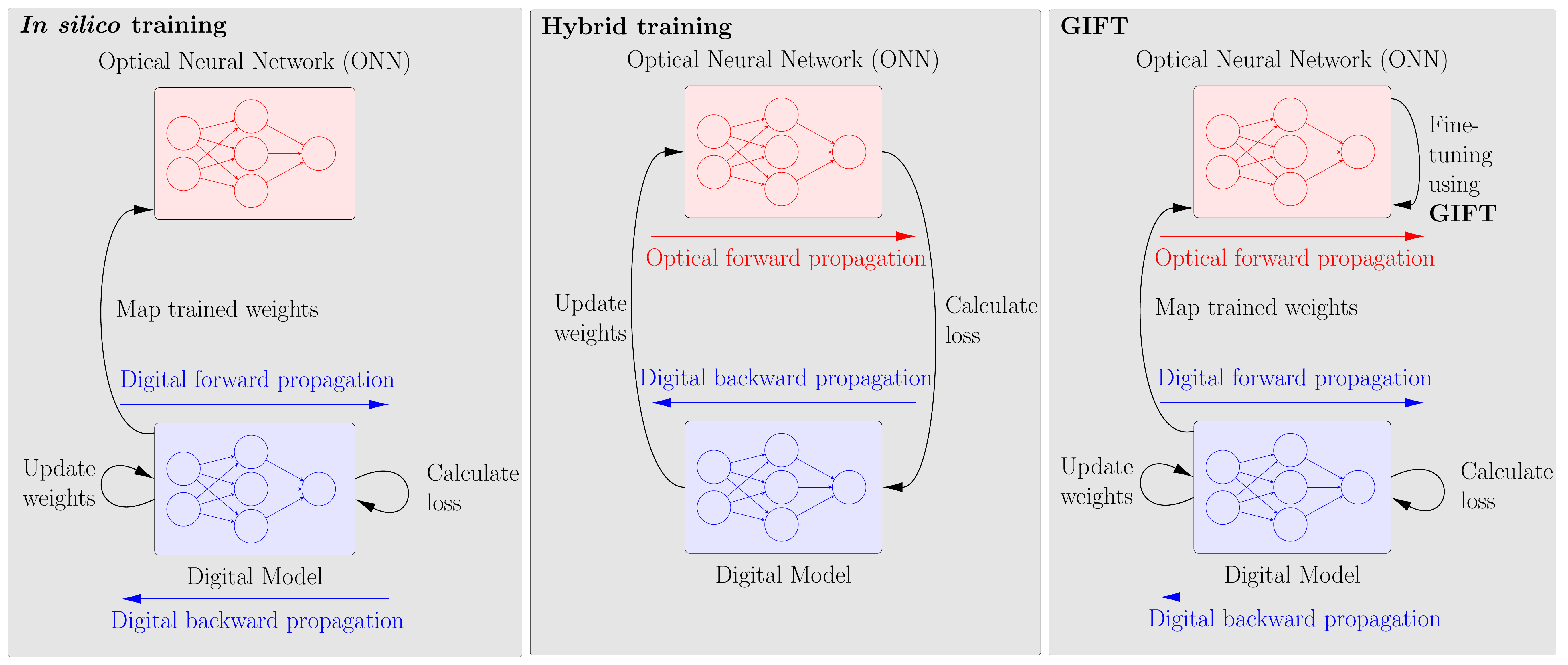}
	\caption{Comparison of training approaches. \emph{Left}: Full \emph{in silico} training, where both forward and backward passes occur in a digital model. \emph{Middle}: Hybrid training, where optical forward passes are backpropagated through a digital model. \emph{Right}: Gradient-Informed Fine-Tuning (\GIFT), which fine-tunes an \emph{in silico} trained model using optical forward passes, without requiring the complex experimental setup of hybrid training.}
	\label{fig:training_approaches}
\end{figure}

\subsection{Summary of results}

A physical \gls{ONN} is inherently noisy and thus, for given weights $w$ and some input $x$, the \gls{ONN}'s output $\Phi^\mathrm{ONN}(x,w)$ is a random variable with certain distribution.
Ideally, a model $M$ of the physical \gls{ONN} incorporates a noise distribution $\mathcal{N}$ such that for all $x$ and $w$ and some realization $\mathbf{N} \sim \mathcal{N}$ of the modeled noise, its output $M(x,w, \, \cdot \, )$ has the exact same distribution as $\Phi^\mathrm{ONN}(x,w)$.
In this ideal scenario, an optimal $w^\ast$ for the model $M$ will also be optimal for the \gls{ONN} $\Phi^\mathrm{ONN}$.
However, a model is a mere approximation of reality and there is typically model misspecification.

In \Cref{sec:Modeling-ONNs}, we construct a function $M_\sd$ that can at the same time describe a noisy \emph{in silico} \gls{NN} and model an \emph{in situ} \gls{ONN}.
There, the presumed noise distribution $\mathcal{N}_\sd$ includes a \emph{noise level parameter} $\sd \in (0, \infty)$ that effectively determines the variance of the present \gls{AWGN}.
Our construction is such because we care about the scenario in which a practitioner has an \emph{estimate noise level} $\sd_0 \in (0, \infty)$ of some \emph{underlying hidden, unknown and true noise level} $\sd_t \in (0, \infty)$, and such that $\sd_0 \approx \sd_t$ and $\sd_0 \neq \sd_t$.

\subsubsection{Result 1: Characterization of the minimizers under \texorpdfstring{$L_2$}{L2}-loss}
\label{sec:Identification-of-the-objective-function-and-of-minimizers}

We assume that a practitioner is training $M_\sd$ using a projected \gls{SGD} algorithm on $L_2$-loss \emph{in silico}.
This means specifically that the practitioner is iterating
\begin{align}
	w^{\{k+1\}}
	& =
	p_H
	\big(
	w^{\{k\}}
	-
	\eps_k
	\nabla_w
	\bigl(
	y^{\{k\}}
	-
	M_\sd( {x}^{\{k\}}, w^{\{k\}}, \mathbf{N}^{\{k\}} )
	\bigr)^2
	\big)
	\label{eqn:projected-sgd}
\end{align}
\emph{in silico}.
Here,
the $w^{\{0\}}$ denote initial weights, $\mathbf{N}^{\{k\}} \sim \mathcal{N}_\sd$ noise realizations, $({x}^{\{k\}},y^{\{k\}}) \sim \mu$ data samples, and $\eps_k > 0$ step sizes.

For $s \in (0, \infty)$, let $\mathcal{N}_s$ refer specifically to the noise distribution implied by \Cref{eqn:Normally-distributed-noise-matrices} in \Cref{sec:Modeling-ONNs},
and $\mu$ a data distribution on $\R^{d_0} \times \R^{d_L}$.
Furthermore, let
\begin{align}
	\ENO_s(w)
	:=
	\int
	\int
	(
	y
	-
	M_\sd({x} , w, \mathbf{N})
	)^2
	\,
	\mathrm{d}\mathcal{N}_s(\mathbf{N})
	\,
	\mathrm{d}\mu({x},y)
	\label{eqn:objective-function-as-a-function-of-s}
\end{align}
be the \emph{objective function at level $\sd$}.
The fact that \Cref{eqn:objective-function-as-a-function-of-s} really is the objective function is implied by our first result: the exact description of the limiting behavior of \Cref{eqn:projected-sgd} in \Cref{thm:stochastic-approximation-theorem-limit-trajectories,thm:stochastic-approximation-theorem-limitpoints} in \Cref{sec:Modeling-ONNs}.
Informally stated, for one for example, \Cref{thm:stochastic-approximation-theorem-limitpoints} implies:

\begin{result}[Informal]
	For any $\sd \in (0,\infty)$, under typical assumptions, the limit of \Cref{eqn:projected-sgd} converges to a point in the set
	$
	\{
	w
	:
	\nabla_w \ENO_\sd
	=
	0
	\}
	.
	$
\end{result}

The proofs of \Cref{thm:stochastic-approximation-theorem-limit-trajectories,thm:stochastic-approximation-theorem-limitpoints} can be found in \Cref{appendix:proof-ODE}.
These proofs are modifications of proofs in \cite{senen2020almost}, and ultimately rely on the so-called \gls{ODE} method \cite{kusher2003stochastic}.
The proof parts that deal with the \gls{AWGN} are new.

Now that we have identified exactly what the objective function is, and what the limit points of \Cref{eqn:projected-sgd} are, we can tackle the problem of model misspecification.

\subsubsection{Result 2: Minimizers from a misspecified model can be improved}
\label{sec:There-is-scope-for-improving-the-minimizer}

Assume now that $\sd_0 \neq \sd_t$, i.e., that there is model misspecification.
Our second result, \Cref{thm:FT-is-needed}, gives conditions under which there is in fact scope for improvement.
This is made explicit by the strict inequality in \Cref{eqn:ENOt-wt-is-strictly-less-than-ENOt-w0}.
Its proof is given in \Cref{sec:Proof-that-finetuning-is-needed}:

\begin{theorem}
	\label{thm:FT-is-needed}
	
	Let $s_0, s_t \in (0, \infty)$.
	Assume that
	\begin{equation}
		w_0 \in \{ w : \nabla_w \ENO_{\sd_0}(w) = 0 \}
		\neq
		\emptyset
		\quad
		\textnormal{and}
		\quad
		w_t \in \{ w : \nabla_w \ENO_{\sd_t}(w) = 0 \}
		\neq
		\emptyset
		.
		\label{ass:Nonempty-minimizer-sets}
	\end{equation}
	
	The following then holds:
	if
	\begin{equation}
		\forall
		\zeta \in ( s_0 \wedge \sd_\true, s_0 \vee \sd_\true )
		,
		\quad
		0
		<
		|\sd_\true-\sd_0|
		<
		\frac{
			\|
			\frac{\partial}{\partial \sd}
			\nabla_w \ENO_{\sd_0}(w_0)\|_2
		}{
			\frac{1}{2}
			\Bigl|\Bigl(
			\frac{\partial^2}{\partial \sd^2}
			\nabla_w \ENO_{\zeta}(w_0)
			\Bigr)^T
			\Bigl(\frac{\partial}{\partial \sd}
			\nabla_w \ENO_{\sd_0}(w_0)\Bigr)\Bigr|
		}
		,
		\label{ass:Laplacian-properties}
	\end{equation}
	then $w_0 \neq w_\true$.
	Moreover,
	\begin{align}
		\ENO_{\sd_t}(w_\true)
		<
		\ENO_{\sd_t}(w_0).
		\label{eqn:ENOt-wt-is-strictly-less-than-ENOt-w0}
	\end{align}
\end{theorem}

\subsubsection{Result 3: \GIFT improves minimizers from a misspecified model}
\label{sec:GIFT-improves-misspecified-minimizers}

Our third result is \Cref{alg:GIFT}, nicknamed \GIFT.
Importantly, \GIFT is designed to be implementable \emph{in situ}.
Specifically, for $\ell \in [L]$, it takes \emph{in silico} estimates of opportune directions,
\begin{align}
	&
	D^{[0]}_{W^{(\ell)}}(K_1,K_2)
	\label{eq:sample_D0_Wi}
	\\          &
	=
	\frac{1}{K_1} \sum_{k=1}^{K_1}\frac{1}{K_2} \sum_{m=1}^{K_2}
	\Bigl[
	\Bigl(
	\sum_{\alpha\in \mathcal{S}} \big(\sd_0^{-2}\big(N^{\alpha,\{m_k\}}\big)^T N^{\alpha,\{m_k\}} - d_{f(\alpha)}\big)
	\Bigr)
	\errorvar^{(\ell),\{k\}}
	\bigl(
	A^{(i-1),\{k\}}
	\bigr)^T
	\Bigr]
	\nonumber
\end{align}
and
\begin{align}
	&
	D^{[0]}_{b^{(\ell)}}(K_1,K_2)
	\nonumber \\          &
	=
	\frac{1}{K_1} \sum_{k=1}^{K_1}\frac{1}{K_2} \sum_{m=1}^{K_2}
	\Bigl[
	\Bigl(
	\sum_{\alpha\in \mathcal{S}} \big(\sd_0^{-2}\big(N^{\alpha,\{m_k\}}\big)^T N^{\alpha,\{m_k\}} - d_{f(\alpha)}\big)
	\Bigr)
	\errorvar^{(i),\{k\}}
	\Bigr]
	\label{eq:sample_D0_bi}
\end{align}
(see \Cref{gift:Estimate-of-D0}), along which an \emph{in situ} search is then performed (see \Crefrange{gift:Directional-search}{gift:Directional-evaluation}).
The candidate point along the opportune direction that scores best according to the \emph{in situ} evaluation function
\begin{equation}
	\mathrm{Eval}_w(K_1,K_2)
	=
	\frac{1}{K_1}
	\sum_{k=1}^{K_1}
	\frac{1}{K_2}
	\sum_{m=1}^{K_2}
	\bigl(
	y^{ \{k\} } - \Phi^{\text{ONN}}(x^{ \{k\} }, w )
	\bigr)^2
	\label{def:Eval}
\end{equation}
is returned by the algorithm.

\begin{algorithm}[hbtp]
	\caption{Gradient-Informed Fine-Tuning (\GIFT)}
	\label{alg:GIFT}
	
	\begin{algorithmic}[1]
		\Require An estimate $\sd_0 \in (0, \infty)$ of $s_t$, parameters $\eta > 0, K_1, K_2 \in \mathbb{N}_+$, and a performance evaluation function $\mathrm{Eval}_w(K_1,K_2)$ evaluable \emph{in situ}
		\Statex \textbf{Initialization:}
		\State Calculate an approximate minimizer $w_0 \in \{ w : \nabla_w \ENO_{\sd_0}(w) = 0 \}$, e.g.\ via \Cref{eqn:projected-sgd}, \emph{in silico}
		\State Calculate an estimate of $(\partial / \partial_\sd) \nabla_w \mathcal{J}_s(w)$, say $D^{[0]}(K_1, K_2)$, \emph{in silico} \label{gift:Estimate-of-D0}
		\Statex \textbf{Directional search:}
		\State Initialize $w^{[0]} \gets w_0$, $i \gets 0$
		\State Evaluate $L^{[0]} \gets \mathrm{Eval}_{K_1,K_2}(w^{[0]})$ \emph{in situ}
		\Repeat \label{gift:Directional-search}
		\State Calculate $w^{[i+1,\pm]} \gets w^{[i]} \pm \eta D^{[0]}$
		\State Evaluate $L^{[i+1,\pm]} \gets \mathrm{Eval}_{K_1,K_2}(w^{[i+1,\pm]})$ \emph{in situ}
		\State Increment $i \gets i + 1$
		\Until{$L^{[i,+]} \vee L^{[i,-]} \geq L^{[0]}$} \label{gift:Directional-evaluation}
		\Statex \textbf{Finalization:}
		\State Return $w_{\ft} \in \argmin_{ w \in \{ w^{[\mathrm{c,\cdot}]} \mid \mathrm{c} \in [i], \cdot \in \{+,-\} \} } \bigl\{ \mathrm{Eval}_{K_1,K_2}(w) \bigr\}$ \label{giftalg:selection}
	\end{algorithmic}
\end{algorithm}

For \GIFT, we prove the following performance guarantee in \Cref{sec:Proof-that-finetuning-is-needed}:

\begin{theorem}
	\label{thm:ft-guarantee}
	
	Let $s_0, s_t \in (0, \infty)$.
	Assume that \Cref{ass:Nonempty-minimizer-sets,ass:Laplacian-properties} hold.
	Let
	\begin{equation}
		w_\ft
		=
		\textnormal{\GIFT}
		( s_0; \eta, K_1, K_2 )
	\end{equation}
	refer to the fine-tuned weights as output by \Cref{alg:GIFT}.
	Assume that $w_0 \in \{ w : \nabla \mathcal{J}_{\sd_0}(w) = 0 \}$.
	Then, for sufficiently small $\eta$ and sufficiently large $K_1, K_2$,
	\begin{align}
		\ENO_{\sd_t}(w_\ft)
		<
		\ENO_{\sd_t}(w_0)
		.
		\label{eq: claim thoerem GIFT 2}
	\end{align}
\end{theorem}

\subsubsection{Conceptual explanation\texorpdfstring{ of \Cref{thm:FT-is-needed,thm:ft-guarantee}}{}}
\label{sec:Conceptual-explanation}

We found inspiration for \Cref{thm:FT-is-needed,thm:ft-guarantee} in \cite{sanders2016optimality}; a paper that establishes the existence of optimality gaps in appropriately dimensioned large many-server systems.
As in \cite{sanders2016optimality}, but for different reasons, there exists a performance gap when optimizing a misspecified \gls{ONN}.

\begin{figure}[hbtp]
	\centering
	\includegraphics[width=.99\linewidth]{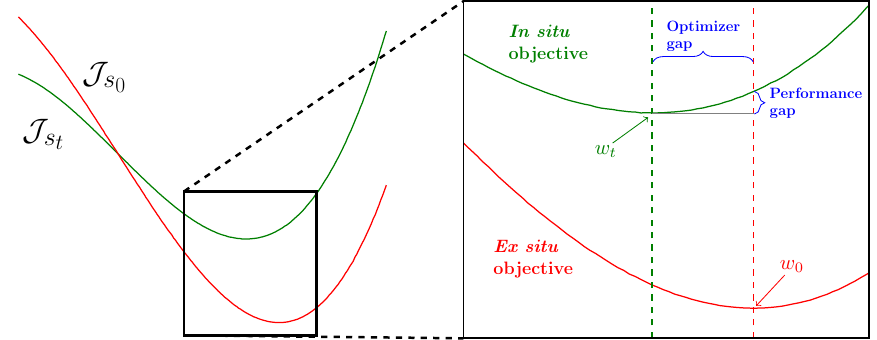}
	
	\caption{%
		(Left) Schematic depictions of $\mathcal{J}_{\sd_0}$ and $\mathcal{J}_{\sd_t}$ for some $\sd_0 \neq \sd_t$.
		(Right) The curves around the minimizers of the \emph{in situ} and \emph{ex situ} objectives.
	}
	\label{fig:example-fine-tuning}
\end{figure}

Let us explain this intuitively for a one-dimensional, stylized example in \Cref{fig:example-fine-tuning}.

In \Cref{fig:example-fine-tuning}, illustrations of $\mathcal{J}_{\sd_0}$ and $\mathcal{J}_{\sd_\true}$ are shown in red and green, respectively, for some $\sd_0 \neq \sd_\true$ while also $\sd_0 \approx \sd_\true$.
Notice that the curves are not identical because $\sd_0 \neq \sd_\true$, but that they are close to each other because $\sd_0 \approx \sd_\true$.
Notice also that the minimum of the top curve is strictly less than its value at the minimizer of the bottom curve.
This is the content of \Cref{thm:FT-is-needed}.
Thus, in short, a better minimizer and better minimum can be found when there is model misspecification.
Finally, observe that by moving away from the \emph{ex situ} minimizer $w_0$ (say experimentally towards the left and right) and evaluating $\ENO_{\sd_\true}$'s value by trial and error, one can indeed be led towards the actual minimizer $w_\true$.
\GIFT builds on this intuition; and the fact that this idea works, is the content of \Cref{thm:ft-guarantee}.

Recall, however, that the stylized example in \Cref{fig:example-fine-tuning} is one dimensional.
The parameter space of a \gls{NN} typically is of a much higher dimension; and the direction in which to adjust the parameters via trial--and--error is not obvious.
To address this, \Cref{alg:GIFT} calculates the direction of error descent as function of the noise level $\sd$ \emph{ex situ}.
\Cref{alg:GIFT} then explores this direction \emph{in situ} by testing along the one-dimensional line opened by this calculated direction.

\subsubsection{Example for condition \texorpdfstring{\eqref{ass:Laplacian-properties}}{(4)} in \texorpdfstring{\Cref{thm:FT-is-needed}}{Theorem 1}}

To make condition \eqref{ass:Laplacian-properties} in \Cref{thm:FT-is-needed} more transparent, we explicitly evaluate the right-hand side of the bound for a linear network and a specific learning task. Since \Cref{thm:FT-is-needed} is derived using Taylor expansions, a linear network---representing the first-order approximation---serves as a natural limiting case. In this setting, we observe that the dependence on $\zeta$ vanishes.

\begin{example}
	\label{ex: linear network}
	Consider a deep linear network of depth $L$ with \gls{AWGN} perturbed input $x \in \mathbb{R}^d$.
	The input distribution is assumed to have finite first and second moment.
	The input noise $\xi$ is drawn as
	$
	\xi \sim \mathcal{N}(0, s^2 I_d),
	$
	{and the perturbed input is}
	$
	\tilde{x} = x + \xi.
	$	
	The true function is assumed to be linear:
	\begin{align}
		y
		=
		f(x)
		:= v_L \dots v_2 v_1 x = V x \in \R,
	\end{align}
	where $V \in \mathbb{R}^{1 \times d}$ is the true effective linear map.
	Our model consists of $L$ linear layers without activations:
	\begin{align}
		\hat{y}
		=
		w_L \dots w_2 w_1 \tilde{x}
		=
		W \tilde{x},
		\quad
		\text{with}
		\quad
		W := w_L \dots w_2 w_1.
	\end{align}
	Note that the expected squared loss over both data and input noise then becomes
	\begin{align}
		\ENO_s(w_1, \ldots, w_L)
		=
		\mathbb{E}_{x, \xi}
		\bigg[
		\|V x - W (x + \xi)\|^2
		\bigg].
	\end{align}
	In that case, assuming additionally that all $x_i$ are uncorrelated and have the same variance,	\eqref{ass:Laplacian-properties} simplifies to
	\begin{align}
		|s_\true-s_0|
		<
		\Bigl(
		1+\frac{s_0^2}{\mathbb{E}\bigl[ x_i^2 \bigr]}
		\Bigr)
		\frac{1}{2\|V\|_2}.
		\label{eq condition finetuning works linear net}
	\end{align}
\end{example}

The condition derived in \eqref{eq condition finetuning works linear net} characterizes the range of permissible true noise levels $s_\true$ relative to the noise level used during training $s_0$, under which \GIFT's effectiveness remains guaranteed. In the linear setting considered, this range is quite generous, indicating that the restriction placed by condition \eqref{eq condition finetuning works linear net} is quite weak. Notably, assuming $\|V\|_2=1/2$ and $\mathbb{E}\bigl[ x_i^2 \bigr]=1$, \finetuning is always successful when the true system experiences less noise than assumed during training ($s_\true < s_0$). Moreover, if the true noise exceeds the assumed noise in training ($s_0<s_\true$), \finetuning remains guaranteed for all $s_\true < 1$, even when there was no noise in training.

The dependence on $\| V \|_2$ can be interpreted as a scaling effect of the network. As the number of parameters $d$ grows, all other statistics equal, $\| V \|_2$ typically grows, thus tightening the condition and reducing the permissible deviation between training and true noise levels. This suggests that larger networks, which exhibit greater capacity, require better knowledge of the noise conditions.

The bound \eqref{eq condition finetuning works linear net} suggests to pick $s_0$ with a slight bias towards larger values. Overshooting what one assumes to be the true noise standard deviation helps increasing the right-hand side of \eqref{eq condition finetuning works linear net}. However because training with excessive noise degrade training performance, one should try to not grossly overshoot.

\subsection{Related literature}

\paragraph{Noise in Optical Neural Networks}

Noise is a critical factor influencing the performance of \glspl{ONN}.
Noise originates from various sources during data encoding and signal processing, such as \glspl{DAC} and \glspl{ADC}, laser intensity fluctuations, and resolution limits of optical modulators \cite{de2023analysis}.
Thermal crosstalk in photonic weighting architectures \cite{tait2017neuromorphic}, cumulative noise in optical communication systems \cite{essiambre2010capacity, li2007channel, passalis2021training}, and noise introduced by activation functions \cite{de2019noise} further contribute to these challenges.

At the device level, architectural components such as microring weight banks, \gls{MZI} meshes, and \gls{SOA}-based networks introduce distinct noise characteristics.
Microring and \gls{SOA} noise are often modeled as \gls{AWGN}, while, for example, phase-shifter noise introduces \gls{AWGN} on phase shifts \cite{ma2020photonicMRRweightbankAWGN, shi2021noise, shen2017deep}, which is multiplicative noise.
Fabrication errors, such as those in directional couplers, can also result in \gls{AWGN} on transmission coefficients \cite{gu2022adept, marchesin2025braided}.

Given that \gls{AWGN} is the predominant noise model in \gls{ONN} research \cite{semenova2022understanding, passalis2021training, mourgias2022noise, kirtas2022robust}, we adopt it in our work.
Our abstraction models the \gls{ONN} as a network of interconnected nodes with noise introduced between layers.
This approach is agnostic to specific implementations and incorporates both additive and multiplicative noise, accounting for nonlinear interlayer interactions \cite{semenova2022understanding, de2019noise}.

This noise modeling framework applies to diverse \gls{ONN} architectures, including \gls{AO} and \gls{E/O/E} systems \cite{mourgias2019all, shi2019deep}.
For \gls{AO}~\glspl{ONN}, noise manifests as \gls{AWGN} at the layer output due to activation functions, while in \gls{E/O/E} systems, it arises after optical-to-electrical conversion (\gls{DAC} and \gls{ADC}).
While the present paper focuses on \gls{AWGN} as the dominant noise source, future work could explore additional noise effects, such as shot noise and Johnson--Nyquist noise, to refine the model.

\paragraph{Robustness to Noise and Precision Limitations}

Ensuring functional correctness in analog photonic computing is challenging due to process variations, device noise, environmental factors, and limited endurance.
Additionally, precision constraints caused by finite \gls{ADC} and \gls{DAC} resolutions introduce significant overhead for high-precision designs.

Quantization-aware techniques \cite{gu2020roq, zhu2022elight, wang2022integrated, zhu2023multi, kirtas2022quantization} mitigate the impact of low-precision arithmetic, improving neural network resilience.
However, quantization-aware techniques alone may not fully address the loss of computational accuracy caused by low-precision \glspl{ADC}.

Noise-aware training is another widely adopted approach, where variations are introduced during training to improve robustness \cite{zhao2019design, gu2020roq, gu2021o2nn, sludds2022delocalized}.
This method incorporates explicit robustness optimization terms or employs knowledge distillation to guide noisy student models with noise-free teacher models \cite{gu2021o2nn}.
Studies have also explored training \glspl{ONN} directly on noisy hardware to naturally account for device-level variations \cite{feng2022compact, wright2022deep, zhan2024physics}.
In another recent line of research, susceptibility of \glspl{ONN} to noise and device imperfections is addressed by searching for minimizers in flat regions in the loss landscape \cite{kosmella2024towards, Xu24, Varri24}, by adding a gradient penalty.
In addition to training strategies, hardware-level solutions, such as reducing active device counts through pruning redundant devices \cite{gu2022adept}, weight blocks \cite{gu2020towards, feng2022compact}, or averaging out noise through duplication schemes \cite{kosmella2023higher,kosmella2024noise} can mitigate noise-induced errors.
Circuit-level optimizations, including design space exploration to reduce process variation and crosstalk \cite{sunny2021crosslight, mirza2021silicon}, further enhance robustness.

\paragraph{Integrated photonics for neural computing}

Integrated photonics provides a possible technological foundation for modern \glspl{ONN}, enabling scalable, compact, and energy-efficient implementations of optical computation \cite{shen2017deep, feldmann2021parallel, tait2017neuromorphic}.
Photonic integration leverages the ability to fabricate optical components such as waveguides, splitters, modulators, and interferometers on a single chip using \gls{CMOS}-compatible processes.
This allows for high-density integration, reduced cost per component, and the possibility of mass production. The increasing availability of validated \glspl{PDK} from commercial foundries is moving the photonics industry toward a standardized framework, akin to the fabless semiconductor industry \cite{lim2013review}.

A fundamental building block of many \glspl{ONN} is the \gls{MZI}-mesh, which can implement arbitrary unitary matrices and thus serves as a key mechanism for optical linear transformations \cite{clements2016optimal, fldzhyan2020optimal, feng2022compact, marchesin2025braided}.

Recent advances in integrated photonics have pushed the boundaries of what is achievable with optical computing.
For instance, \gls{WDM} has been proposed to increase parallelism and bandwidth by using multiple wavelengths for simultaneous computation \cite{tait2017neuromorphic,shi2019deep, feldmann2021parallel}.
Nonlinear activation functions have also been realized using photodetectors \cite{ashtiani2022chip, tait2019silicon}, saturable absorbers \cite{shi2023photonic}, and semiconductor optical amplifiers \cite{shi2019first}.

Among the realized networks in integrated photonics are \glspl{MLP} \cite{shen2017deep}, as well as more specialized architectures such as \glspl{CNN} \cite{chen2023all, meng2023compact, xu2021tops} and \glspl{SNN} \cite{chakraborty2019photonic, feldmann2019all}.

\section{Preliminaries}
\label{sec:Preliminaries}

\subsection{Modeling \texorpdfstring{\glsentrylongpl{ONN}}{Optical Neural Networks}}
\label{sec:Modeling-ONNs}

For noise levels $s \in (0, \infty)$, we will model $\Phi^{\mathrm{ONN}}$ using the noisy, layered, parameterized function $M_s$ as described in \cite[Section 1.2]{kosmella2024noise}.
Here are the details.

First, let us define all variables pertaining to the structure of the \gls{ONN}.
Let $L \in \mathbb{N}_+$ be the number of layers, and $\sigma : \mathbb{R} \to \mathbb{R}$ some activation function.
For $\ell \in \{ 1, \ldots, L \}$, let $d_\ell \in \mathbb{N}_+$ denote to the number of neurons in layer $\ell$, $W^{{(\ell)}} \in \R^{d_\ell\times d_{\ell-1}}$ the weight matrix in layer $\ell$, and $b^{(\ell)} \in \R^{d_\ell}$ the bias in layer $\ell$.
Let $w := (W^{{(\ell)}},b^{(\ell)})_{\ell \in \{ 1, \ldots, L \}} \in \mathcal{W}$ refer to the tunable parameters of the \gls{ONN}.
Here, $\mathcal{W} := \times_{\ell=1}^L ( \R^{d_\ell\times d_{\ell-1}} \times \R^{d_{\ell}})$.

Next, let us define all relevant random variables describing noise within the \gls{ONN}.
For $\ell \in \{ 1, \ldots, L \}$, let
\begin{equation}
	N^{\mathrm{w},(\ell)} \sim \mathcal{N}(\mathbf{0}, \sd^2 \mathrm{I}_{d_i})
	\quad
	\textnormal{and}
	\quad
	N^{\mathrm{a},(\ell)} \sim \mathcal{N}(\mathbf{0}, \sd^2 \mathrm{I}_{d_i})
	\label{eqn:Normally-distributed-noise-matrices}
\end{equation}
be multivariate normal distributed with the indicated location vectors and covariance matrices.
The random variables $N^{\mathrm{w},(\ell)}$ and $N^{\mathrm{a},(\ell)} $ model the noise associated with weighing and to applying the activation function, respectively.
We assume that all noise terms are independent of each other and of the data.
Finally, we let $\mathbf{N} = \{ N^{\mathrm{w},(\ell)}, N^{\mathrm{a},(\ell)} \}_{\ell=1}^L$ refer to their collection.

We are now in position to define the parameterized random variable $M_s: \R^{d_0} \times \mathcal{W} \to \R^{d_L}$.
For $(x, w) \in \R^{d_0} \times \mathcal{W}$, define
\begin{align}
	&
	A^{(0)}
	:=
	x + N^{\mathrm{a},(0)},
	\\
	&
	A^{(\ell)}
	:=
	\sigma
	\bigl(
	W^{(\ell)}
	A^{(\ell-1)} + b^{(\ell)} + N^{\mathrm{w},(\ell)}
	\bigr)
	+
	N^{\mathrm{a},(\ell)}
	\quad
	\forall i
	\in\{ 1,\dots,L-1 \},
	\label{def:matrices-Ai}
	\\
	&
	M_\sd
	(x , w, \mathbf{N})
	:=
	A^{(L)}
	:=
	W^{(L)}A^{(L-1)}+b^{(L)}+N^{\mathrm{w},(L)}
	.
	\label{def:matrix-AL}
\end{align}
The activation function $\sigma$ is applied component-wise here.

\subsection{Modeling physical limitations}
\label{sec:Modeling-physical-limitations}

We suppose that due to physical limitations---such as optical power constraints and activation functions implementable in \glspl{ONN}---only a compact, convex subset of weights ${H} \subseteq \mathcal{W}$ can be achieved.

The practical example of such a constraint set that we restrict the analysis to is, for $W_\textnormal{min}, W_\textnormal{max}, b_\textnormal{min}, b_\textnormal{max} \in \mathbb{R}$ satisfying $W_\textnormal{min} < W_\textnormal{max}, b_\textnormal{min} < b_\textnormal{max}$, the hyperrectangle
\begin{equation}
	H
	=
	\times_{\ell = 1}^L
	\big(
	[ - W_\textnormal{min}, W_\textnormal{max} ]^{d_\ell \times d_{\ell-1}}
	\times
	[ - b_\textnormal{min}, b_\textnormal{max} ]^{d_\ell}
	\big)
	.
	\label{ass:H-is-a-hyperrectangle}
\end{equation}
Our assumption that $H$ is this hyperrectangle is used in the proof of \Cref{lem:Square-Integrability-of-the-Forward-and-Backward-Passes}.

\subsection{Modeling \emph{in silico} training}
\label{sec: Training with stochastic approximation}

Training \glspl{NN} is done by minimizing some loss function with respect to the weights.
The go-to tool to do this is \gls{SGD}.
Now, \gls{SGD} can certainly be used exactly as in \Cref{eqn:projected-sgd} to train a noisy \gls{NN}, but the noise affects the objective function that one is actually minimizing.

In order to analyze \Cref{alg:GIFT}, we will thus need to understand the limit behavior of \Cref{eqn:projected-sgd}.
In practice, training is considered done once the \gls{NN} parameters have numerically converged.
Within the context of our model, this refers to the limit
\begin{equation}
	\lim_{k \to \infty} w^{\{k\}}
	\label{eqn:Limit-of-wk}
	,
\end{equation}
which is hopefully well-behaved and a minimizer of $\mathcal{J}_\sd$.

\subsection{Convergence results on \emph{in silico} training}
\label{sec:Convergence-results-on-in-silico-training}

To study \Cref{eqn:Limit-of-wk}, we rely on the theory of stochastic approximation \cite{kusher2003stochastic}.
We suppose that the following construction is in a suitable probability space $(\Omega, \mathcal{A}, \mathbb{P})$ while foregoing the uneventful details on this aspect.

The proof relies on the construction of a c\`{a}dl\`{a}g approximating function $w^0(t)$ for which one can establish a limit scaling, as follows.
Let, for $k \in \naturalNumbersPlus$,
\begin{equation}
	\eps_k
	Z^{\{k\}}
	:=
	w^{\{k+1\}} - w^{\{k\}}
	+
	\eps_k
	\nabla_w
	(
	y^{\{k\}}
	-
	M_\sd({x}^{\{k\}} , w^{\{k\}}, \mathbf{N}^{\{k\}})
	)^2
\end{equation}
refer to the shortest vectors that push the \gls{SGD} algorithm back into $H$ whenever needed.
Let
\begin{equation}
	t_n
	=
	\sum_{i=0}^{n-1}\eps_{i}
\end{equation}
be the typical timescale for the approximating function $w^0(t)$.
Concretely, let $m(t)$ be the unique value of $k$ such that $t_k\leq t < t_{k+1}$, and then define for $t_{m(t)} \leq t < t_{m(t)+1}$,
\begin{equation}
	w^0(t)
	:=
	w^{\{m(t)\}}
	,
	\quad
	Z^0(t)
	=
	\sum_{i=0}^{m(t)-1}
	\eps_iZ^{\{i\}}
\end{equation}
as well as for $k \in \naturalNumbersPlus$,
\begin{align}
	Z^k(t)
	=
	Z^0(t_k+t)-Z^0(t_k)
	=
	\sum_{i=k}^{m(t_k+t)-1}
	\eps_iZ^{\{i\}}
	.
\end{align}

\subsubsection{Limit trajectories of the training algorithm}
\label{sec:Limit-trajectories-of-the-training-algorithm}

We now make the following technical assumptions:

\begin{assumptionsNoNum}
	\begin{enumerate}[label=(A\arabic*)]
		\item
		$\sigma \in C^2_{\mathrm{PB}}(\mathbb{R})$; see \cite[Def.~5]{senen2020almost} for a definition of $C_{\mathrm{PB}}^2$.
		\label{item:A1}
		
		\item
		For $m \in \{ 1, 2, 3, 4 \}, n \in \mathbb{N}_0$, $\mathbb{E}[\|Y\|_2^m\|X\|_2^n] < \infty$.
		\label{item:A2}
		
		\item
		For $k \in \mathbb{N}_0$, the random variables $\mathbf{N}^{\{k\}}$ and $( x^{\{k\}}, y^{\{k\}} )$ are independent copies of $\mathbf{N} \sim \mathcal{N}_s$ and $(x, y) \sim \mu$, respectively.
		\label{item:A3}
		
		\item
		The step sizes are deterministic and satisfy
		$
		\sum_{t=1}^{\infty}\eps_t=\infty
		$
		and
		$
		\sum_{t=1}^{\infty}\eps_t^2<\infty
		.
		$
		\label{item:A4}
		
		\item
		The set $H$ satisfies \Cref{ass:H-is-a-hyperrectangle}.
		\label{item:A5}
		
	\end{enumerate}
\end{assumptionsNoNum}

Under these assumptions, we prove the following in \Cref{appendix:proof-ODE}:

\begin{proposition}[Limiting behavior of \Cref{eqn:projected-sgd}]
	\label{thm:stochastic-approximation-theorem-limit-trajectories}
	
	Assume \ref{item:A1}---\ref{item:A5}.
	There exists a null set $N\subset \Omega$ such that for $\omega\notin N$, $(w(t)(\omega))_t$ converges to a limit set of the projected \gls{ODE}
	\begin{align}
		\frac{\mathrm{d}w}{\mathrm{d}t}
		=
		-
		\nabla_w
		\ENO\bigr\rvert_H
		(w)
		+
		\pi(w)
		.
		\label{eqn:projected-ODE}
	\end{align}
\end{proposition}

For details on how to construct $\pi$ in \Cref{eqn:projected-ODE}, we refer interested readers to \cite[\S{4.3}]{kusher2003stochastic} and/or \cite[Appendix C]{senen2020almost}.

Our proof of \Cref{thm:stochastic-approximation-theorem-limit-trajectories} can be found in \Cref{appendix:proof-ODE}.
It actually implies, more precisely,
that
(i) the set of functions $\{w^{\{k\}}(\omega, \allowbreak \cdot), \allowbreak Z^{\{k\}}(\omega, \cdot)\}$ is equicontinuous and that $(w(\omega, \allowbreak \cdot), \allowbreak Z(\omega, \cdot))$ can be chosen as the limit of a convergent subsequence that satisfies \eqref{eqn:projected-ODE},
and that
(ii) \eqref{eqn:Limit-of-wk} converges to this limit.

\subsubsection{Limit points of the training algorithm}
\label{sec:Limit-points-of-the-training-algorithm}

With some additional assumptions we can also characterize the limit points of \Cref{eqn:projected-sgd}:

\begin{assumptionsNoNum}
	\begin{enumerate}[label=(B\arabic*)]
		\item $\sigma \in C_{\mathrm{PB}}^r$ with $\mathrm{dim}(\mathcal{W}) \leq r$; see \cite[Def.~5]{senen2020almost} for a definition of $C_{\mathrm{PB}}^r$.
		\label{item:B1}
		
		\item If $\nabla_w\ENO\bigr\rvert_H(w)\neq 0$ then $-\nabla_w\ENO\bigr\rvert_H (w)+\pi(w)\neq 0$.
		\label{item:B2}
		
	\end{enumerate}
\end{assumptionsNoNum}

Specifically, by assuming \Cref{item:B1,item:B2} on top of \Crefrange{item:A1}{item:A5}, we can also identify the set from which the limiting points originate:

\begin{proposition}[Limit points of \Cref{eqn:projected-sgd}]
	\label{thm:stochastic-approximation-theorem-limitpoints}
	
	Assume \Cref{item:B1,item:B2} on top of \Crefrange{item:A1}{item:A5}.
	Then, for almost all $\omega\in\Omega$, $(w(t)(\omega))_t$ converges to a unique point in $\{ w \mid \nabla_w
	\ENO\bigr\rvert_H
	(w)=0 \}$.
\end{proposition}

The proof of \Cref{thm:stochastic-approximation-theorem-limitpoints} can also be found in \Cref{appendix:proof-ODE}.
Most work goes into establishing the following boundedness result:

\begin{lemma}
	\label{lem:Square-Integrability-of-the-Forward-and-Backward-Passes}
	
	Presume \Cref{item:A1,item:A2,item:A4}.
	For $s \in (0,\infty)$,
	\begin{equation}
		\mathbb{E}
		\big\|
		\nabla_{W^{(i)}} \ENO_s(w) \bigr\rvert_{\mathbf{N},x,y}
		\big\|^4
		<
		\infty
		\quad
		\textnormal{and}
		\quad
		\mathbb{E}
		\big\|
		\nabla_{b^{(i)}} \ENO_s(w) \bigr\rvert_{\mathbf{N},x,y}
		\big\|^4
		<
		\infty
		.
		\label{eqn:finiteness-of-the-expected-norm-of-the-gradient}
	\end{equation}
\end{lemma}

The proof of \Cref{lem:Square-Integrability-of-the-Forward-and-Backward-Passes} is relegated to \Cref{sec:Boundedness-of-ONNs-and-their-gradients}.

\section{Analysis of \GIFT}
\label{sec:Analysis-of-GIFT}

\subsection{Ideas behind \GIFT}
\label{sec:Ideas-behind-GIFT}

The first concept behind \GIFT is to exploit the continuity of $\nabla_w \ENO_\sd$ as a function of $\sd$.
Think specifically of exploring one small step with gradient descent: if $w_{\sd_0}$ is sufficiently close to $w_{\sd_\true}$ to begin with, then there exist $\sd$ in the neighborhood of $\sd_0$, such that
\begin{align}
	\ENO_\true\big(
	v_\sd\big)
	<
	\ENO_\true\big(
	w_{\sd_0}\big)
	\quad
	\textnormal{with}
	\quad
	v_\sd
	:=
	w_{\sd_0}
	+
	\Bigl(
	\frac{\partial}{\partial \sd}
	\nabla_w
	\ENO_{\sd}
	\big\vert_{\sd = \sd_0}
	\Bigr)
	\cdot
	(\sd-\sd_0)
	.
	\label{eq: gift taylor rule}
\end{align}
Here, $v_\sd$ can be understood as a promising candidate weight vector obtained by nudging the weight vector $w_{\sd_0}$ in the direction of $( \partial / \partial\sd ) \nabla_w \ENO_{\sd_0}$.

Now, for the second and key step of \GIFT we need an estimate of the gradient $(\partial / \partial \sd) \nabla_w \ENO_{\sd_0}$.
Denoting this estimate by $D^{[0]}$, we can then generate candidates $\hat{v}_\sd$ by performing a line search along the direction $D^{[0]}$.
Similar to \Cref{eq: gift taylor rule}, one can then still expect that if $D^{[0]} \approx \frac{\partial}{\partial \sd} \nabla_w \ENO_{\sd_0}$, there exist some $\sd$ in the neighborhood of $\sd_0$ such that
\begin{align}
	\ENO_\true\big(
	\hat{v}_\sd\big)
	<
	\ENO_\true\big(
	w_{\sd_0}\big)
	\quad
	\textnormal{with}
	\quad
	\hat{v}_\sd
	:=
	w_{\sd_0}
	+
	D^{[0]}
	(\sd-\sd_0)
	.
\end{align}
Now, to test whether such candidate $\hat{v}_\sd$ really is better, \GIFT evaluates the performance of such candidate in the \emph{actual} \gls{ONN} system experimentally.

\subsection{Why \texorpdfstring{$D^{[0]}$}{D0} is a promising direction for \GIFT}
\label{sec:Why-D0-is-a-promising-direction-for-GIFT}

\subsubsection{Backpropagated matrices}
\label{sec:Backpropagated-matrices}

Note that the gradient
$
\nabla
(
y
-
M_s({x} , w, \mathbf{N})
)^2
$
necessary for \Cref{eqn:projected-sgd} (and \eqref{eq: gift taylor rule}) can be calculated efficiently using the backpropagation algorithm.
This algorithm calculates iteratively, for $\ell = L,L-1,\dots,1$,
\begin{align}
	\errorvar^{(L)}
	&
	:=
	y
	-
	M_s({x} , w, \mathbf{N}), \text{ and}
	\label{def:matrix-RL}
	\\
	\errorvar^{(\ell)}
	&
	:=
	\big(
	W^{(\ell+1)}
	\big)^T
	\errorvar^{\ell+1}
	\odot
	\sigma^\prime
	\big(
	W^{(\ell)}
	A^{(\ell-1)} + b^{(\ell)}+ N^{\mathrm{w},(\ell)} \big)
	\label{def:matrices-Ri}
	.
\end{align}
The $W^{(\ell)}$ component of
$
\nabla (
y
-
M_s({x} , w, \mathbf{N})
)^2
$ is then given by
\begin{equation}
	\nabla_{W^{(\ell)}}
	(
	y
	-
	M_s({x} , w, \mathbf{N})
	)^2
	=
	-2
	\errorvar^{(\ell)}
	\big(A^{(i-1)}\big)^T
	,
	\label{eqn:backpropagation-for-W}
\end{equation}
and the $b^{(\ell)}$ component of $\nabla (
y
-
M_s({x} , w, \mathbf{N})
)^2$ is then given by
\begin{equation}
	\nabla_{b^{(\ell)}} (
	y
	-
	M_s({x} , w, \mathbf{N})
	)^2
	=
	-2
	\errorvar^{(\ell)}
	.
	\label{eqn:backpropagation-for-b}
\end{equation}
We prove \Cref{eqn:backpropagation-for-W,eqn:backpropagation-for-b} in \Cref{sec:backpropagation_proof}.

\subsubsection{Calculation of \texorpdfstring{$(\partial / \partial \sd) \nabla_w \ENO_{\sd}$}{the derivative wrt s}}
\label{sec:Calculation-of-d-ds-nabla-ENO}

Start by observing that
\begin{align}
	\frac{\partial}{\partial \sd}
	\nabla_{W^{(\ell)}}
	\ENO_{\sd}
	&
	\eqcom{\ref{eqn:objective-function-as-a-function-of-s}}=
	\int
	\Big(
	\frac{\partial}{\partial \sd}\mathbb{E}_\mathbf{N}
	\nabla_{W^{(\ell)}}
	(
	y
	-
	M_s({x} , w, \mathbf{N})
	)^2\Big)
	\mathrm{d}
	\mu({x},y)
	\label{eq: finite sample ENO}
	\\  &
	\eqcom{\ref{eqn:backpropagation-for-W}}=
	\int
	\Bigl(
	-2
	\frac{\partial}{\partial \sd}
	\mathbb{E}_\mathbf{N}
	\errorvar^{(\ell)}
	\big(A^{(\ell-1)}\big)^T
	\Bigr)
	\mathrm{d}
	\mu({x},y)
	\label{eq:integral_form_d_ds_nabla_W_ENO}
\end{align}
and similarly
\begin{align}
	\frac{\partial}{\partial\sd}
	\nabla_{b^{(\ell)}}
	\ENO_{\sd}
	=
	\int
	\Big(-2
	\frac{\partial}{\partial \sd}
	\mathbb{E}_\mathbf{N}
	\errorvar^{(\ell)}\Big)
	\mathrm{d}
	\mu({x},y)
	.
	\label{eq:integral_form_d_ds_nabla_b_ENO}
\end{align}
Observe next that
\begin{align}
	\mathbb{E}_\mathbf{N}
	\errorvar^{(\ell)}
	=
	&
	\int_{\R^{d_0}}
	\int_{\R^{d_0}}
	\cdots
	\int_{\R^{d_{L}}}
	\int_{\R^{d_L}}
	\prod_{i=0}^{L}
	\phi_{\sd}(n^{\mathrm{w},(\ell)})\phi_{\sd}(n^{\mathrm{a},(\ell)})
	\nonumber \\  &
	\Bigl[
	\bigl(
	W^{(i+1)}
	\bigr)^T
	\errorvar^{(i+1)}
	\odot
	\sigma^\prime
	\bigl(
	W^{(\ell)}
	A^{(\ell-1)} + b^{(\ell)}+ n^{\mathrm{w},(\ell)}
	\bigr)
	\Bigr]
	\nonumber \\  &
	\phantom{=}
	\mathrm{d}n^{\mathrm{w},(0)}
	\mathrm{d}n^{\mathrm{a},(0)}
	\cdots
	\mathrm{d}n^{\mathrm{w},(L)}
	\mathrm{d}n^{\mathrm{a},(L)}
	.
\end{align}
Here, $\phi_{\sd}$ refers to the probability density function of a multivariate normal distribution with mean $0$ and covariance matrix $\Sigma = s^2 I_{\dim(n)}$, i.e.,
\begin{align}
	\phi_{\sd}(n)
	=
	\frac{1}{\sd^{\dim(n)}(2\pi)^{\dim(n)/2}}
	\exp\Bigl(
		-
		\frac{1}{2\sd^2}
		n^T n
	\Bigr)
	.
\end{align}
Now, to compute $(\partial / \partial \sd) \mathbb{E}_\mathbf{N} \errorvar^{(\ell)}$, we may move the derivative operator $\partial / \partial \sd$ inside the integrals.
This is justified by the dominated convergence theorem (recall \Cref{lem:Square-Integrability-of-the-Forward-and-Backward-Passes}).
We find that
\begin{align}
	\frac{\partial}{\partial \sd}\mathbb{E}_\mathbf{N}
	\errorvar^{(\ell)}
	=
	&
	\int_{\R^{d_0}}
	\int_{\R^{d_0}}
	\cdots
	\int_{\R^{d_{L}}}
	\int_{\R^{d_L}}
	\frac{\partial}{\partial \sd}\bigg(\prod_{i=0}^{L}
	\phi_{\sd}(n^{\mathrm{w},(\ell)})\phi_{\sd}(n^{\mathrm{a},(\ell)})
	\bigg)
	\nonumber \\  &
	\Bigl[
	\big(
	W^{(i+1)}
	\big)^T
	\errorvar^{(i+1)}
	\odot
	\sigma^\prime
	\big(
	W^{(\ell)}
	A^{(\ell-1)} + b^{(\ell)}+ n^{\mathrm{w},(\ell)} \big)
	\Bigr]
	\nonumber \\  &
	\mathrm{d}n^{\mathrm{w},(0)}
	\mathrm{d}n^{\mathrm{a},(0)}
	\cdots
	\mathrm{d}n^{\mathrm{w},(L)}
	\mathrm{d}n^{\mathrm{a},(L)}
	.
	\label{eq:expecation_Ri_with_product}
\end{align}

Next, let $n^{\mathrm{w},(\ell)} := n^{(2i)}$ and $n^{\mathrm{a},(\ell)} := n^{(2i+1)}$.
This enables us to write
\begin{align}
	\prod_{i=0}^{L}
	\phi_{\sd}(n^{\mathrm{w},(\ell)})\phi_{\sd}(n^{\mathrm{a},(\ell)})
	=
	\prod_{j=0}^{2L+1}
	\phi_{\sd}(n^{(j)})
	.
\end{align}
By \Cref{lem:chain-derivative} in \Cref{sec:Partial-derivative-of-a-product-of-Gaussian-densities},
\begin{align}
	\frac{\partial}{\partial\sd}\bigg(\prod_{j=0}^{2L+1}
	\phi_{\sd}(n^{(j)})\bigg)
	=
	\sum_{j=0}^{2L+1} \big(\sd^{-2}(n^{(j)})^T n^{(j)} - d_j\big)
	\bigg(\prod_{k=0}^{2L+1}
	\phi_{\sd}(n^{(k)})\bigg).
	\label{eq:application_of_lemma_chain_derivative_to_prdoctuct_in_backprop}
\end{align}
Substituting \Cref{eq:application_of_lemma_chain_derivative_to_prdoctuct_in_backprop} into \Cref{eq:expecation_Ri_with_product} we find that
\begin{align}
	\frac{\partial}{\partial \sd}\mathbb{E}_\mathbf{N}
	\errorvar^{(\ell)}
	=
	&
	\int_{\R^{d_0}}
	\int_{\R^{d_0}}
	\cdots
	\int_{\R^{d_{L}}}
	\int_{\R^{d_L}}
	\sum_{j=0}^{2L+1} \big(\sd^{-2}(n^{(j)})^T n^{(j)} - d_j\big)
	\bigg(\prod_{k=0}^{2L+1}
	\phi_{\sd}(n^{(k)})\bigg)
	\nonumber \\  &
	\Bigl[
	\bigl(
	W^{(i+1)}
	\bigr)^T
	\errorvar^{(i+1)}
	\odot
	\sigma^\prime
	\bigl(
	W^{(\ell)}
	A^{((\ell-1))} + b^{(\ell)}+ \underbrace{n^{\mathrm{w},(\ell)}}_{=n^{(2i)}}
	\bigr)
	\Bigr]
	\nonumber \\  &
	\mathrm{d}n^{\mathrm{w},(0)}
	\mathrm{d}n^{\mathrm{a},(0)}
	\cdots
	\mathrm{d}n^{\mathrm{w},(L)}
	\mathrm{d}n^{\mathrm{a},(L)}
	.
\end{align}

Finally, let $\mathcal{S} = \{(\mathrm{a},(0)), \allowbreak (\mathrm{w}, \allowbreak (1)), \allowbreak (\mathrm{a}, \allowbreak (1)), \allowbreak \dots, \allowbreak (\mathrm{w}, \allowbreak (L))\}$.
Then
\begin{align}
	\frac{\partial}{\partial \sd}\nabla_{b^{(\ell)}}
	\ENO_{\sd_0}
	& =
	\int
	\mathbb{E}_\mathbf{N}
	\Bigl[
	\Big(\sum_{\alpha\in \mathcal{S}} \big(\sd^{-2}\big(N^{\alpha}\big)^T N^\alpha - d_{f(\alpha)}\big)\Big)
	\errorvar^{(\ell)}
	\Bigr]
	\mathrm{d}
	\mu({x},y)
	\label{eq: d ds nablab ENO}
\end{align}
and similarly
\begin{align}
	\frac{\partial}{\partial \sd}\nabla_{W^{(\ell)}}
	\ENO_{\sd_0}
	& =
	\int
	\frac{\partial}{\partial \sd}
	\mathbb{E}_\mathbf{N}
	\bigl[
	\errorvar^{(\ell)}
	\big(A^{(\ell-1)}\big)^T
	\bigr]
	\mathrm{d}
	\mu({x},y)
	\nonumber \\  &
	=
	\int
	\mathbb{E}_\mathbf{N}
	\Bigl[
	\Bigl(
	\sum_{\alpha\in \mathcal{S}}
	\bigl(
	\sd^{-2}\big(N^{\alpha}\big)^T N^\alpha - d_{f(\alpha)}
	\bigr)
	\Bigr)
	\errorvar^{(\ell)}
	\big(A^{(\ell-1)}\big)^T
	\Bigr]
	\mathrm{d}
	\mu({x},y)
	.
	\label{eqn:d-ds-nablaW-ENO}
\end{align}

We can now establish that \Cref{eq:sample_D0_Wi,eq:sample_D0_bi} are unbiased sample mean estimates of the integrals in \Cref{eqn:d-ds-nablaW-ENO,eq: d ds nablab ENO}.
This is the content of \Cref{lem:unbiased-sampling-D0}:

\begin{lemma}
	\label{lem:unbiased-sampling-D0}
	
	Presume \Cref{item:A1,item:A2,item:A4}.
	Then
	\begin{align}
		\lim_{K_1,K_2\rightarrow\infty}
		D^{[0]}(K_1,K_2)
		=
		\frac{\partial}{\partial \sd}
		\nabla_w \ENO_{\sd_0}(w_0)
		\quad
		\textnormal{almost surely}
		.
	\end{align}
\end{lemma}

The proof can be found in \Cref{sec:Proof-that-D0-is-approx-d-ds-nabla-ENO}.

\subsection{Formal establishment that Algorithm~\ref{alg:GIFT} works}

We now establish conditions under which \Cref{alg:GIFT} guarantees an improvement.
Intuitively, the true minimizer of the \gls{ONN} model with the correct noise model should achieve a lower objective function value than the minimizer obtained from a misspecified model.
However, establishing this requires careful analysis.

In the context of a misspecified model where the \gls{AWGN} has standard deviation $\sd_0$ and the true model follows \gls{AWGN} with standard deviation $\sd_\true$, our goal is to rigorously show that
\begin{align}
	\ENO_{\sd_t}(w_\true) < \ENO_{\sd_t}(w_0)
\end{align}
strictly.
Here, $w_0$ and $w_\true$ are critical points of $\ENO_{\sd_0}$ and $\ENO_{\sd_t}$, respectively.
More specifically, we assume
that \Cref{alg:GIFT}'s input satisfies
$
w_0
\in
\{ w \mid \nabla_w \ENO_{\sd_0}(w) = 0 \}
$, i.e., that \Cref{alg:GIFT} starts from a stationary point $w_0$ under the assumed standard deviation $\sd_0$; recall also \Cref{thm:stochastic-approximation-theorem-limitpoints}.

\subsubsection{Proof \texorpdfstring{of \Cref{thm:FT-is-needed}}{that \finetuning is needed}}
\label{sec:Proof-that-finetuning-is-needed}

We will show that an $\alpha^\ast\in \R$ exists such that
\begin{align}
	\ENO_{\sd_t}
	\Bigl(
	w_0
	-
	\alpha^\ast
	\frac{\partial}{\partial \sd}
	\nabla_w \ENO_{\sd_0}(w_0)
	\Bigr)
	<
	\ENO_{\sd_t} (w_0)
	.
	\label{eq:part_1_to_show}
\end{align}
Furthermore, we will show that for all $\alpha$ between $0$ and $\alpha^\ast$, \eqref{eq:part_1_to_show} holds as well.

When applying Taylor's theorem in Lagrange form to the true objective function, we find that there exists a $\vartheta \in (0,1)$ such that
\begin{align}
	&
	\ENO_{\sd_t}
	\Bigl(
	w_0-\alpha
	\frac{\partial}{\partial \sd}
	\nabla_w \ENO_{\sd_0}(w_0)
	\Bigr)
	=
	\ENO_{\sd_t}(w_0)
	-
	\alpha
	\bigl(\nabla_w \ENO_{\sd_t}(w_0)\bigr)^T
	\frac{\partial}{\partial \sd}
	\nabla_w \ENO_{\sd_0}(w_0)
	+
	\label{eq:taylor_expanded_gift_descent}
	\\  &
	\frac{1}{2}
	\Bigl(
	-\alpha\frac{\partial}{\partial \sd}
	\nabla_w \ENO_{\sd_0}(w_0)
	\Bigr)^T
	\mathrm{Hess}
	\Bigl(
	\ENO_{\sd_t}\Big(w_0- \vartheta \alpha \frac{\partial}{\partial\sd}\nabla_w \ENO_{\sd_t}(w_0)\Big)
	\Bigr)
	\Bigl(
	-\alpha\frac{\partial}{\partial \sd}
	\nabla_w \ENO_{\sd_0}(w_0)
	\Bigr)
	.
\end{align}
By substituting \eqref{eq:taylor_expanded_gift_descent} into \eqref{eq:part_1_to_show} we obtain the equivalence of \eqref{eq:part_1_to_show} to
\begin{align}
	&
	\frac{\alpha^2}{2}
	\Bigl(
	\frac{\partial}{\partial \sd}
	\nabla_w \ENO_{\sd_0}(w_0)
	\Bigr)^T
	\mathrm{Hess}\bigg(\ENO_{\sd_t}\Big(w_0- \vartheta \alpha \frac{\partial}{\partial\sd}\nabla_w \ENO_{\sd_t}(w_0)\Big)\bigg)
	\Bigl(
	\frac{\partial}{\partial \sd}
	\nabla_w \ENO_{\sd_0}(w_0)
	\Bigr)
	\nonumber \\  &
	<
	\alpha
	\Bigl(
	\nabla_w \ENO_{\sd_t}(w_0)
	\Bigr)^T
	\Bigl(
	\frac{\partial}{\partial \sd}
	\nabla_w \ENO_{\sd_0}(w_0)
	\Bigr).
	\label{eq:cond_for_improvement}
\end{align}
Observe that \eqref{eq:cond_for_improvement} reads $\alpha^2 C_2 < \alpha C_1$ for some $C_2,C_1\in\R$.
We will now distinguish between multiple cases.

We will first show that the case
$
C_1
=
\bigl(
\nabla_w \ENO_{\sd_t}(w_0)
\bigr)^T
\frac{\partial}{\partial \sd}
\nabla_w \ENO_{\sd_0}(w_0)
=
0
$ does not occur under \Cref{ass:Laplacian-properties}.
First expand $\nabla_w \ENO_{\sd_t}(w_0)$ using Taylor's theorem in Lagrange form:
that is, there exists a $\zeta$ between $\sd_0$ and $\sd_\true$ such that
\begin{align}
	&
	\nabla_w \ENO_{\sd_t}(w_0)
	\label{eq:taylor_nabla_w_ENO}
	\\  &
	=
	\nabla_w \ENO_{\sd_0}(w_0)
	+
	\Bigl(
	\frac{\partial}{\partial \sd}
	\nabla_w \ENO_{\sd_0}(w_0)
	\Bigr)
	\cdot (\sd_\true-\sd_0)
	+
	\frac{1}{2}
	\Bigl(
	\frac{\partial^2}{\partial \sd^2}
	\nabla_w \ENO_{\zeta}(w_0)
	\Bigr)
	\cdot (\sd_\true-\sd_0)^2
	.
	\nonumber
\end{align}
By assumption $\nabla_w \ENO_{\sd_0}(w_0) = 0$.
Substituting this fact and \eqref{eq:taylor_nabla_w_ENO} into $C_1$, we find that
\begin{align}
	C_1
	&
	=
	\big(\nabla_w \ENO_{\sd_t}(w_0)\big)^T
	\frac{\partial}{\partial \sd}
	\nabla_w \ENO_{\sd_0}(w_0)
	\nonumber \\  &
	=
	{\Bigl(
		\frac{\partial}{\partial \sd}
		\nabla_w \ENO_{\sd_0}(w_0)
		\Bigr)^T
		\Bigl(
		\frac{\partial}{\partial \sd}
		\nabla_w \ENO_{\sd_0}(w_0)}
	\Bigr)
	\cdot (\sd_\true-\sd_0)
	\nonumber \\  &
	\phantom{=}
	+
	\frac{1}{2}
	\Bigl(
	\frac{\partial^2}{\partial \sd^2}
	\nabla_w \ENO_{\zeta}(w_0)
	\Bigr)^T
	\Bigl(
	\frac{\partial}{\partial \sd}
	\nabla_w \ENO_{\sd_0}(w_0)
	\Bigr)
	\cdot
	(\sd_\true-\sd_0)^2.
	\label{eq:non_zero_direction}
\end{align}
Notice that
$
(
\frac{\partial}{\partial \sd}
\nabla_w \ENO_{\sd_0}(w_0)
)^T
(
\frac{\partial}{\partial \sd}
\nabla_w \ENO_{\sd_0}(w_0))
=
\|\frac{\partial}{\partial \sd}
\nabla_w \ENO_{\sd_0}(w_0)\|_2
>
0
$
under \Cref{ass:Laplacian-properties}.
If now
\begin{align}
	\frac{1}{2}
	\Bigl(
	\frac{\partial^2}{\partial \sd^2}
	\nabla_w \ENO_{\zeta}(w_0)
	\Bigr)^T
	\Bigl(\frac{\partial}{\partial \sd}
	\nabla_w \ENO_{\sd_0}(w_0)\Bigr)=0,
	\label{eq: second degree term proof theorem 1}
\end{align}
then $C_1\neq 0$.
If \eqref{eq: second degree term proof theorem 1} does not hold, we will show that under \Cref{ass:Laplacian-properties} $C_1=0$ would lead to a contradiction.
Specifically, $C_1=0$ implies that we can rearrange \eqref{eq:non_zero_direction} to
\begin{align}
	&
	-
	\Bigl(
	\frac{\partial}{\partial \sd}
	\nabla_w \ENO_{\sd_0}(w_0)
	\Bigr)^T
	\Bigl(
	\frac{\partial}{\partial \sd}
	\nabla_w \ENO_{\sd_0}(w_0)
	\Bigr)
	\nonumber \\  &
	=
	\frac{1}{2}
	\Bigl(
	\frac{\partial^2}{\partial \sd^2}
	\nabla_w \ENO_{\zeta}(w_0)
	\Bigr)^T
	\Bigl(\frac{\partial}{\partial \sd}
	\nabla_w \ENO_{\sd_0}(w_0)\Bigr)
	\cdot (\sd_\true-\sd_0)
\end{align}
and thus to
\begin{align}
	\sd_\true - \sd_0
	&
	=
	\frac{
		-{\Bigl(
			\frac{\partial}{\partial \sd}
			\nabla_w \ENO_{\sd_0}(w_0)
			\Bigr)^T
			\Bigl(
			\frac{\partial}{\partial \sd}
			\nabla_w \ENO_{\sd_0}(w_0)}
		\Bigr)
	}{
		\frac{1}{2}
		\Bigl(
		\frac{\partial^2}{\partial \sd^2}
		\nabla_w \ENO_{\zeta}(w_0)
		\Bigr)^T
		\Bigl(\frac{\partial}{\partial \sd}
		\nabla_w \ENO_{\sd_0}(w_0)\Bigr)
	}
	\nonumber \\  &
	=
	\frac{-
		\|\frac{\partial}{\partial \sd}
		\nabla_w \ENO_{\sd_0}(w_0)\|_2
	}{
		\frac{1}{2}
		\Bigl(
		\frac{\partial^2}{\partial \sd^2}
		\nabla_w \ENO_{\zeta}(w_0)
		\Bigr)^T
		\Bigl(\frac{\partial}{\partial \sd}
		\nabla_w \ENO_{\sd_0}(w_0)\Bigr)
	}.
\end{align}
This however contradicts the rightmost inequality in \Cref{ass:Laplacian-properties}, which demands strictness.
We thus no longer need to consider the case $C_1=0$.

Now, given that $C_1\neq 0$, assume that $C_2\leq 0$.
Then $\alpha^2C_2\leq 0$.
If $C_1<0$, then by picking any $\alpha\in(-\infty,0)$ we have that $\alpha C_1>0$ and thus $\alpha^2C_2\leq 0 < \alpha C_1$.
Similarly, if $C_1>0$ we have for any $\alpha \in (0,\infty)$ that $\alpha C_1>0$ and again $\alpha^2C_2\leq 0 < \alpha C_1$.
Since \eqref{eq:cond_for_improvement} holds under these conditions, \eqref{eq:part_1_to_show} also holds.

Now assume that $C_2>0$.
We can then divide both sides of the inequality by $C_2$ without changing the direction of the inequality in $\alpha^2 C_2 < \alpha C_1$ to obtain $\alpha^2 < \alpha(C_1 / C_2)$.
Rewriting this to $\alpha^2-\alpha(C_1 / C_2)<0$ and factoring out $\alpha$ we find that equivalently
\begin{align}
	\alpha(\alpha-C_1 / C_2)<0.
	\label{eq:case_alpha_cs}
\end{align}
We can see that \eqref{eq:case_alpha_cs} holds:
\begin{itemize}
	\item for $\alpha\in ( 0, C_1 / C_2 )$ whenever $C_1 > 0$, because $\alpha$ is then positive and $\alpha - C_1 / C_2$ then negative;
	\item for $\alpha\in ( C_1 / C_2, 0 )$ whenever $C_1 < 0$, because $\alpha$ is then negative and $\alpha - C_1 / C_2$ then positive.
\end{itemize}

\Cref{table:alpha-values} summarizes the ranges for $\alpha$ for which we have proven that \eqref{eq:cond_for_improvement} (and thus \eqref{eq:part_1_to_show}) holds.

\begin{table}[ht]
	\centering
	\begin{tabular}{ c|cc }
		\toprule
		Cases      & $C_1<0$                     & $C_1>0$                     \\
		\midrule
		$C_2\leq0$ & $\alpha \in (-\infty,0)$    & $\alpha \in (0,\infty)$     \\
		$C_2>0$    & $\alpha \in (C_1 / C_2, 0)$ & $\alpha \in (0, C_1 / C_2)$ \\
		\bottomrule
	\end{tabular}
	\caption{Values of $\alpha$ for which \eqref{eq:cond_for_improvement} holds.}
	\label{table:alpha-values}
\end{table}

Note finally that if \eqref{eq:cond_for_improvement} holds for a specific $\alpha^\ast$, it also holds for all $\alpha$ between $0$ and $\alpha^\ast$.
That is it.
\QuodEratDemonstrandum

\subsubsection{Proof \texorpdfstring{of \Cref{thm:ft-guarantee}}{that GIFT works}}
\label{sec:Proof-that-GIFT-works}

The proof of \Cref{thm:FT-is-needed} shows that there exists a value $\alpha^\ast$ for which \eqref{eq:part_1_to_show} holds.
Moreover, it shows that \eqref{eq:part_1_to_show} is valid for all $\alpha$ within the ranges specified in \Cref{table:alpha-values}.
This means that for any case in \Cref{table:alpha-values}, given an $\alpha^\ast$ from its corresponding range, any $\alpha$ between $\alpha^\ast$ and $0$ also satisfies \eqref{eq:part_1_to_show}, i.e., for any $0 < |\alpha| < |\alpha^\ast|$, with $\alpha$ the same sign as $\alpha^\ast$.

To exploit this, \Cref{thm:ft-guarantee} assumes that the step size $\eta$ used in \Cref{alg:GIFT} is small enough.
Then, because the line search in \Cref{alg:GIFT} explores both the positive and the negative direction, we can assume without loss of generality that the sign of $\eta$ is such that $w_1^\mathrm{ideal}:=w_0-\eta\frac{\partial}{\partial \sd}\nabla_w \ENO_{\sd_0}(w_0)$ is an improvement over $w_0$, i.e., that $\ENO_{\sd_t}(w_1^\mathrm{ideal})<\ENO_{\sd_t}(w_0)$.

We now show that the updated point $w_1 := w_0 - \eta D^{[0]}(K_1, K_2)$ yields a lower value of the evaluation function $\mathrm{Eval}$ with high probability.
Because \Cref{lem:unbiased-sampling-D0} guarantees
$
D^{[0]}(K_1,K_2)
\xrightarrow{a.s.
}
\frac{\partial}{\partial \sd}
\nabla_w \ENO_{\sd_0}(w_0)
$
as $K_1,K_2\rightarrow\infty$, we have that $w_1\xrightarrow{a.s.} w_1^\mathrm{ideal}$ as $K_1,K_2\rightarrow\infty$.
The proof of \Cref{thm:FT-is-needed} shows that $\ENO_{\sd_t}(w_1^\mathrm{ideal})<\ENO_{\sd_t}(w_0)$, and because of the continuity of $\ENO_{\sd_t}$ we have that
$
\ENO_{\sd_t}(w_1)\xrightarrow{a.s.
}\ENO_{\sd_t}(w_1^\mathrm{ideal})
$
as $K_1,K_2\rightarrow\infty$.
Consequently, for sufficiently large $K_1,K_2$, we have that $\ENO_{\sd_t}(w_1)<\ENO_{\sd_t}(w_0)$ with high probability.

Assume now for a moment that \Cref{alg:GIFT} would have used $\ENO_{\sd_t}$ instead of the $\mathrm{Eval}$ function:
then, certainly, for sufficiently large $K_1, K_2$, \Cref{alg:GIFT} would have determined that $\ENO_{\sd_t}(w_1)<\ENO_{\sd_t}(w_0)$.
Moreover, the set in \cref{giftalg:selection} of \Cref{alg:GIFT} will then not be empty.
Accordingly, \Cref{alg:GIFT} would have found some $w_\ft$ such that $\ENO_{\sd_t}(w_\ft)<\ENO_{\sd_t}(w_0)$.

What therefore remains is to show that $\mathrm{Eval}$ approximates $\ENO_{\sd_t}$ accurately.
To do so, decompose the evaluations as follows:
\begin{align}
	&
	\mathrm{Eval}_{K_1,K_2}(w_0)-\mathrm{Eval}_{K_1,K_2}(w_1)
	\\  &
	=
	\Bigl[
	\mathrm{Eval}_{K_1,K_2}(w_0)
	-
	\ENO_{\sd_t}(w_0)
	+
	\ENO_{\sd_t}(w_0)
	\Bigr]
	-
	\Bigl[
	\mathrm{Eval}_{K_1,K_2}(w_1)
	-
	\ENO_{\sd_t}(w_1)
	+
	\ENO_{\sd_t}(w_1)
	\Bigr]
	\nonumber \\  &
	=
	\ENO_{\sd_t}(w_0)
	-
	\ENO_{\sd_t}(w_1)
	+
	\Bigl[
	\mathrm{Eval}_{K_1,K_2}(w_0)
	-
	\ENO_{\sd_t}(w_0)
	\Bigr]
	-
	\Bigl[
	\mathrm{Eval}_{K_1,K_2}(w_1)
	-
	\ENO_{\sd_t}(w_1)
	\Bigr]
	.
	\nonumber
\end{align}
Let $\eps := \ENO_{\sd_t}(w_0)
-
\ENO_{\sd_t}(w_1)>0$, which is---in the $K_1,K_2\rightarrow\infty$ limit---strictly positive by our choice of $\eta$.
We will now show that the approximation errors
\begin{align}
	& |\mathrm{Eval}_{K_1,K_2}(w_0)
	-
	\ENO_{\sd_t}(w_0)|
	\xrightarrow[K_1,K_2\rightarrow\infty]{a.s.
	}
	0
	\\
	& \textnormal{and}
	\quad
	|\mathrm{Eval}_{K_1,K_2}(w_1)
	-
	\ENO_{\sd_t}(w_1)
	|
	\xrightarrow[K_1,K_2\rightarrow\infty]{a.s.}
	0
	.
	\label{eq:approx_error_goes_to_zero_thm2}
\end{align}

For $i=0,1$, the two expressions
$
|\mathrm{Eval}_{K_1,K_2}(w_i)
-
\ENO_{\sd_t}(w_i)|
$
have the same structure.
Let $w\in\{w_0,w_1\}$ therefore be arbitrary.
Define $Z_{K_2}(x,y)= \frac{1}{K_2}\sum_{m=1}^{K_2}
(y-M_{\sd_\true}(x,w,\mathbf{N}^{ \{m\} }))^2$.
The proof now requires these equalities:
\begin{align}
	\lim_{K_1,K_2\rightarrow \infty}
	\mathrm{Eval}_{K_1,K_2}(w)
	& \eqcom{a}=
	\lim_{K_1,K_2\rightarrow \infty}\frac{1}{K_1}\sum_{k=1}^{K_1}
	Z_{K_2}(x^{\{k\}}, y^{\{k\}})
	\\  &
	\eqcom{a.s.
		+ b}=
	\mathbb{E}_{(X,Y)}
	\big[\mathbb{E}_\mathbf{N}
	[(Y-M_{\sd_\true}(X,w,\mathbf{N}^{ \{m\} }))^2]\big]
	\\
	& \eqcom{c}=
	\ENO_{\sd_t}(w),
	\label{eq:final_chain_of_equations_thm2}
\end{align}
where $(X,Y) \sim \mu$ (recall Assumption \ref{item:A3}).
Here, {(a)} uses the definition of $\mathrm{Eval}$.
{For (b)} we invoke that \Cref{lem:Square-Integrability-of-the-Forward-and-Backward-Passes} ensures finite first and second moment of $(y-M_\sd(x,w,\mathbf{N}))^2$.
Therefore, for the random variables $(X,Y)$ and $\mathbf{N}$ and the function $f((X,Y),\mathbf{N})=(Y-M_\sd(X,w,\mathbf{N}))^2$, the conditions of \Cref{lem:Hierarchical-sampling-is-unbiased} are met, which gives the almost sure convergence of (b).
Lastly, (c) is simply the definition of $\ENO_{\sd_t}$.

Finally, \eqref{eq:final_chain_of_equations_thm2}
implies \eqref{eq:approx_error_goes_to_zero_thm2}.
\QuodEratDemonstrandum
\section{Simulation study on \GIFT}
\label{sec:simulations}

To evaluate the effectiveness of \Cref{alg:GIFT}, we conducted experiments on two \gls{NN} architectures, a \emph{shallower \gls{NN}}, with layer dimensions $784 \text{--} 500 \text{--} 100 \text{--} 100 \text{--} 10$, and a \emph{deeper \gls{NN}} with dimensions $784 \text{--} 500 \text{--} 250 \text{--} 250 \text{--} 100 \text{--} 50 \text{--} 10$.
Both architectures use hyperbolic tangent activation functions and were trained on the \gls{MNIST} classification task. The training algorithm closely follows \eqref{eqn:projected-sgd}, with two key deviations: we omit the projection step, and updates are performed after processing batches of $(x^{\{k\}}, y^{\{k\}})$ pairs rather than individual samples.

All experiments were implemented using a modified \texttt{PyTorch} framework and executed on systems equipped with \texttt{NVIDIA Tesla V100-PCIE-16GB} and \texttt{NVIDIA GeForce RTX 2080 Ti} GPUs.

First, we execute \GIFT to obtain \finetuning weights and evaluate the improvements on the training objective, i.e., on the training data.
We analyze the results separately for the deeper network (\Cref{sec:simulations_deeper_train_data}), the shallower network (\Cref{sec:simulations_shallower_train_data}), and cases where the \gls{AWGN} assumption is violated in the deeper network (\Cref{sec:simulations_violation_train_data}).
Finally, in \Cref{sec:simulations_violation_test_data}, we assess whether the fine-tuned weights $w_\ft$ also yield better performance on the unseen holdout/test data.
In each experiment, we trained 10 networks per training noise standard deviation, each initialized with a different random seed.

\subsection{\GIFT Performance on the Deeper Network}\label{sec:simulations_deeper_train_data}

\Cref{fig:performance_gift} illustrates the performance improvements obtained by applying \GIFT to the deeper \gls{NN}.

\begin{figure}[h]
	\centering
	\includegraphics[width=0.95\linewidth]{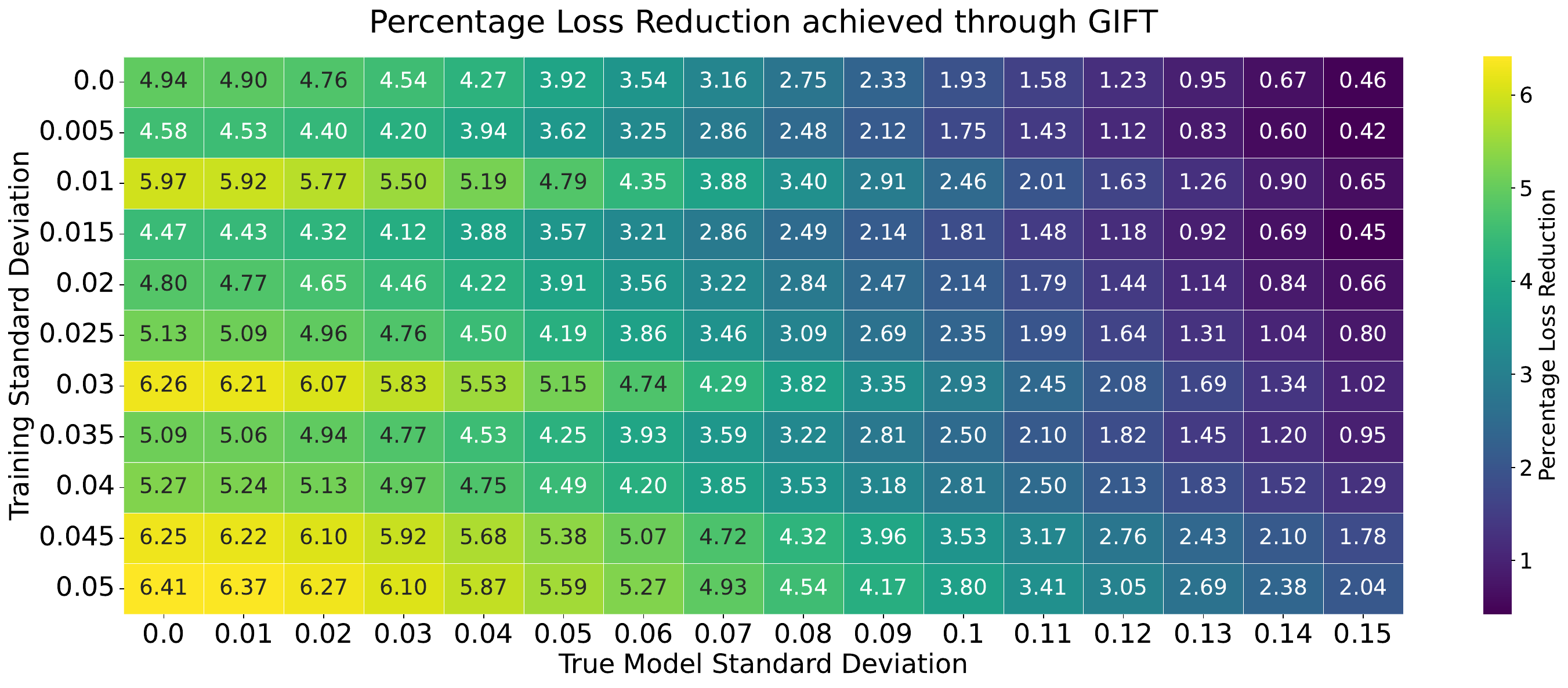}
	\includegraphics[width=0.95\linewidth]{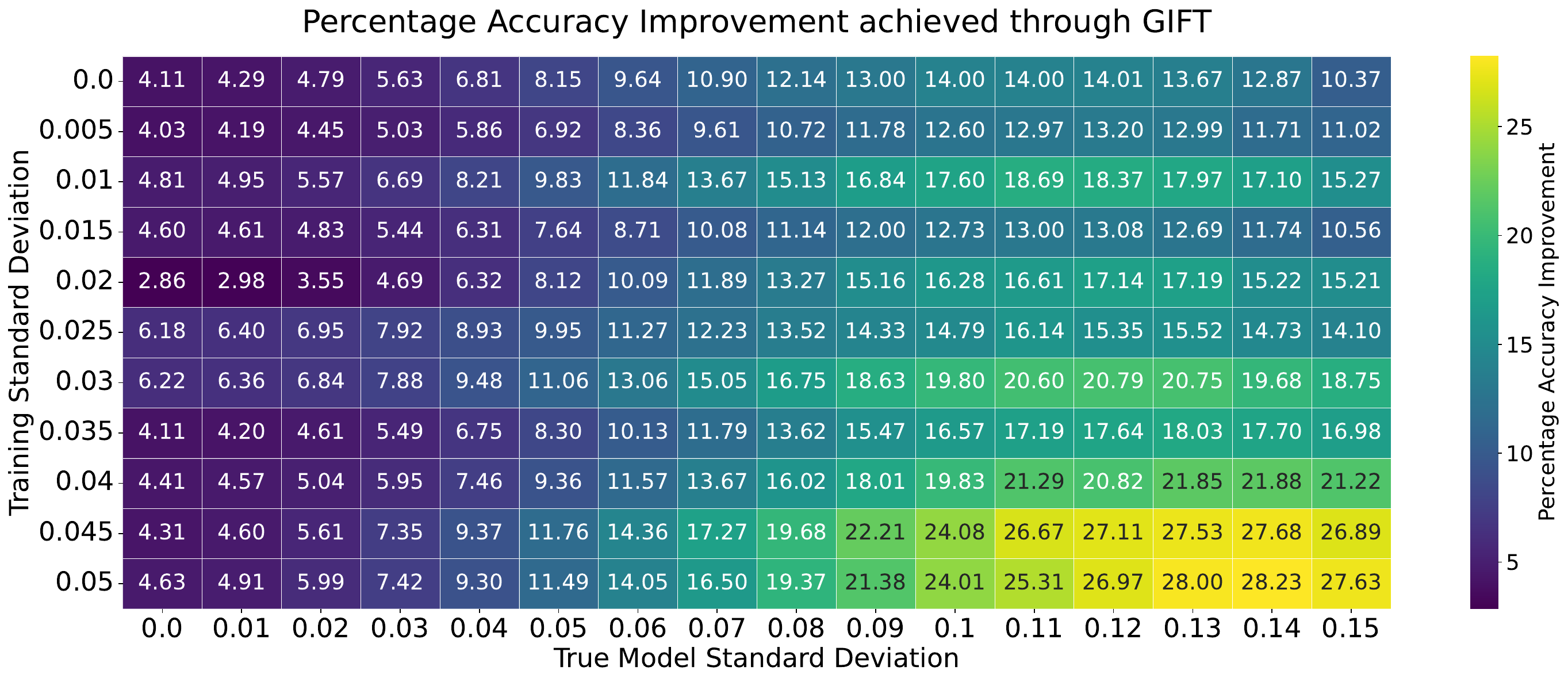}
	\caption{Percentage improvements in loss (\gls{MSE}) and accuracy achieved with \GIFT (\Cref{alg:GIFT}) compared to the baseline (no \finetuning) for the deeper \gls{NN}, under varying training ($\sd_0$) and true ($\sd_\true$) noise standard deviations.}
	\label{fig:performance_gift}
\end{figure}

Observe in \Cref{fig:performance_gift} that both the loss and accuracy percentage improvements increase as the magnitude of the injected noise (i.e., the variance of the \gls{AWGN} used during training) increases.
This positive correlation holds true irrespective of whether the subsequent true noise level is higher or lower than the noise level assumed during training.
However, the trends in loss improvement and accuracy improvement exhibit distinct behaviors relative to the true noise level: the largest percentage reductions in the \gls{MSE}-loss occur when the true noise level is low, whereas the greatest accuracy improvements occur when the true noise level is high.

A plausible explanation for this observed divergence arises from the fundamental differences between the accuracy and \gls{MSE} metrics.
Accuracy is threshold-based and hence, in high-noise conditions, even slight adjustments to decision boundaries can correct misclassifications, resulting in substantial relative improvements.

On the other hand, \gls{MSE} measures the average squared deviation and naturally increases under higher noise due to increased random fluctuations.
Consequently, even if the absolute reduction in loss remains constant across noise levels, the relative (percentage) improvement appears smaller when the baseline loss is higher.

The opposite reasoning applies for accuracy.
As noise increases, baseline accuracy deteriorates, amplifying the relative importance of absolute improvements.
Therefore, the largest relative accuracy improvements tend to appear in high-noise environments, even when the relative improvement in loss appears diminished.

To further clarify this interpretation, we present the absolute (rather than relative) improvements in \Cref{fig:abs_performance_gift} in the appendix.

Notably, even when $\sd_0 = \sd_\true$, following the gradient direction $D^{[0]}$ in \Cref{alg:GIFT} still results in substantial loss improvements.
This occurs because $D^{[0]}$, as computed via\eqref{eq: d ds nablab ENO} and \eqref{eqn:d-ds-nablaW-ENO}, reflects the gradient with respect to the full (available training) data distribution. In contrast, the original training procedure relies on iterative updates from small mini-batches; \emph{cf.}\ the introduction to this section. \GIFT applies a structured gradient-informed adjustment that follows the gradient direction as calculated for the whole data distribution.
This approach has some similarities to line search methods in optimization but is applied in the \finetuning context, leveraging two descent directions of the broader data distribution and the noise structure (the one of the gradient and then how the gradient changes with changing noise, the latter is less important in the $\sd_0 = \sd_\true$ case), rather than mini-batch updates.
Therefore, in the $\sd_0 = \sd_\true$ case, \GIFT refines the model parameters beyond conventional training and thus achieves improvements.

\subsection{\GIFT Performance on the Shallower Network}
\label{sec:simulations_shallower_train_data}

\Cref{fig:performance_gift_shallow} shows the performance improvements achieved by \GIFT on the shallower network.

\begin{figure}[h]
	\centering
	\includegraphics[width=0.495\linewidth]{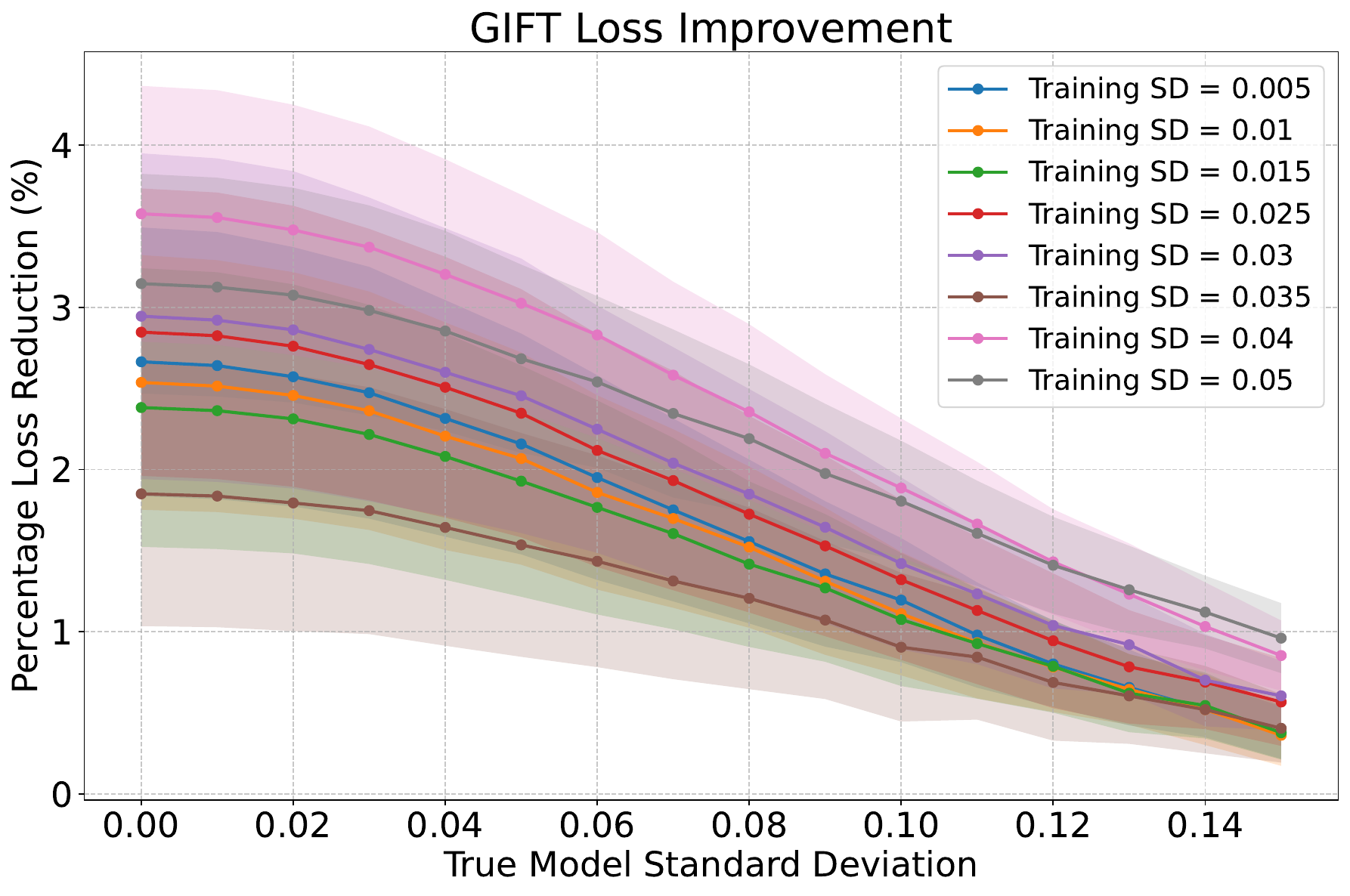}
	\includegraphics[width=0.495\linewidth]{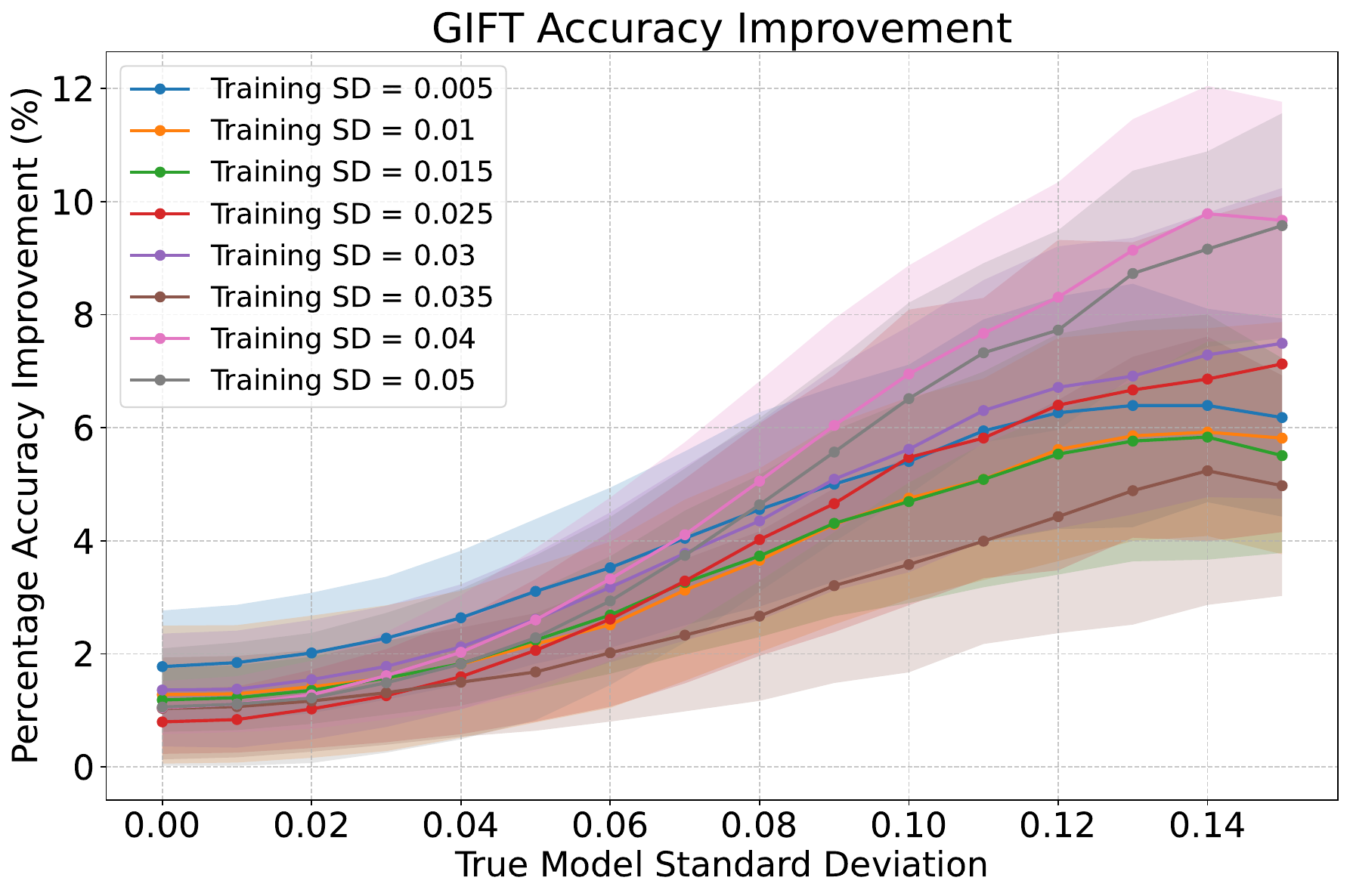}
	\caption{Percentage improvements in loss (\gls{MSE}) and accuracy achieved with \GIFT (\Cref{alg:GIFT}) compared to the baseline (no \finetuning) for the \textbf{shallower network}, under varying training ($\sd_0$) and true ($\sd_\true$) noise standard deviations.
	}
	\label{fig:performance_gift_shallow}
\end{figure}

Observe that the performance improvements depicted in \Cref{fig:performance_gift_shallow} are generally lower and exhibit less consistent trends with respect to changes in $\sd_0$ or $\sd_\true$ compared to what we saw for the deeper network in \Cref{sec:simulations_deeper_train_data}.
Although we still observe the general tendency seen in \Cref{fig:performance_gift}, where lower true noise results in greater relative loss improvements, and higher true noise yields higher relative accuracy gains, these patterns are significantly less pronounced.
This suggests that misspecification may have a less severe impact on the shallower network, reducing the necessity and consequently the potential benefit of \finetuning adjustments.

A possible explanation is that deeper networks inherently propagate and compound the same noise level across multiple layers, leading to cascaded errors.
Therefore, the deeper network is more severely impacted by noise misspecification, providing greater scope and necessity for corrections, which \GIFT successfully addresses.
Conversely, the shallower network accumulates fewer compounded noise effects, limiting the overall benefit attainable from \finetuning.

\subsection{\GIFT's performance when the noise is not \texorpdfstring{\gls{AWGN}}{AWGN}}
\label{sec:simulations_violation_train_data}

To test the robustness of \GIFT under different noise assumptions, we replaced \gls{AWGN} with the following noise distributions:
\begin{itemize}
	\item[a)] Uniform noise $\mathcal{U}(-\sd_\true, \sd_\true)$;
	\item[b)] Multiplicative Gaussian noise with standard deviation $\sd_\true$;
	\item[c)] Laplace-distributed noise with parameters $0$ and $\sd_\true$.
\end{itemize}
Performance results for these scenarios are presented in \Cref{fig:performance_gift_violated_noise_assumption}.
The top row shows the loss and the bottom row the accuracy.

\begin{figure}[h]
	\centering
	\includegraphics[width=0.33\linewidth]{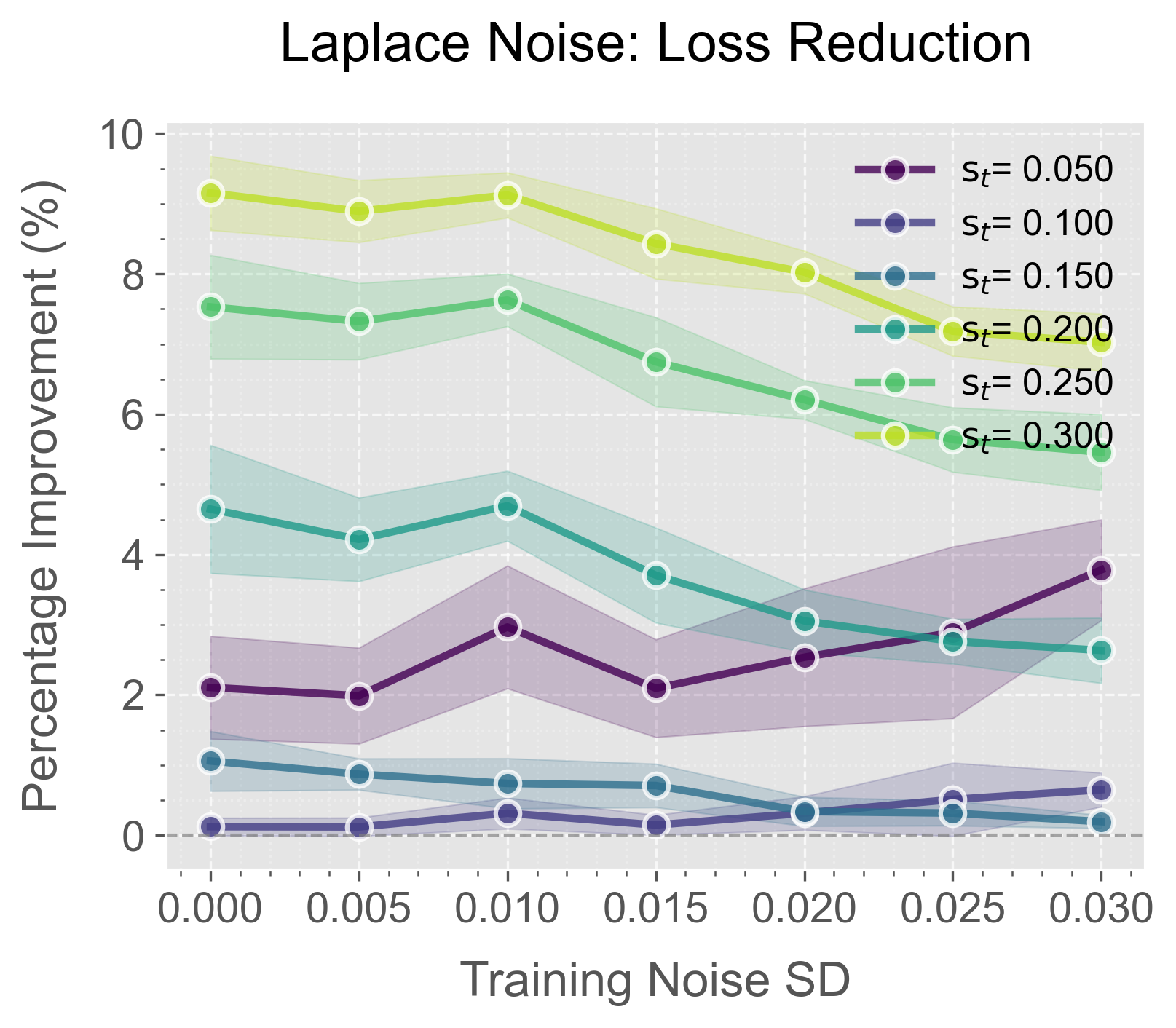}~\includegraphics[width=0.33\linewidth]{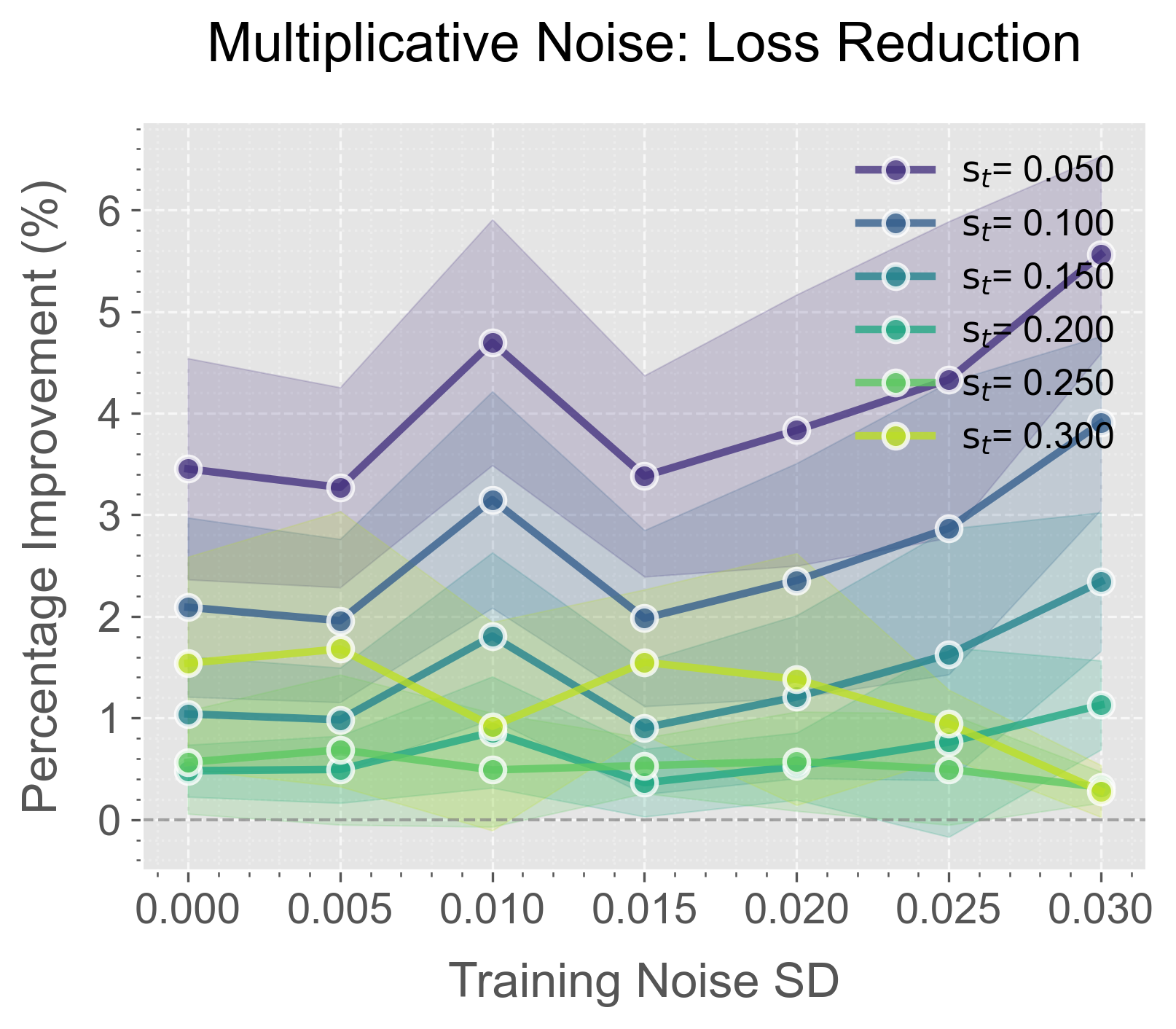}~\includegraphics[width=0.33\linewidth]{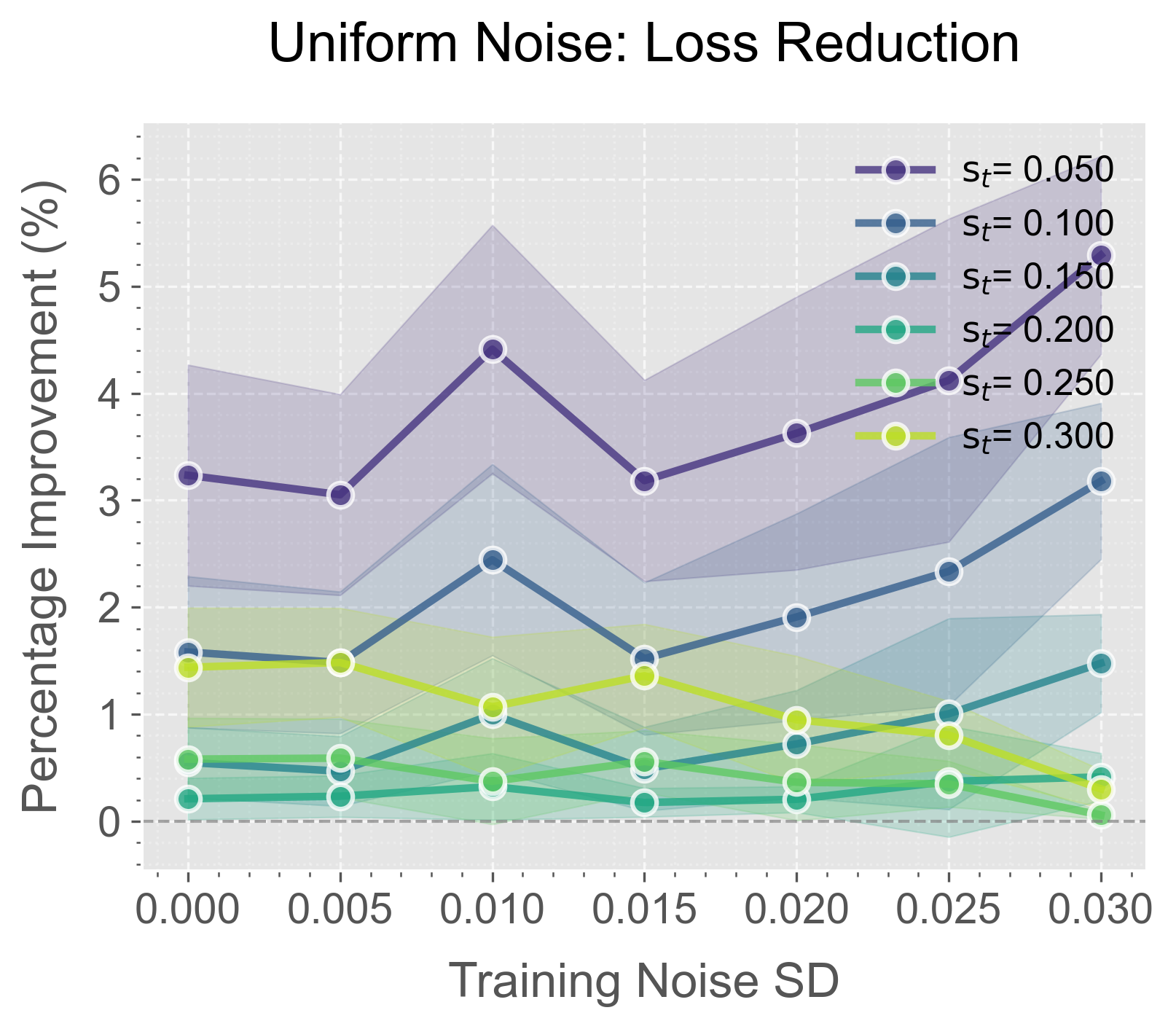}\\
	\includegraphics[width=0.33\linewidth]{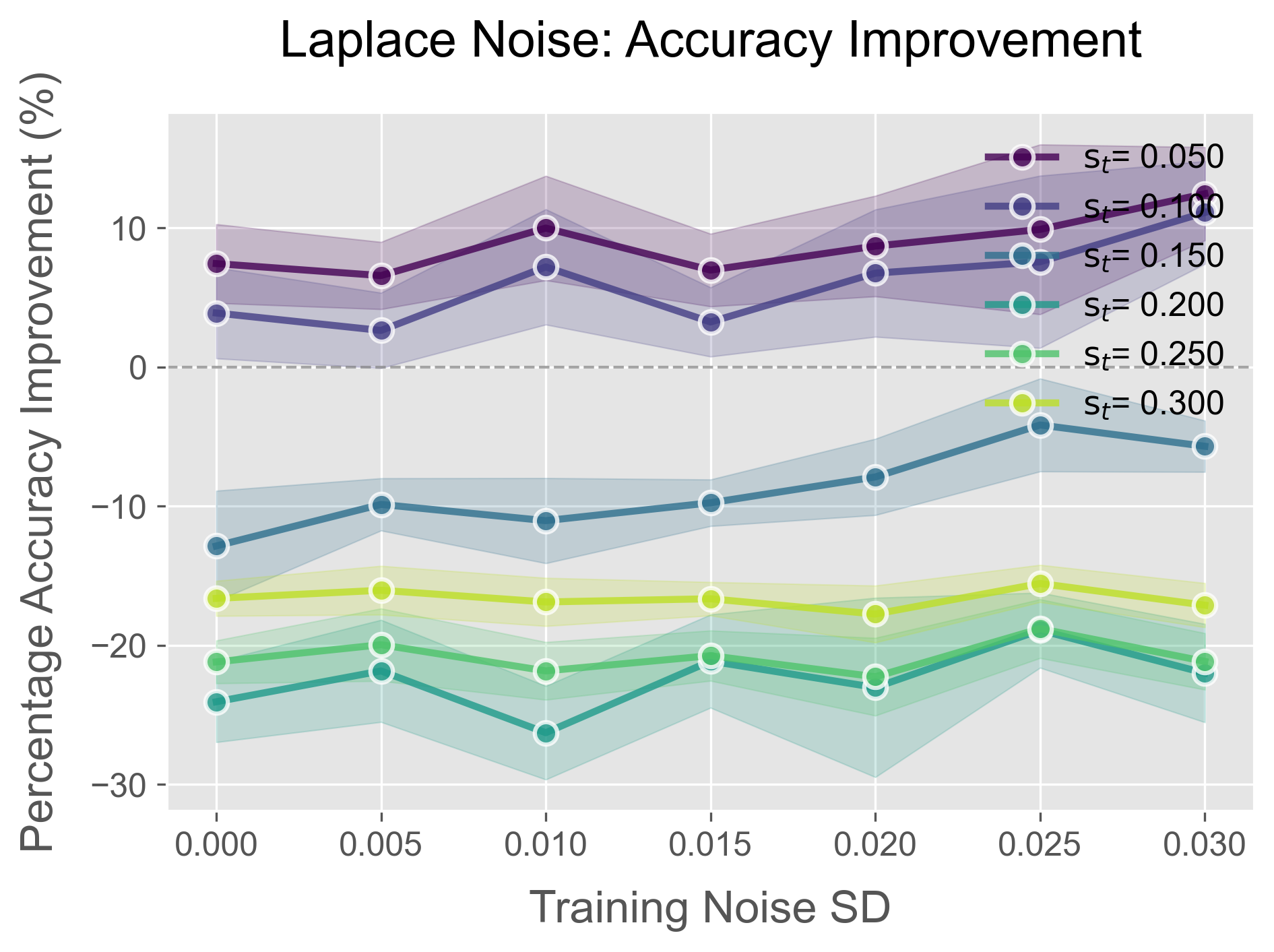}~\includegraphics[width=0.33\linewidth]{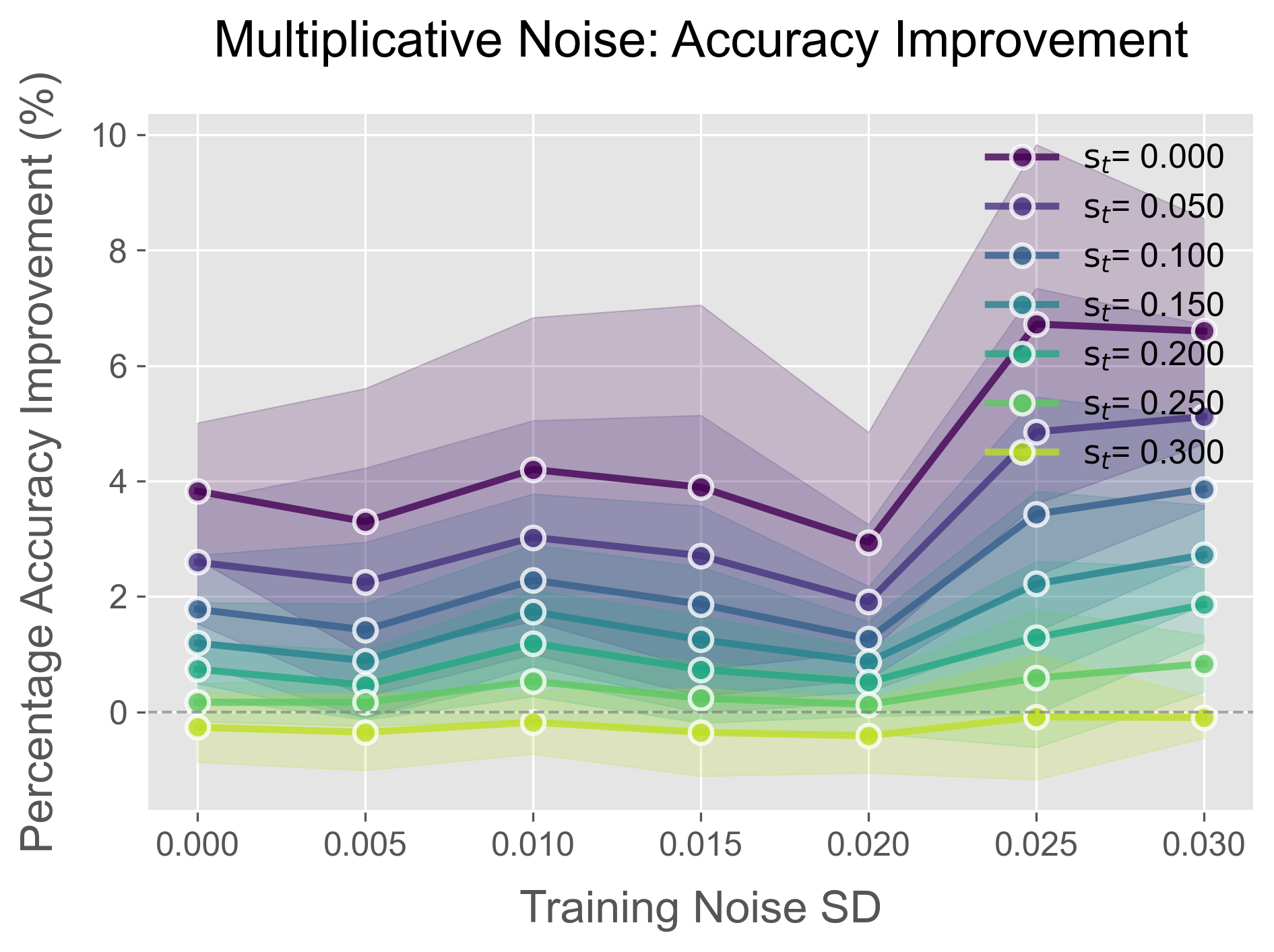}~\includegraphics[width=0.33\linewidth]{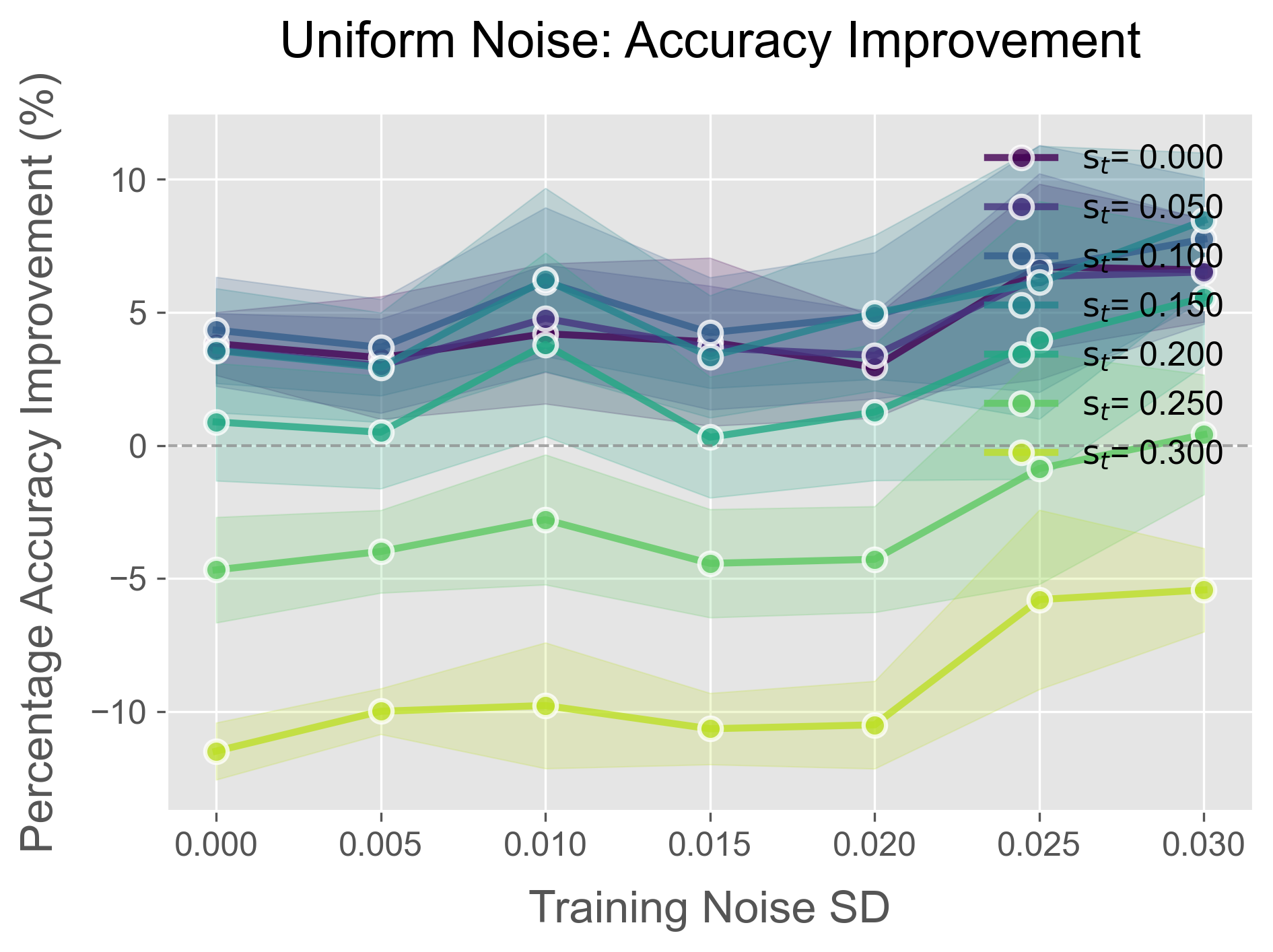}
	\caption{Percentage improvements in loss (top) and accuracy (bottom) achieved by \Cref{alg:GIFT} compared to the baseline (no \finetuning) for the deeper network under violations of the \gls{AWGN} assumption.
		Results are shown for Laplace-distributed noise (left), multiplicative Gaussian noise (middle), and uniformly distributed noise (right).
		Confidence intervals represent $95\%$.
	}
	\label{fig:performance_gift_violated_noise_assumption}
\end{figure}

Because the selection step in \GIFT (\cref{giftalg:selection} in \Cref{alg:GIFT}) includes the initial weights $w_0$ as a baseline, the improvement in loss is inherently nonnegative.
This guarantees that the fine-tuned model will not perform worse than the initial one, which is clearly observed in \Cref{fig:performance_gift_violated_noise_assumption}.

Accuracy improvements (bottom row), however, require a more nuanced interpretation.
In several configurations, percentage accuracy gains of approximately $10\%$ are achievable.
Yet, caution is needed, as in certain extreme cases of double misspecification (both variance and noise distribution deviating significantly from assumptions), accuracy can degrade substantially—exceeding $25\%$ deterioration.
Nonetheless, provided that the double misspecification remains moderate, \GIFT consistently delivers meaningful improvements.

These results demonstrate that \GIFT exhibits notable robustness even when the \gls{AWGN} assumption is violated.
Thus, the proposed algorithm can effectively be applied in physical \glspl{ONN} systems, which typically exhibit more complex or mixed noise characteristics beyond the idealized Gaussian assumptions.

\subsection{Performance of \GIFT on unseen test data}
\label{sec:simulations_violation_test_data}

\Cref{fig:performance_gift_test_data} shows the improvements in both loss and accuracy achieved on unseen test data using the fine-tuned weights $w_\ft$, which were identified by \GIFT using the training data (i.e., the same weights as used in \Cref{fig:performance_gift}).

\begin{figure}[h]
	\centering
	\includegraphics[width=0.49\linewidth]{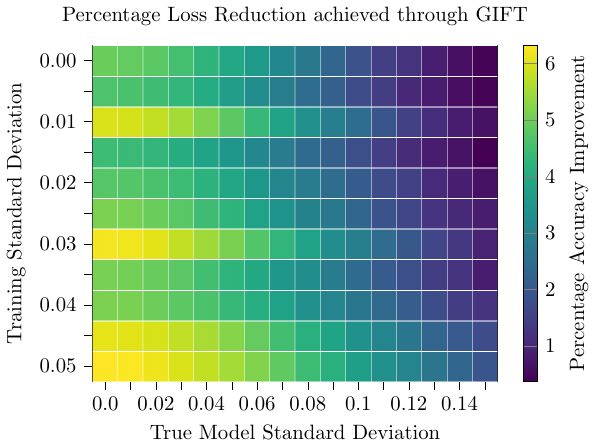}
	\includegraphics[width=0.49\linewidth]{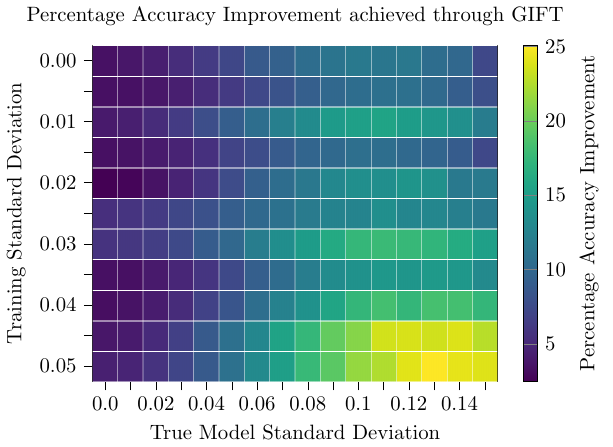}
	\caption{Percentage improvements in loss and accuracy achieved by \GIFT (\Cref{alg:GIFT}) compared to baseline (no \finetuning) for the deeper network on unseen test data under varying noise standard deviations.}
	\label{fig:performance_gift_test_data}
\end{figure}

Observe from \Cref{fig:performance_gift_test_data} that the observed performance enhancements closely resemble those seen on the training set (see \Cref{fig:performance_gift}).
This similarity indicates that the improvements from \GIFT are not merely the result of overfitting or an artifact of the gradient direction computed over the entire training set.
Rather, these results provide strong evidence that the improvements found by \GIFT are genuine and generalize well to previously unseen data.

Moreover, even under violated \gls{AWGN} assumptions, as depicted in \Cref{fig:performance_gift_test_data_violated_awgn} below, \GIFT consistently leads to performance improvements on unseen data.
These gains remain comparable to those observed on the training set (see \Cref{fig:performance_gift_violated_noise_assumption}), further highlighting the robustness and practical utility of \GIFT across a range of noise conditions.

\begin{figure}[h]
	\centering
	\includegraphics[width=0.49\linewidth]{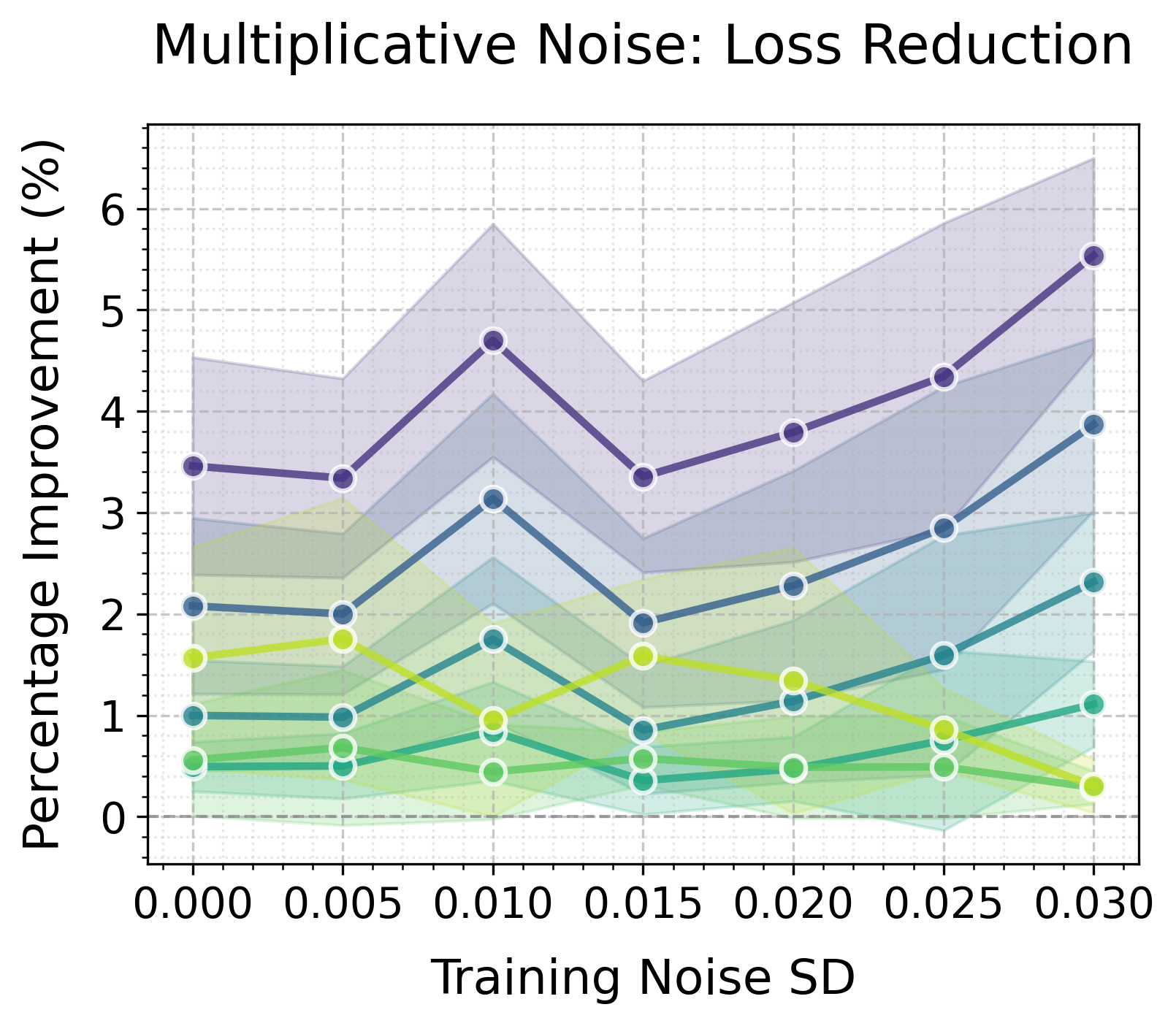}
	\includegraphics[width=0.49\linewidth]{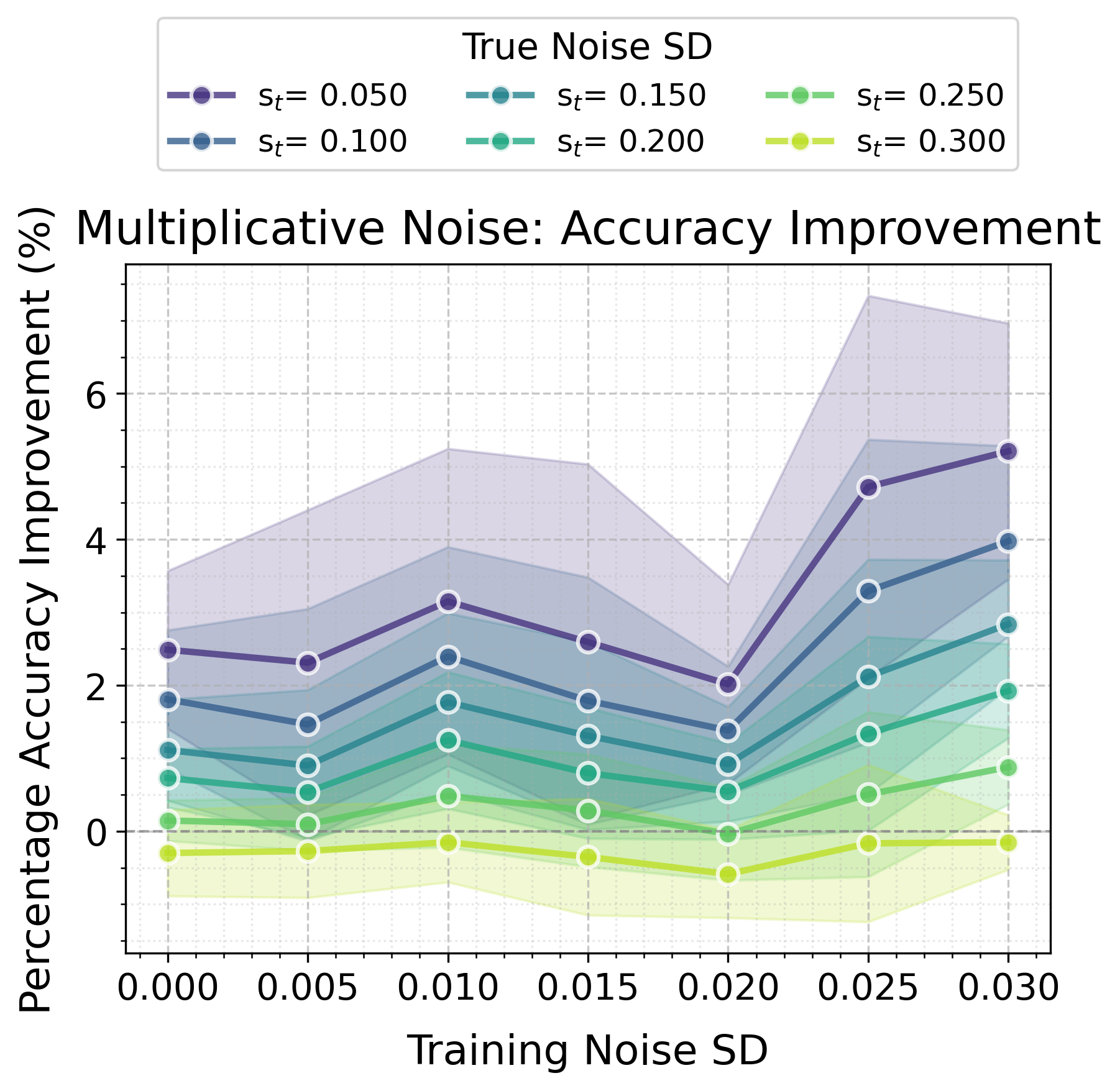}
	\caption{Percentage improvements in loss and accuracy achieved \Cref{alg:GIFT} compared to baseline (no \finetuning) for the deeper network on unseen test data under a \emph{violated \gls{AWGN} assumption} (multiplicative noise).
		Confidence intervals are set at $95\%$.
	}
	\label{fig:performance_gift_test_data_violated_awgn}
\end{figure}
\section{Conclusion}
\label{sec:Conclusion}

This work introduces \GIFT---\Cref{alg:GIFT}---which is a lightweight and effective method for mitigating the impact of misspecification in \emph{ex situ}-trained \glspl{ONN} by \finetuning \emph{in situ}.
\GIFT refines trained models using precomputed information that captures the interplay between noise and network structure, relying solely on on-chip inference to adjust network parameters.
As a result, \GIFT operates without requiring on-chip retraining, which is either infeasible in practical settings or requires complex interfaces.
This makes \GIFT an experimentally simple yet effective solution for adapting models to unpredictable noisy real-world conditions.
Our theoretical and simulation results demonstrate its efficacy.

In short, the key contributions of this paper are:
\begin{itemize}
	\item
	A rigorous theoretical framework establishing the conditions under which \GIFT improves \glspl{ONN} (\Crefrange{sec:Preliminaries}{sec:Analysis-of-GIFT}).
	We provide a formal analysis of noise-aware \gls{ONN} training and theoretically prove that \GIFT \Cref{alg:GIFT} leads to improvements (\Cref{thm:FT-is-needed,thm:ft-guarantee}).
	\item
	Empirical insights into how the magnitude of noise misspecification influences \finetuning success (\Cref{sec:simulations}), including observed accuracy improvements of up to $28\%$ on \gls{MNIST} classification.
	Our results show that deeper \glspl{ONN} benefit most from \GIFT.
\end{itemize}

Beyond its demonstrated performance gains, \GIFT also showed robustness across various noise types, including Laplace, multiplicative Gaussian, and uniform noise; despite it having been designed under the \gls{AWGN} assumption.
This suggests that \GIFT may generalize well beyond the original training assumptions.
We further hypothesize that \GIFT may also compensate for mapping inaccuracies and fabrication--induced imperfections, making it a promising direction for practical \gls{ONN} deployment.

Taken together, the theoretical and experimental findings presented here introduce \GIFT as the first \emph{in situ} \finetuning approach for \emph{ex situ} trained \glspl{ONN} with a proven ability to improve models trained under misspecified \gls{AWGN} noise.
Moreover, \GIFT holds strong potential to correct broader classes of implementation errors, including mapping distortions and hardware-specific deviations.
As such, \GIFT represents a valuable step toward scalable, reliable integration of \glspl{ONN} in real-world applications.

While this study establishes the foundational efficacy of \GIFT, future work can explore possible extensions to other noise models.
For example, one may consider more complex and realistic noise models such as shot noise, Johnson-Nyquist noise, and device-specific distortions in \glspl{ONN}.
Furthermore, showcasing \GIFT's usefullness in a system-level simulator (e.g. \hspace{-10pt} \emph{VPIphotonics}) or in a proper hardware implementations can further underscore the effectiveness of \GIFT.

\section*{Acknowledgments}

This research was supported by the European Union's Horizon 2020 research and innovation programme under the Marie Sklodowska-Curie grant agreement no.~945045, and by the NWO Gravitation project NETWORKS under grant no.~024.002.003.

\bibliography{Bib}{}
\bibliographystyle{plain}

\appendix
\section*{Appendix}
\section{Technical details}
\label{sec:Technical-details}

\subsection{Backpropagation gives derivatives}
\label{sec:backpropagation_proof}

We now prove \Cref{eqn:backpropagation-for-W,eqn:backpropagation-for-b}.
Similar calculations can be found in e.g.\ \cite[Chapter 6.5]{Goodfellow-et-al-2016}, but these are not specific to our model.

Let $(x,y)$ be an input--output pair.
For notational convenience, let
\begin{equation}
   E
   :=
   (
   y
   -
   M_s({x} , w, \mathbf{N})
   )^2
   .
\end{equation}
Recall furthermore that for $\ell \in [L]$, the scalar--matrix and scalar--vector derivatives ${\partial E}/{\partial W^{(\ell)}}$ and ${\partial E}/{\partial b^{(\ell)}}$
are given by
\begin{align}
   \nabla_{W^{(\ell)}}
   E
   =
   \frac{\partial E}{\partial W^{(\ell)}}
   :=
   \begin{pmatrix}
      \frac{\mathrm{d} E}{\mathrm{d} W^{(\ell)}_{11}}     & \frac{\mathrm{d} E}{\mathrm{d} W^{(\ell)}_{12}}     & \dots & \frac{\mathrm{d} E}{\mathrm{d} W^{(\ell)}_{1d_{L}}}       \\
      \frac{\mathrm{d} E}{\mathrm{d} W^{(\ell)}_{11}}     & \frac{\mathrm{d} E}{\mathrm{d} W^{(\ell)}_{12}}     & \dots & \frac{\mathrm{d} E}{\mathrm{d} W^{(\ell)}_{1d_{L-1}}}     \\
      \dots                                               & \dots                                               & \dots & \dots                                                     \\
      \frac{\mathrm{d} E}{\mathrm{d} W^{(\ell)}_{d_{L}1}} & \frac{\mathrm{d} E}{\mathrm{d} W^{(\ell)}_{d_{L}2}} & \dots & \frac{\mathrm{d} E}{\mathrm{d} W^{(\ell)}_{d_{L}d_{L-1}}} \\
   \end{pmatrix}
\end{align}
and
\begin{align}
   \nabla_{b^{(\ell)}}
   E
   = \frac{\partial E}{\partial b^{(L)}}
   :=
   \begin{pmatrix}
      \frac{\mathrm{d}
      E}{\mathrm{d} b^{(L)}_{d_{1}}} \\
      \vdots                         \\
      \frac{\mathrm{d} E}{\mathrm{d} b^{(L)}_{d_{L}}}
   \end{pmatrix}
   .
   \label{eq:definitions_nabla}
\end{align}

Now, to calculate $\nabla_w E$, we start with the gradients with respect to the last layer's parameters ${W^{(L)}}$ and ${b^{(L)}}$.
Using (a) the chain rule, we find that
\begin{align}
   \frac{\partial E}{\partial W^{(L)}}
    &
   \eqcom{a}
   =
   -2
   (
   y
   -
   M_s({x} , w, \mathbf{N})
   )
   \frac{\partial M_s({x} , w, \mathbf{N})}{ \partial W^{(L)}}
   \nonumber \\       &
           \eqcom{\ref{def:matrix-RL}}
   =
           -2
           \errorvar^{(L)}
   \frac{\partial M_s({x} , w, \mathbf{N})}{ \partial W^{(L)}}
   \nonumber \\       &
           \eqcom{\ref{def:matrix-AL}}
   =
           -2
           \errorvar^{(L)}
   \frac{\partial \big( W^{(L)}
      A^{(L-1)} + b^{(L)} + N^{(L)} \big)}{ \partial W^{(L)}}
   =
           -2
           \errorvar^{(L)}
   \big(A^{(L-1)}\big)^T
           .
\end{align}
Similarly,
\begin{align}
   \frac{\partial E}{\partial b^{(L)}}
    &
   \eqcom{a}
   =
   -2
   (
   y
   -
   M_s({x} , w, \mathbf{N})
   )
   \frac{\partial M_s({x} , w, \mathbf{N})}{ \partial b^{(L)}}
   \nonumber \\       &
           \eqcom{\ref{def:matrix-RL}}=
           -2
           \errorvar^{(L)}
   \frac{\partial M_s({x} , w, \mathbf{N})}{ \partial b^{(L)}}
   \nonumber \\       &
           \eqcom{\ref{def:matrix-AL}}
   =
           -2
           \errorvar^{(L)}
   \frac{\partial( W^{(L)}
      A^{(L-1)} + b^{(L)} + N^{\mathrm{w},(L)})}{ \partial b^{(L)}}
   =
           -2
           \errorvar^{(L)}
   .
\end{align}

For the parameters in the second--to--last layer, $W^{(L-1)}$ and $b^{(L-1)}$, we find using again (a) the chain rule that
\begin{align}
    &
   \frac{\partial E}{\partial W^{(L-1)}}
   \nonumber \\          &
              \eqcom{a}
   =
              -2
              (
              y
              -
              M_s({x} , w, \mathbf{N})
              )
              \frac{\partial M_s({x} , w, \mathbf{N})}{ \partial W^{(L-1)}}
   \nonumber \\         &
             \eqcom{\ref{def:matrix-RL}}
   =
             -2
             \errorvar^{(L)}
   \frac{\partial M_s({x} , w, \mathbf{N})}{ \partial W^{(L-1)}}
   \nonumber \\         &
             \eqcom{\ref{def:matrix-AL}}
   =
             -2
             \errorvar^{(L)}
   \frac{\partial( W^{(L)}
      A^{(L-1)} + b^{(L)} + N^{\mathrm{w},(L)})}{ \partial W^{(L-1)}}
   \nonumber \\  &
      \eqcom{\ref{def:matrices-Ai}}
   =
      -2
      \errorvar^{(L)}
   \frac{\partial\bigl( W^{(L)} \bigl(\sigma(W^{(L-1)}A^{(L-2)} + b^{(L-1)}+ N^{\mathrm{w},(L-1)}) + N^{\mathrm{a},(L-1)} \bigr) \bigr) }{ \partial W^{(L-1)}}
   \nonumber \\  &
      \eqcom{a}
   =
      -2
      \big(W^{(L)}\big)^T
      \errorvar^{(L)}
   \frac{\partial\bigl( \sigma \bigl(W^{(L-1)}A^{(L-2)} + b^{(L-1)}+ N^{\mathrm{w},(L-1)}\bigr) + N^{\mathrm{a},(L-1)} \bigr) }{ \partial W^{(L-1)}}
   \nonumber \\  &
      =
      -2
      \big[
      \big(W^{(L)}\big)^T
      \errorvar^{(L)}
   \odot
      \sigma^\prime
      \big(
      W^{(L-1)}A^{(L-2)} + b^{(L-1)}+ N^{\mathrm{w},(L-1)}
   \big)
      \big]
   \nonumber \\  &
      \phantom{= -2}
   \frac{\partial\big( W^{(L-1)}A^{(L-2)} + b^{(L-1)}+ N^{\mathrm{w},(L-1)} + N^{\mathrm{a},(L-1)} \big) }{ \partial W^{(L-1)}}
   \nonumber \\  &
      \eqcom{\ref{def:matrices-Ri}}
   =
      \errorvar^{(L-1)}
   \frac{\partial\big( W^{(L-1)}A^{(L-2)} + b^{(L-1)}+ N^{\mathrm{w},(L-1)} + N^{\mathrm{a},(L-1)} \big) }{ \partial W^{(L-1)}}
   \nonumber \\  &
      \eqcom{\ref{def:matrices-Ai}}
   =
      -2
      \errorvar^{(L-1)}
   \big(A^{(L-2)}\big)^T
      .
\end{align}
Here, $\odot$ denotes the Hadamard product.
And again, similarly,
\begin{align}
   \frac{\partial E}{\partial b^{(L-1)}}
   =
   -2
   \errorvar^{(L-1)}
   .
\end{align}

The pattern repeats itself as we go down layers $\ell = L, L-1, L-2, \dots, 1$.
That is, for $\ell \in [L-1]$, we find that
\begin{equation}
   \frac{\partial E}{\partial W^{(\ell)}}
   =
   -2
   \errorvar^{(\ell)}
   \big(A^{(\ell-1)}\big)^T
   \quad
   \textnormal{and}
   \quad
   \frac{\partial E}{\partial b^{(\ell)}}
   =
   - 2 \errorvar^{(\ell)}
   .
\end{equation}
Recalling \eqref{eq:definitions_nabla} we have thus established \Cref{eqn:backpropagation-for-W,eqn:backpropagation-for-b} and proven the claim.

\subsection{Proof of \Cref{lem:Square-Integrability-of-the-Forward-and-Backward-Passes}}
\label{sec:Boundedness-of-ONNs-and-their-gradients}

The following is a modification of the approach taken in \cite[Appendix D.1.1]{senen2020almost} adapted to our setting with \gls{AWGN}.

\noindent \textit{Proof.}
Recall \Cref{eqn:backpropagation-for-W,eqn:backpropagation-for-b} and combine it with submultiplicativity to find that
\begin{align}
   \|
   \nabla_{b^{(\ell)}}
   (
   y
   -
   M_s({x} , w, \mathbf{N})
   )^2
   \|
   \eqcom{\ref{eqn:backpropagation-for-W}}=
   2
   \|
   \errorvar^{(\ell)}
   \|
   \label{eq:exampleboundb}
\end{align}
and
\begin{align}
   \|
   \nabla_{W^{(\ell)}}
   (
   y
   -
   M_s({x} , w, \mathbf{N})
   )^2
   \|
   \eqcom{\ref{eqn:backpropagation-for-b}}=
   \|-2
   \errorvar^{(\ell)}
   \big(A^{(i-1)}\big)^T
   \|
   \leq
   2
   \|
   \errorvar^{(\ell)}
   \|
   \|
   \big(A^{(i-1)}\big)^T
   \|
   .
   \label{eq:exampleboundw}
\end{align}

Recall now the definitions of the $R^{(\ell)}$ in \Cref{def:matrices-Ri,def:matrix-RL} and
the definitions of the $A^{(\ell)}$ in \Cref{def:matrices-Ai,def:matrix-AL}.
Observe that to bound $\|\errorvar^{(\ell)}\|$ and/or $\| A^{(\ell)} \|$, it suffices to bound, for $\ell \in \{ 2, \allowbreak 3, \allowbreak \ldots, \allowbreak L-1 \}$, the norms
\begin{equation}
   \|\sigma^\prime
   \big(
   W^{(\ell)}
   A^{(i-1)} + b^{(\ell)}+ N^{\mathrm{w},(\ell)} \big)\|
   \textnormal{ and }
   \|\sigma
   \big(
   W^{(\ell)}A^{(i-1)} + b^{(\ell)}+ N^{\mathrm{w},(\ell)} \big)+N^{\mathrm{a},(\ell)}\|
   .
\end{equation}
After all, for $\ell \in \{ 2, \allowbreak 3, \allowbreak \ldots, \allowbreak L-1 \}$,
\begin{align}
   \|
   \errorvar^{(\ell)}
   \|^2
    &
   \eqcom{\ref{def:matrices-Ri}}=
   \|
   \big(
   W^{(\ell+1)}
   \big)^T
   \errorvar^{(\ell+1)}
   \odot
   \sigma^\prime
   \big(
   W^{(\ell)}
   A^{(\ell-1)} + b^{(\ell)}+ N^{\mathrm{w},(\ell)} \big)
   \|^2
   \\  &
      \eqcom{a}
   \leq
      \|
      W^{(\ell+1)}
   \errorvar^{(\ell+1)}
   \|^2
      \|
      \sigma^\prime
      \big(
      W^{(\ell)}A^{(\ell-1)} + b^{(\ell)}+ N^{\mathrm{w},(\ell)} \big)
      \|^2
   \\  &
      \eqcom{b}
   \leq
      \|
      W^{(\ell+1)}
   \|^2
      \|
      \errorvar^{(\ell+1)}
   \|^2
      \|
      \sigma^\prime
      \big(
      W^{(\ell)}A^{(\ell-1)} + b^{(\ell)}+ N^{\mathrm{w},(\ell)} \big)
      \|^2
      .
      \label{eq: recursive bound on Ri}
\end{align}
Here, we used that (a) for any two matrices $A,B\in\R^{n\times m}$,
\begin{align}
   \|A\odot B\|^2
   =
   \sum_{i,j}
   A_{ij}^2
   B_{ij}^2
   \leq
   \sum_{i,j}
   A_{ij}^2
   \sum_{k,l}
   B_{kl}^2
   =
   \|A\|^2
   \|B\|^2
   ,
\end{align}
\emph{cf}.~\cite[Lemma 30]{senen2020almost}, together with (b) submultiplicativity.

\paragraph{Bounding the activation derivative}
Note that by \ref{item:A1}, \cite[(78)]{senen2020almost} applies.
This means that there exist constants $C_a,$ $C_b,$ $k_a,$ $k_b > 0$ such that
\begin{equation}
   \| \sigma(z) \| \leq C_a (1 + \|z\|)^{k_a}
   \quad
   \textnormal{and}
   \quad
   \| \sigma^\prime(z) \| \leq C_b (1 + \|z\|)^{k_b}
   .
   \label{def:polynomial-boundedness}
\end{equation}
Note that the coefficients $C_a$ and $C_b$ do depend on the dimension of $z$ (see \cite[Lemma~30]{senen2020almost}).
Letting now $C = \max \{ C_a, C_b \}$ and $k = \max \{ k_a, k_b \}$, we have that
\begin{equation}
   \| \sigma(z) \| \leq C (1 + \|z\|)^k
   \quad
   \textnormal{and}
   \quad
   \| \sigma^\prime(z) \| \leq C (1 + \|z\|)^k
   .
   \label{eqn:polynomial-bounds-for-lemma}
\end{equation}

Now, for $\ell \in \{ 1, \ldots, L \}$, let $C^{(\ell)} := C_a^{(\ell)} \vee C_b^{(\ell)}$ refer to the constant associated with the bounds on $\sigma(A^{(i)})$ and $\sigma^\prime(A^{(i)})$, and let $C_{\textnormal{max}} := \max_{\ell=1,\dots,L} C^{(\ell)}$.
Note that we need not worry about layer-dependency of $k$: by construction, there is no such dependency.
This allows us to conclude that
\begin{align}
    &
   \|
   \sigma^\prime
   \big(
   W^{(\ell)}
   A^{(\ell-1)} + b^{(\ell)}+ N^{\mathrm{w},(\ell)} \big)
   \|
   \\  &
      \eqcom{\ref{eqn:polynomial-bounds-for-lemma}}\leq
      C_{\textnormal{max}}
   \bigl(
      1
      +
      \|
      W^{(\ell)}A^{(\ell-1)} + b^{(\ell)}+ N^{\mathrm{w},(\ell)}
   \|
      \bigr)^k
   \\  &
      \leq
      C_{\textnormal{max}}
   \bigl(
      1
      +
      \| W^{(\ell)}A^{(\ell-1)} \|
      +
      \| b^{(\ell)} \|
      +
      \| N^{\mathrm{w},(\ell)} \|
      \bigr)^k
   \\  &
      \leq
      C_{\textnormal{max}}
   \bigl(
      1
      +
      \|W^{(\ell)}\|
      \|A^{(\ell-1)}\|
      +
      \|b^{(\ell)}\|
      +
      \|N^{\mathrm{w},(\ell)}\|
      \bigr)^k
      .
\end{align}

Next, observe that \Cref{ass:H-is-a-hyperrectangle} implies that there exists a constant $\tilde{C} < \infty$ such that for $\ell \in \{ 1, \ldots, L \}$, $\|W^{(\ell)}\| \vee \|b^{(\ell)}\| < \tilde{C}$.
Continuing with this, we can conclude that
\begin{align}
    &
   \|
   \sigma^\prime
   \big(
   W^{(i)}
   A^{(i-1)} + b^{(i)}+ N^{\mathrm{w},(i)} \big)
   \|
   \\  &
      \leq
      C_{\textnormal{max}}
   \bigl(
      1
      +
      \tilde{C} (1+\|A^{(i-1)}\|)
      +
      \|N^{\mathrm{w},(i)}\|
      \bigr)^k
   \\  &
      =
      C_{\textnormal{max}}
   \sum_{j=0}^k
      \textstyle \binom{k}{j}
   \bigl(
      \tilde{C}
      ( 1+\|A^{(i-1)}\| )
      +
      \| N^{\mathrm{w},(i)}\|
      \bigr)^j
   \\  &
      =
      C_{\textnormal{max}}
   \sum_{j=0}^k
      \textstyle \binom{k}{j}
   \sum_{m=0}^j
      \textstyle \binom{j}{m}
   \big(
      \tilde{C}
      (1+\|A^{(i-1)}\|)
      \big)^{j-m}
   \| N^{\mathrm{w},(i)} \|^m
      \label{eq: bound derivative sigma}
   .
\end{align}

Since for any $\ell, m$, the expectation of $\|N^{\mathrm{w},(\ell)}\|^m$ is finite, all that remains is to bound the moments of $A^{(\ell-1)}$.

\paragraph{Bounding the activation norm $\| A^{(\ell-1)} \|$}
Using the arguments leading up to \Cref{eq: bound derivative sigma} \emph{mutatis mutandis}, we find that there also exists a $k>0$ such that
\begin{align}
    &
   \|A^{(\ell-1)}\|
   \eqcom{\ref{def:matrices-Ai}}
   =
   \bigl\|
   \sigma
   \bigl(
   W^{(\ell-1)}
   A^{(\ell-2)} + b^{(\ell-1)}+ N^{\mathrm{w},(\ell-1)}
   \bigr)
   +
   N^{\mathrm{a},(\ell-1)}
   \bigr\|
   \\  &
      \eqcom{\ref{eqn:polynomial-bounds-for-lemma}}
   \leq
      C_{\textnormal{max}}
   \bigl(
      1
      +
      \|
      W^{(\ell-1)}A^{(\ell-2)} + b^{(\ell-1)}+ N^{\mathrm{w},(\ell-1)}
   \|
      \bigr)^k
      +
      \|N^{\mathrm{a},(\ell-1)}\|
   \\  &
      \leq
      \ldots
      \textnormal{ similar to the derivation of \Cref{eq: bound derivative sigma}; \emph{mutatis mutandis}}
   \ldots
   \nonumber \\  &
      \leq
      C_{\textnormal{max}}
   \sum_{j=0}^k
      \textstyle \binom{k}{j}
   \sum_{m=0}^j
      \textstyle \binom{j}{m}
   \bigl(
      \tilde{C}
      (
      1
      +
      \|A^{(\ell-2)}\|
      )
      \bigr)^{j-m}
   \|N^{\mathrm{w},(\ell-1)}\|^m
      +
      \|N^{\mathrm{a},(\ell-1)}\|
      \label{eq:bound_a_i-1_to_be_iterated}
   .
\end{align}

This bound can be applied recursively.
For example, upon one more iteration, we find that
\begin{align}
    &
   C_{\textnormal{max}}^{-1}
   \bigl(
   \|A^{(\ell-1)}\| - \|N^{\mathrm{a},(\ell-1)}\|
   \bigr)
   \eqcom{\ref{eq:bound_a_i-1_to_be_iterated}}
   \leq
   \label{eq: Ai bound}
   \\              &
                  \sum_{j_{\ell-1}=0}^k
                  \textstyle \binom{k}{j_{\ell-1}}
   \sum_{m_{\ell-1}=0}^{j_{\ell-1}}
   \textstyle \binom{j_{\ell-1}}{m_{\ell-1}}
   \|N^{\mathrm{w},(\ell-1)}\|^{m_{\ell-1}}
                  \Bigl\{
                  \tilde{C}
   \Bigl(
                  1
                  +
                  \Bigl\|
                  C_{\textnormal{max}}
   \sum_{j_{\ell-2}=0}^k
                  \textstyle \binom{k}{j_{\ell-2}}
   \nonumber \\            &
                \sum_{m_{\ell-2}=0}^{j_{\ell-2}}
   \textstyle \binom{j_{\ell-2}}{m_{\ell-2}}(Q(1+\|A^{(\ell-3)}\|))^{j_{\ell-2}-m_{\ell-2}}
                \|
                N^{\mathrm{w},(\ell-2)}
   \|^{m_{\ell-2}}
                +
                \|
                N^{\mathrm{w},(\ell-2)}
   \|
                \Bigr\|
                \Bigr)
                \Bigr\}^{j-m_{\ell-1}}
                .
                \nonumber
\end{align}
By continuing and iterating this bound down to $A^{(0)} = x + \| N^{\mathrm{w},(0)} \|$, we obtain a bound that is a polynomial with indeterminates $\| N^{\mathrm{w},(L)} \|, \ldots, \| N^{\mathrm{w},(0)} \|$, $\| N^{\mathrm{a},(L)} \|, \ldots, \| N^{\mathrm{a},(0)} \|$, and $x$.
That is to say that
\begin{align}
   \|
   A^{(\ell-1)}
   \|^4
   \leq
   P(
   \|N^{\mathrm{w},0}\|,\|N^{\mathrm{w},1}\|,\|N^{\mathrm{a},1}\|,\dots,\|N^{\mathrm{w},L-1}\|,\|N^{\mathrm{a},L-1}\|,\|N^{\mathrm{w},L}\|,\|x\|
   )
\end{align}
for some polynomial $P$.

Similarly, one finds that
\begin{align}
   \|
   \errorvar^{(\ell)}
   \|^4
   \leq
   \|y\|^4
   Q(\|N^{\mathrm{w},0}\|,\|N^{\mathrm{w},1}\|,\|N^{\mathrm{a},1}\|,\dots,\|N^{\mathrm{w},L-1}\|,\|N^{\mathrm{a},L-1}\|,\|N^{\mathrm{w},L}\|,\|x\|),
\end{align}
for some polynomial $Q$.

Recall finally that the $N^{\cdot,\cdot}$-terms are normally distributed, and specifically that all moments of normal distributions are finite.
Furthermore, note that \Cref{item:A2} ensures that the data distribution $X$ has sufficiently many finite moments.
Thus clearly, the expectations of \Cref{eq:exampleboundb,eq:exampleboundw} are finite.
This implies \Cref{lem:Square-Integrability-of-the-Forward-and-Backward-Passes}.
\QuodEratDemonstrandum
\subsection{Proofs of \texorpdfstring{\Cref{thm:stochastic-approximation-theorem-limit-trajectories,thm:stochastic-approximation-theorem-limitpoints}}{limit trajectories and limit points}}
\label{appendix:proof-ODE}

We now prove \Cref{thm:stochastic-approximation-theorem-limit-trajectories,thm:stochastic-approximation-theorem-limitpoints}.
Our method will be to verify assumptions \cite[A2.1--A2.6, p.~126; (1.1), (1.2), p.~120]{kusher2003stochastic}, which will allow us to apply \cite[Theorem 2.1, p.~127]{kusher2003stochastic}.
From \cite[Theorem 2.1, p.~127]{kusher2003stochastic}, \Cref{thm:stochastic-approximation-theorem-limit-trajectories,thm:stochastic-approximation-theorem-limitpoints} follow.

Let us note immediately that \cite[(1.1), (1.2), p.~120]{kusher2003stochastic} are simply \Cref{eqn:projected-sgd} and \Cref{ass:H-is-a-hyperrectangle}.
These assumptions are thus satisfied.

To verify \texorpdfstring{\cite[A2.1--A2.6, p.~126]{kusher2003stochastic}}{A2.1--A2.6}, most work goes into establishing that the gradients of the objective function are bounded.
Recall that this is the content of \Cref{lem:Square-Integrability-of-the-Forward-and-Backward-Passes}.

To establish \Cref{thm:stochastic-approximation-theorem-limit-trajectories} using \cite[Theorem~2.1, p.~127]{kusher2003stochastic}, verification of \cite[A2.1--A2.5, p.~126]{kusher2003stochastic} under our \Cref{item:A1,item:A2,item:A3,item:A4,item:A5} suffices.
To establish \Cref{thm:stochastic-approximation-theorem-limitpoints} using \cite[Theorem~2.1, p.~127]{kusher2003stochastic}, we must also verify \cite[A2.6, p.~126]{kusher2003stochastic}.
Given the statement of \Cref{thm:stochastic-approximation-theorem-limitpoints} though, we then work under the stronger assumption set of \Cref{item:A2,item:A3,item:A4,item:A5} plus \Cref{item:B1,item:B2}.

\subsubsection{Verification of assumptions \texorpdfstring{\cite[A2.1--A2.6, p.~126]{kusher2003stochastic}}{A2.1-A2.6}}
\label{sec:Verification-of-Kushners-assumptions}

Let us begin by verifying \cite[A2.1--A2.5, p.~126]{kusher2003stochastic} under \Cref{item:A1,item:A2,item:A3,item:A4,item:A5}:

\begin{itemize}
   \item
         Assumption \cite[A2.1, p.~126]{kusher2003stochastic} is to require that
         \begin{align}
            \sup_n \mathbb{E}\| \nabla_w \ENO(w^{\{n\}})\bigr\rvert_{\mathbf{N},{x},y} \|^2
            <
            \infty
            .
            \label{eqn:Assumption-A21}
         \end{align}

         \emph{Verification.}
         Observe that if for every $W^{(\ell)}$, $b^{(\ell)}$,
         \begin{align}
            \sup_n \mathbb{E}\| \nabla_{W^{(\ell)}} \ENO(w^{\{n\}})\bigr\rvert_{\mathbf{N},{x},y} \|^2
            <
            \infty
            ,
            \quad
            \sup_n \mathbb{E}\| \nabla_{b^{(\ell)}} \ENO(w^{\{n\}})\bigr\rvert_{\mathbf{N},{x},y} \|^2
            <
            \infty
            ,
            \label{eqn:intermediate-step-to-verify-assumption-a21}
         \end{align}
         then \Cref{eqn:Assumption-A21} also holds.
         Because of \Cref{item:A1,item:A2,item:A4}, \Cref{lem:Square-Integrability-of-the-Forward-and-Backward-Passes} applies.
         Equation \Cref{eqn:finiteness-of-the-expected-norm-of-the-gradient} then implies that \Cref{eqn:intermediate-step-to-verify-assumption-a21} holds and consequently \Cref{eqn:Assumption-A21} also.

   \item
         Assumption \cite[A2.2, p.~126]{kusher2003stochastic} is that
         (i) there is a measurable function $\bar{g}(\,\cdot\,)$ of $w$
         and that
         (ii) there exist random variables $\beta^{n}$,
         such that
         \begin{align}
            \mathbb{E}
            \Bigl[
               \nabla_w
               \bigl(
               y^{\{n\}}
               -
               M_\sd( {x}^{\{n\}}, w^{\{n\}}, \mathbf{N}^{\{n\}} )
               \bigr)^2
               \mid
               \mathcal{F}_{n-1}
               \Bigr]
            =
            \bar{g}(w^n)
            +
            \beta^{n}
            .
            \label{eqn:intermediate-assumption-A2-2}
         \end{align}
         Here $\mathcal{F}_{n-1}$, denotes the smallest $\sigma$-algebra generated by
         $
            \cup_{k\leq n-1} \big\{ w^0, \allowbreak (\mathbf{N}^{k}, \allowbreak X^k, \allowbreak Y^k) \big\}
         $.

         \emph{Verification.}
         Examine \Cref{eqn:projected-sgd} and conclude that (a) $w^{\{n\}} \in \mathcal{F}_{n-1}$.
         Recall also that by construction, (b) the random variables $x^{\{n\}}, y^{\{n\}}, \mathbf{N}^{\{n\}}$ are independent of $w^0, x^{\{1\}}, y^{\{1\}}, \mathbf{N}^{\{1\}}$, $\ldots$, $x^{\{n-1\}}, y^{\{n-1\}}, \mathbf{N}^{\{n-1\}}$.
         Furthermore, recall that each iteration, (c) the random variables $x^{\{n\}}, y^{\{n\}}, \mathbf{N}^{\{n\}}$ are generated in an identically distributed manner.
         Therefore
         \begin{align}
             &
            \mathbb{E}
            \Bigl[
               \nabla_w
               \bigl(
               y^{\{n\}}
               -
               M_\sd( {x}^{\{n\}}, w^{\{n\}}, \mathbf{N}^{\{n\}} )
               \bigr)^2
               \mid
               \mathcal{F}_{n-1}
               \Bigr]
            \nonumber \\                                                       &
                                                                    \eqcom{a, b, c}=
                                                                    \int
                                                                    \nabla_w
                                                                    \bigl(
                                                                    y
                                                                    -
                                                                    M_\sd( x, w^{\{n\}}, \mathbf{N} )
                                                                    \bigr)^2
                                                                    \mathrm{d}
            \mathbb{P}
                                                                    [
                                                                       (\mathcal{N}_\sd,X,Y) = (\mathbf{N}, x,y)
                                                                       ]
            .
         \end{align}

         Finally: interchanging the order of differentiation and integration is warranted because:
         (i) for any fixed $y, \mathbf{N}$, the function $M(y, \cdot, \mathbf{N})$ is continuous since it is a composition of continuous functions; and
         (ii) for any fixed $\tilde{w}$, the random variable $\nabla_w M(X, \tilde{w}, Y)$ is square-integrable as implied by the proof of \Cref{lem:Square-Integrability-of-the-Forward-and-Backward-Passes}.

         In summary, we conclude that
         \begin{align}
             &
            \mathbb{E}
            \Bigl[
               \nabla_w
               \bigl(
               y^{\{n\}}
               -
               M_\sd( {x}^{\{n\}}, w^{\{n\}}, \mathbf{N}^{\{n\}} )
               \bigr)^2
               \mid
               \mathcal{F}_{n-1}
               \Bigr]
            \nonumber \\                                                       &
                                                                    =
                                                                    \nabla_w
                                                                    \int
                                                                    (
                                                                    y
                                                                    -
                                                                    M_\sd({x}, w^{\{n\}}, \mathbf{N})
                                                                    )^2
                                                                    \mathrm{d}
            \mathbb{P}[ (\mathcal{N}_\sd,X,Y)=(\mathbf{N},{x},y) ]
            \nonumber \\                                                       &
                                                                    =
                                                                    \nabla_w
                                                                    \ENO( w^{\{n\}})
            .
                                                                    \label{eqn:our-gbar-and-beta}
         \end{align}
         The result thus follows for
         \begin{equation}
            \overline{g}
            \equiv
            \nabla_w \ENO_s,
            \quad
            \beta_n
            \equiv
            0
            .
            \label{eqn:Identification-of-gbar-and-beta}
         \end{equation}

   \item
         Assumption \cite[A2.3, p.~126]{kusher2003stochastic} is that the function $\bar{g}$ in \eqref{eqn:intermediate-assumption-A2-2} is continuous.

         \emph{Verification.}
         Recall that we have identified $\bar{g}$ to be equal to $\nabla_w \ENO_s$ in \Cref{eqn:Identification-of-gbar-and-beta}.

         Examine now the definition of $\ENO_s$ in \Cref{eqn:objective-function-as-a-function-of-s}.
         Given the fact that $M$ is a composition of linear transformations of activation functions $\sigma$ that are twice continuously differentiable by assumption, $\nabla_w \ENO_s$ specifically is also continuous.

   \item
         Assumption \cite[A2.4, p.~126]{kusher2003stochastic} is that the step sizes satisfy
         \begin{equation}
            \sum_{t=1}^\infty \eps_t = \infty,
            \eps_n \geq 0, \eps_n \to 0
            \textnormal{ for }
            n \geq 0
            \textnormal{ and }
            \eps_n = 0
            \textnormal{ for }
            n < 0
            ;
            \enskip
            \textnormal{and}
            \enskip
            \sum_{t=1}^\infty \eps_t^2 < \infty
            .
            \label{eqn:Step_sizes_diverge_but_not_too_fast_appendix}
         \end{equation}

         \emph{Verification.}
         \eqref{eqn:Step_sizes_diverge_but_not_too_fast_appendix} is immediate by \ref{item:A4}.

   \item
         Assumption \cite[A2.5, p.~126]{kusher2003stochastic} is that
         $
            \sum_n
            \eps_n
            \| \beta^{ \{n\} } \|_\mathrm{F}
            <
            \infty
         $
         with probability one.

         \emph{Verification.}
         Recall that $\beta^{\{n\}}$ is identified in \eqref{eqn:Identification-of-gbar-and-beta}.
         In fact, $\beta^{\{n\}} \equiv 0$, implying the assumption immediately.
\end{itemize}

Our verification of \cite[A2.1--A2.5, p.~126]{kusher2003stochastic} has now, effectively, proven \Cref{thm:stochastic-approximation-theorem-limit-trajectories}.
We can now namely simply invoke \cite[Thm.~2.1, p.~127]{kusher2003stochastic} to obtain the result.

To prove \Cref{thm:stochastic-approximation-theorem-limitpoints}, all that remains is to prove \cite[A2.6, p.~126]{kusher2003stochastic} also.
Recall though that we will now work under the stronger assumption set of \Cref{item:A2,item:A3,item:A4,item:A5} plus \Cref{item:B1,item:B2}:

\begin{itemize}
   \item
         Assumption \cite[A2.6, p.~126]{kusher2003stochastic} is that there exists a continuously differentiable real-valued $h(\,\cdot\,)$, constant on each stationary set, such that $\bar{g}(\,\cdot\,)=- \nabla h(\,\cdot\,)$.

         \emph{Verification.}
         This follows from \cite[Lemma~18, Lemma~19]{senen2020almost} \emph{mutatis mutandis}.
         Consult specifically \cite[\S{D.1.2}, \S{D.1.3}]{senen2020almost}.
         The key point to realize is that under \Cref{item:B1}, for any multi-index $k$,
         \begin{align}
             &
            \|
            \partial^k
            (y-M_\sd(x,w,\mathbf{N}))
            \|
            \\                                                       &
                                                                    \leq
                                                                    \|y\|^2
                                                                    P(\|N^{\mathrm{w},0}\|,\|N^{\mathrm{w},1}\|,\|N^{\mathrm{a},1}\|,\dots,\|N^{\mathrm{w},L-1}\|,\|N^{\mathrm{a},L-1}\|,\|N^{\mathrm{w},L}\|,\|x\|)
                                                                    \nonumber
         \end{align}
         for some polynomial with finite exponents.
         This then gives sufficient differentiability of the objective function; see also the discussion below \cite[(88)]{senen2020almost}.
\end{itemize}

That is it.
\QuodEratDemonstrandum
\subsection{Partial derivative of a product of Gaussian densities}
\label{sec:Partial-derivative-of-a-product-of-Gaussian-densities}

\begin{lemma}
    \label{lem:chain-derivative}

    For $d \in \mathbb{N}_+$, $\sd > 0$, let $\phi_s^{(d)}$ be the probability density function of a $d$-dimensional Gaussian with mean $\mu = 0$ and covariance matrix $\Sigma \allowbreak = \allowbreak \sd^2 \mathrm{1}_d$.
    Then, for $d_1, \ldots, d_L \in \mathbb{N}_+$, $n^{(1)}, \ldots, n^{(L)} \in \mathbb{R}^{d_i}$, and $\sd > 0$, we have that
    \begin{align}
        \frac{\partial}{\partial\sd}\bigg(\prod_{i=1}^L
        \phi_{\sd}^{(d_i)}(n^{(i)})\bigg)
        =
        \sum_{i=1}^L
        \big(
        \sd^{-2}(n^{(i)})^T n^{(i)} - d_i
        \big)
        \bigg(\prod_{j=1}^L
        \phi_{\sd}(n^{(j)})\bigg).
        \label{eqn:partial-derivative-of-product-of-gaussian-densities}
    \end{align}
    Furthermore, the map
    $
        s
        \mapsto
        \prod_{i=1}^L
        \phi_{\sd}^{(d_i)}(n^{(i)})
    $
    is $\mathcal{C}^\infty$ on $(0, \infty)$.
\end{lemma}

\begin{proof}
    First, recall that for $d \in \mathbb{N}_+$, $\sd > 0$, and $n \in \mathbb{R}^d$,
    \begin{align}
        \phi_{\sd}^{(d)}(n)
        =
        \frac{1}{(2\pi)^{d/2}}
        \frac{1}{\sqrt{\det\Sigma}}
        \exp{
            \big( -\tfrac{1}{2} n^T \Sigma^{-1} n \big)
        }
        ;
    \end{align}
    see, for example, \cite[3.3 Multivariate Random Variables]{borovkov1999probability}.
    Second, note that by assumption, the covariance matrix is a diagonal matrix.
    Since its determinant then equals the product of its diagonal entries, we have that
    \begin{align}
        \phi_{\sd}^{(d)}(n)
        =
        \frac{1}{\sd^{d}(2\pi)^{d/2}}
        \exp{
            \big( -\tfrac{1}{2\sd^2} n^T n \big)
        }
        .
        \label{eqn:pdf-of-high-dimensional-gaussian-density-with-diagonal-covariance}
    \end{align}

    Next: let $d_1, \ldots, d_L \in \mathbb{N}_+$, $n^{(1)}, \ldots, n^{(L)} \in \mathbb{R}^{d_i}$, and $\sd > 0$.
    Take the product of $\phi_{\sd}^{(d_i)}(n^{(i)})$ over $i = 1, \ldots, \ell$ and substitute \Cref{eqn:pdf-of-high-dimensional-gaussian-density-with-diagonal-covariance}, to find that
    \begin{align}
        \prod_{i=1}^L
        \phi_{\sd}^{(d_i)}(n^{(i)})
         &
        =
        \prod_{i=1}^L
        \frac{1}{\sd^{d_i}(2\pi)^{d_i/2}}
        \exp{
            \big( - \tfrac{1}{2\sd^2} \big(n^{(i)}\big)^T n^{(i)} \big)
        }
        \\                                                    &
                                                             =
                                                             \frac{1}{ (2\pi \sd^2)^{\sum_id_i/2}}
        \exp{
                                                                 \Big( - \tfrac{1}{2\sd^2} \sum_{i=1}^L \big(n^{(i)}\big)^T n^{(i)} \Big)
                                                                 }
        =:
                                                             f_{\vect{d}, \vect{n}}(s)
                                                             .
    \end{align}
    This proves \Cref{eqn:partial-derivative-of-product-of-gaussian-densities}.

    Finally, observe that the map $s \mapsto f_{\vect{d}, \vect{n}}(s)$ is $\mathcal{C}^\infty$ on $(0, \infty)$ because it is a product/composition of the maps $x \mapsto 1/x$, $x \mapsto x^2$, and $x \mapsto \exp{(-x)}$ which are all $\mathcal{C}^\infty$ on (at least) $(0, \infty)$.
\end{proof}

\subsection{Proof that \texorpdfstring{$D^{[0]} \approx (\partial / \partial \sd) \nabla_w \mathcal{J}(\sd_0)$}{D0}}
\label{sec:Proof-that-D0-is-approx-d-ds-nabla-ENO}

\Cref{lem:unbiased-sampling-D0} follows almost immediately from \Cref{lem:Hierarchical-sampling-is-unbiased} when one verifies the integrability condition in \Cref{eq: integrability in hierarchical sampling lemma}.
\Cref{lem:Hierarchical-sampling-is-unbiased} is a variant of the \gls{LLN}:

\begin{lemma}
	\label{lem:Hierarchical-sampling-is-unbiased}
	
	Let $(\Omega, \mathcal{A}, \mathbb{P})$ be a probability space, and let $A \sim \mu$, and $B \mid A \sim \nu(\cdot \mid A)$ be random variables with state spaces $\mathcal{M}_1$ and $\mathcal{M}_2$, respectively.
	Let $f: \mathcal{M}_1 \times \mathcal{M}_2 \to \R$ be a measurable function satisfying:
	\begin{align}
		\mathbb{E}[ |f(A,B)| ], \quad \mathbb{E}[ f(A,B)^2 ] < \infty.
		\label{eq: integrability in hierarchical sampling lemma}
	\end{align}
	
	Consider the following sampling procedure:
	(i) draw $ K_1 $ samples $A^i \sim \mu$ in an \gls{IID} fashion, and then
	(ii) for each $A^i$, draw $K_2$ samples $B^{i,j} \sim \nu(\cdot \mid A^i)$.
	Define
	\begin{align}
		S_{\text{hi}}
		=
		\frac{1}{K_1 K_2} \sum_{i=1}^{K_1} \sum_{j=1}^{K_2} f(A^i, B^{i,j})
		.
		\label{eqn:Hierarchical-sampler}
	\end{align}
	Then, as $ K_1, K_2 \to \infty $,
	\begin{align}
		S_{\text{hi}} \xrightarrow{\text{a.s.
		}} \mathbb{E}[f(A, B)].
		\label{eqn:Convergence-of-hierarchical-sampler}
	\end{align}
\end{lemma}

Observe that according to \Cref{eqn:Convergence-of-hierarchical-sampler}, the sample mean estimator in \Cref{eqn:Hierarchical-sampler} has the same limiting behavior as the typical, canonical sample mean estimator.
That is, when drawing $K_1 K_2$ pairs $( A^k, \allowbreak B^k )_{k=1}^{K_1 K_2}$ in an \gls{IID} fashion, one has that
\begin{align}
	S_{\text{ind}}
	=
	\frac{1}{K_1 K_2} \sum_{k=1}^{K_1 K_2} f(A^k, B^k)
\end{align}
satisfies
\begin{align}
	S_{\text{ind}}
	=
	\frac{1}{K_1 K_2}
	\sum_{k=1}^{K_1 K_2}
	f(A^k, B^k)
	\xrightarrow[K_1 K_2 \to \infty]{\text{a.s.
	}}
	\mathbb{E}[f(A, B)]
\end{align}
by the \gls{LLN}.
Consequently,
\begin{align}
	\lim_{K_1, K_2 \to \infty}
	S_{\text{hi}} = \lim_{K_1, K_2 \to \infty} S_{\text{ind}} \quad \text{a.s.
	}
\end{align}

The proof of \Cref{lem:unbiased-sampling-D0}, i.e., the verification of \Cref{eq: integrability in hierarchical sampling lemma}, can be found in \Cref{sec:Proof-that-sampler-is-unbiased}.
It, too, relies on the boundedness result in \Cref{lem:Square-Integrability-of-the-Forward-and-Backward-Passes}.

\subsection{Proof of \Cref{lem:Hierarchical-sampling-is-unbiased}}
\label{sec:Proof-that-hierarchical-sampling-is-unbiased}

Start by decomposing the empirical mean:
\begin{align}
    S_{\text{hi}}
    =
    \frac{1}{K_1}
    \sum_{i=1}^{K_1}
    \Bigl( \frac{1}{K_2} \sum_{j=1}^{K_2} f(A^i, B^{i,j}) \Bigr)
    =:
    \frac{1}{K_1}
    \sum_{i=1}^{K_1}
    Y_i
    ,
\end{align}
say.

Define $g(A) \coloneqq \mathbb{E}[f(A, B)]$.
Under the integrability assumption in \eqref{eq: integrability in hierarchical sampling lemma}, the strong \gls{LLN} applies.
Thus, for all $a\in\mathcal{M}_1$,
\begin{align}
    \frac{1}{K_2} \sum_{j=1}^{K_2} f(a, B^{i,j}) \xrightarrow[K_2 \to \infty]{\text{a.s.
        }} g(a).
\end{align}
Furthermore, for $\mu$-almost every $A^i$
\begin{align}
    Y_i = \frac{1}{K_2} \sum_{j=1}^{K_2} f(A^i, B^{i,j}) \xrightarrow[K_2 \to \infty]{\text{a.s.
        }} g(A^i).
\end{align}

Using a \emph{nullergänzung}, we can rewrite:
\begin{align}
    S_{\text{hi}} = \frac{1}{K_1} \sum_{i=1}^{K_1} Y_i = {\frac{1}{K_1} \sum_{i=1}^{K_1} \left(Y_i - g(A^i)\right)}
    +
    {\frac{1}{K_1} \sum_{i=1}^{K_1} g(A^i)}.
    \label{eqn:decomposition-of-hierarchical-mean}
\end{align}
By again leveraging that the integrability assumption in \eqref{eq: integrability in hierarchical sampling lemma} implies the strong \gls{LLN} for $\mu$, the second term in \eqref{eqn:decomposition-of-hierarchical-mean} satisfies
\begin{align}
    \frac{1}{K_1} \sum_{i=1}^{K_1} g(A^i)
    \xrightarrow[K_1 \to \infty]{\text{a.s.
        }}
    \mathbb{E}[g(A)]
    =
    \mathbb{E}[f(A, B)].
\end{align}

We next show that the deviations given by $Y_i - g(A^i)$ converges to zero almost surely.
Since $Y_i - g(A^i)$ has mean zero, its variance is given by
\begin{align}
    \mathrm{Var}[
        Y_i - g(A^i)
    ]
    =
    \mathbb{E} \Bigl[
        \mathrm{Var}[
            Y_i \mid A^i
        ]
        \Bigr]
    =
    \frac{1}{K_2}
    \mathbb{E}
    \Bigl[
        \mathrm{Var}[
            f(A, B) \mid A
        ]
        \Bigr]
    .
    \label{eqn:Variance-of-sum-of-Yi-centered-around-gAi}
\end{align}
Since the terms $\{Y_i - g(A^i)\}_{i=1}^{K_1}$ are independent across $i$, the variance of their empirical mean satisfies
\begin{align}
    \mathrm{Var} \Bigl[
    \frac{1}{K_1}
    \sum_{i=1}^{K_1}
    (
    Y_i - g(A^i)
    )
    \Bigr]
     &
    =
    \frac{1}{K_1^2}
    \sum_{i=1}^{K_1}
    \mathrm{Var}[
        Y_i - g(A^i)
    ]
    \eqcom{\ref{eqn:Variance-of-sum-of-Yi-centered-around-gAi}}=
    \frac{1}{K_1 K_2} \mathbb{E}
    \bigl[
        \mathrm{Var}[
            f(A, B) \mid A
        ]
        \bigr].
\end{align}
As $K_1, K_2 \to \infty$, we see that indeed
\begin{align}
    \mathrm{Var} \Bigl[ \frac{1}{K_1}
    \sum_{i=1}^{K_1}
    (Y_i - g(A^i)) \Bigr]
    \to
    0
    .
\end{align}

By Chebyshev’s inequality and the Borel--Cantelli lemma, this implies almost sure convergence:
\begin{align}
    \frac{1}{K_1} \sum_{i=1}^{K_1} (Y_i - g(A^i)) \xrightarrow[K_1, K_2 \to \infty]{\text{a.s.
        }} 0.
\end{align}
Since both terms in the right-hand side of \eqref{eqn:decomposition-of-hierarchical-mean} converge almost surely,
\begin{align}
    S_{\text{hi}} \xrightarrow[K_1, K_2 \to \infty]{\text{a.s.
        }} \mathbb{E}[f(A, B)].
\end{align}
That is it.
\QuodEratDemonstrandum
\subsection{Proof of \Cref{lem:unbiased-sampling-D0}}
\label{sec:Proof-that-sampler-is-unbiased}

We prove \Cref{lem:unbiased-sampling-D0} via \Cref{lem:Hierarchical-sampling-is-unbiased}.
This requires verifying that the conditions of \Cref{lem:Hierarchical-sampling-is-unbiased} hold for
\begin{equation}
    D^{[0]}_{W^{(\ell)}}(K_1,K_2)
    \quad
    \textnormal{and}
    \quad
    D^{[0]}_{b^{(\ell)}}(K_1,K_2)
    .
\end{equation}

Begin by noting that the involved random variables $(X^{\{k\}},Y^{\{k\}})$ are independent copies of $(X,Y)$, and that the $\mathbf{N}^{\{m_k\}}$ are independent copies of $\mathbf{N}$.
Furthermore, for each $jk$th component of the $\ell$th weight matrix,
\begin{align}
     &
    f_{W^{(\ell)}_{jk}}((X^{\{k\}},Y^{\{k\}}),\mathbf{N}^{\{m_k\}})
    \nonumber \\                                                &
                                                     =
                                                     \left[\Big(\sum_{\alpha\in S_i} \big(\sd_0^{-2}\big(N^{\alpha,\{m_k\}}\big)^T N^{\alpha,\{m_k\}} - d_{f(\alpha)}\big)\Big)
                                                         \errorvar^{(\ell),\{k\}}
        \big(A^{(\ell-1),\{k\}}\big)^T\right]_{jk}
    ;
\end{align}
and similarly, for each $j$th component of the $\ell$th bias vector,
\begin{align}
     &
    f_{b^{(\ell)}_j}((X^{\{k\}},Y^{\{k\}}),\mathbf{N}^{\{m_k\}})
    \nonumber \\                                                &
                                                     =
                                                     \left[\Big(\sum_{\alpha\in S_i} \big(\sd_0^{-2}\big(N^{\alpha,\{m_k\}}\big)^T N^{\alpha,\{m_k\}} - d_{f(\alpha)}\big)\Big)
                                                         \errorvar^{(\ell),\{k\}}\right]_j.
\end{align}

To apply \Cref{lem:Hierarchical-sampling-is-unbiased}, we need to verify that the following moments are finite:
\begin{equation}
    \begin{split}
        \mathbb{E}[| f_{W^{(\ell)}_{jk}}((X^{\{k\}},Y^{\{k\}}),\mathbf{N}^{\{m_k\}})|],
        \quad
        \mathbb{E}[| f_{W^{(\ell)}_{jk}}((X^{\{k\}},Y^{\{k\}}),\mathbf{N}^{\{m_k\}})|^2],
        \\
        \mathbb{E}[| f_{b^{(\ell)}_j}((X^{\{k\}},Y^{\{k\}}),\mathbf{N}^{\{m_k\}})|],
        \quad
        \textnormal{and}
        \quad
        \mathbb{E}[| f_{b^{(\ell)}_j}((X^{\{k\}},Y^{\{k\}}),\mathbf{N}^{\{m_k\}})|^2].
    \end{split}
    \label{eq:finite_moment_terms_in_D-rightarrow-J}
\end{equation}
Note that if we would instead show that the full matrix and vector norms, respectively, are finite, then finiteness of the expectations in \Cref{eq:finite_moment_terms_in_D-rightarrow-J} would follow also.
Note also that for $\beta = 1, 2$, submultiplicativity implies that
\begin{align}
     &
    \mathbb{E}
    \Bigl[
    \bigl\|
    f_{W^{(\ell)}}((X^{\{k\}},Y^{\{k\}}),\mathbf{N}^{\{m_k\}})
    \bigr\|^\beta
    \Bigr]
    \label{eq:to_bound_in_D0rightarrowpartialpatial}
    \\                                                &
                                                     \leq
                                                     \mathbb{E}
    \Bigl[
                                                         \Bigl\|
                                                         \sum_{\alpha\in S_i} \big(\sd_0^{-2}\big(N^{\alpha,\{m_k\}}\big)^T N^{\alpha,\{m_k\}} - d_{f(\alpha)}\big)
                                                         \Bigr\|^\beta
                                                         \|
                                                         \errorvar^{(\ell),\{k\}}
        \|^\beta
                                                         \bigl\|
                                                         \big(A^{(\ell-1),\{k\}}\big)^T
                                                         \bigr\|^\beta
                                                         \Bigr]
    .
                                                     \nonumber
\end{align}

Let us show an intermediate bound for generic {nonnegative} random variables $\X, \Y, \Z$, not necessarily independent, which we will then immediately apply.
Recall that
\begin{equation}
    \mathbb{E}[\X\,\Y]
    =
    \mathbb{E}[\X]\mathbb{E}[\Y]
    +
    \mathrm{Cov}(\X,\Y)
    .
    \label{eqn:Covariance-of-nonnegative-RVs-X-Y}
\end{equation}
Iterating \Cref{eqn:Covariance-of-nonnegative-RVs-X-Y}, we find that
\begin{align}
    \mathbb{E}[\X \Y \Z]
    =
    \mathbb{E}[\X]\mathbb{E}[\Y]\mathbb{E}[\Z]
    +
    \mathrm{Cov}(\X,\Y)\mathbb{E}[\Z]
    +
    \mathrm{Cov}(\X \Y \Z)
    .
    \label{eqn:Itereated-covariance-of-nonnegative-RVs-X-Y-Z}
\end{align}
By the Cauchy--Schwarz inequality,
\begin{align}
    \mathrm{Cov}(\X,\Y)
    \leq
    \sqrt{\mathrm{Cov}(\X,\X)\mathrm{Cov}(\Y,\Y)}
    =
    \sqrt{\mathrm{Var}(\X)\mathrm{Var}(\Y)}
    .
    \label{eqn:Cauchy--Schwarz_inequality}
\end{align}
Using \Cref{eqn:Cauchy--Schwarz_inequality} twice to bound \Cref{eqn:Itereated-covariance-of-nonnegative-RVs-X-Y-Z}, we find that
\begin{equation}
    \mathbb{E}[\X \Y \Z]
    \leq
    \mathbb{E}[\X]\mathbb{E}[\Y]\mathbb{E}[\Z]
    +
    \sqrt{
        \mathrm{Var}[ Z_1 ]
        \mathrm{Var}[ Z_2 ]
    }
    \mathbb{E}[\Z]
    +
    \sqrt{
        \mathrm{Var}[ Z_1 Z_2 ]
        \mathrm{Var}[ Z_3 ]
    }
    .
    \label{eq:XYZbeta_bound}
\end{equation}
Choosing
\begin{align}
    \X
     &
    =
    \|
    \errorvar^{(\ell),\{k\}}
    \|^\beta
    ,
    \nonumber \\
    \quad
    \Y
     &
    =
    \|
    \big(A^{(\ell-1),\{k\}}\big)^T
    \|^\beta
    ,
    \nonumber \\
    \textnormal{and}
    \quad
    \Z
     &
    =
    \|
    \sum_{\alpha\in S_i}
    \big(\sd_0^{-2}\big(N^{\alpha,\{m_k\}}\big)^T N^{\alpha,\{m_k\}} - d_{f(\alpha)}\big)
    \|^\beta
    \label{eqn:Explicit-choice-of-XYZ}
\end{align}
enables us to bound \eqref{eq:to_bound_in_D0rightarrowpartialpatial} via \eqref{eq:XYZbeta_bound}.

Specifically, for $\beta \in {1, 2}$, and with $\X, \Y, \Z$ as in \Cref{eqn:Explicit-choice-of-XYZ}, \Cref{lem:Square-Integrability-of-the-Forward-and-Backward-Passes} implies that $\mathbb{E}[\X]$ is bounded.
Similarly, \eqref{eq: Ai bound}, established in the proof of \Cref{lem:Square-Integrability-of-the-Forward-and-Backward-Passes}, implies that $\mathbb{E}[\Y]$ is bounded.
Finally, since
\begin{equation}
    \mathrm{Var}[ \X\Y ]
    =
    \mathbb{E}
    \big[
        \X^2
        \Y^2
        \big]
    -
    2\mathbb{E}
    \big[
        \X
        \Y
        \big]
    \mathbb{E}
    [
        \X
        \Y
    ]
    +
    \mathbb{E}
    \big[
        \X
        \Y
        \big]^2
    ,
\end{equation}
the highest order moments are $\|\X\|^4$ and $\|\Y\|^4$, which are also bounded by \Cref{lem:Square-Integrability-of-the-Forward-and-Backward-Passes} under the assumptions of \Cref{thm:ft-guarantee}.

Finally, both $\mathbb{E}[\Z]$ and $\mathrm{Var}(|\Z|^\beta)$ correspond to the first and second moments of a polynomial function of Gaussian random variables.
These are finite because all moments of the (multivariate) Gaussian distribution exist and are finite.
The conditions of \Cref{lem:Hierarchical-sampling-is-unbiased} are thus met, proving the claim.
\QuodEratDemonstrandum

\section{Absolute improvements}
In the main text we discussed how the relative improvement as a measure tends to inflate accuracy improvements and deflates loss improvements, as loss goes up and accuracy goes down with noise. Therefore, we present the absolute (rather than relative) improvements in the below \Cref{fig:abs_performance_gift}.

\begin{figure}[h]
	\centering
	\includegraphics[width=0.95\linewidth]{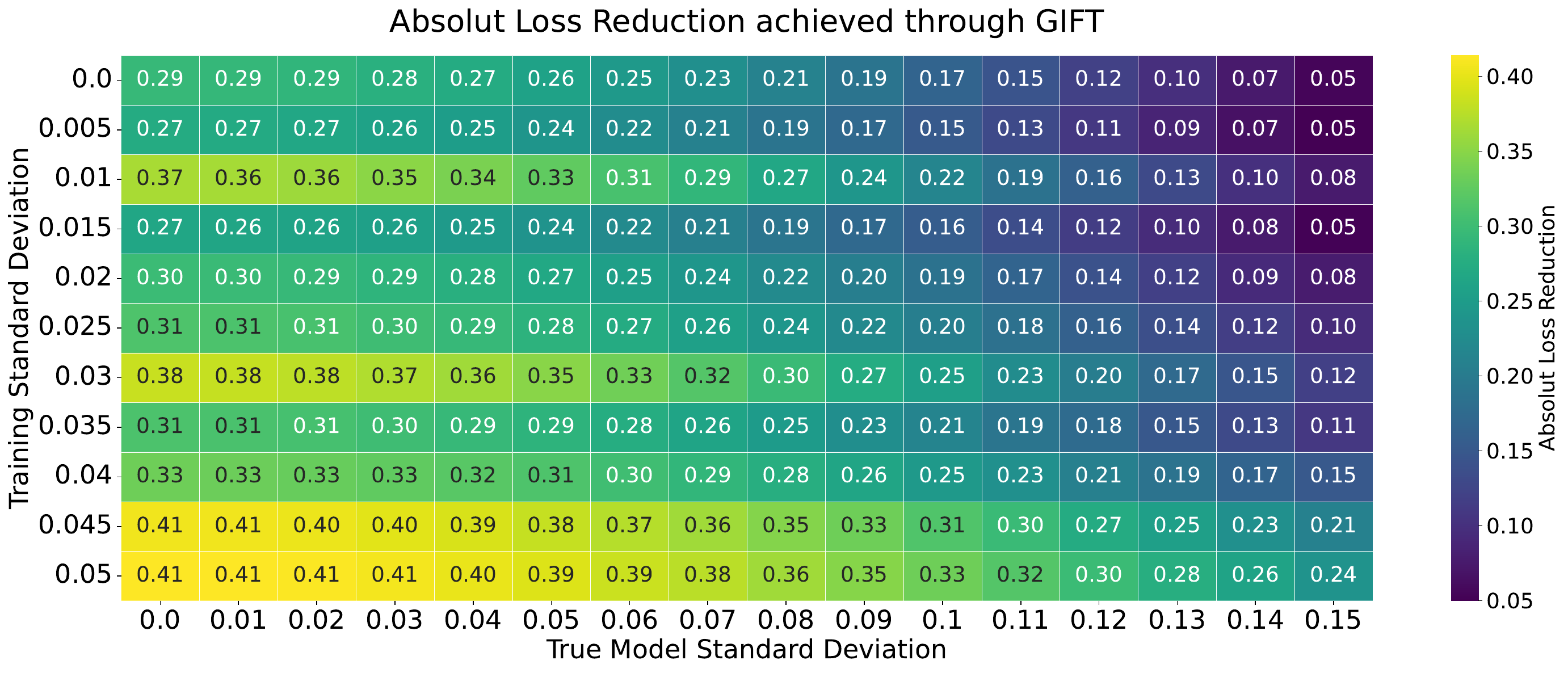}
	\includegraphics[width=0.95\linewidth]{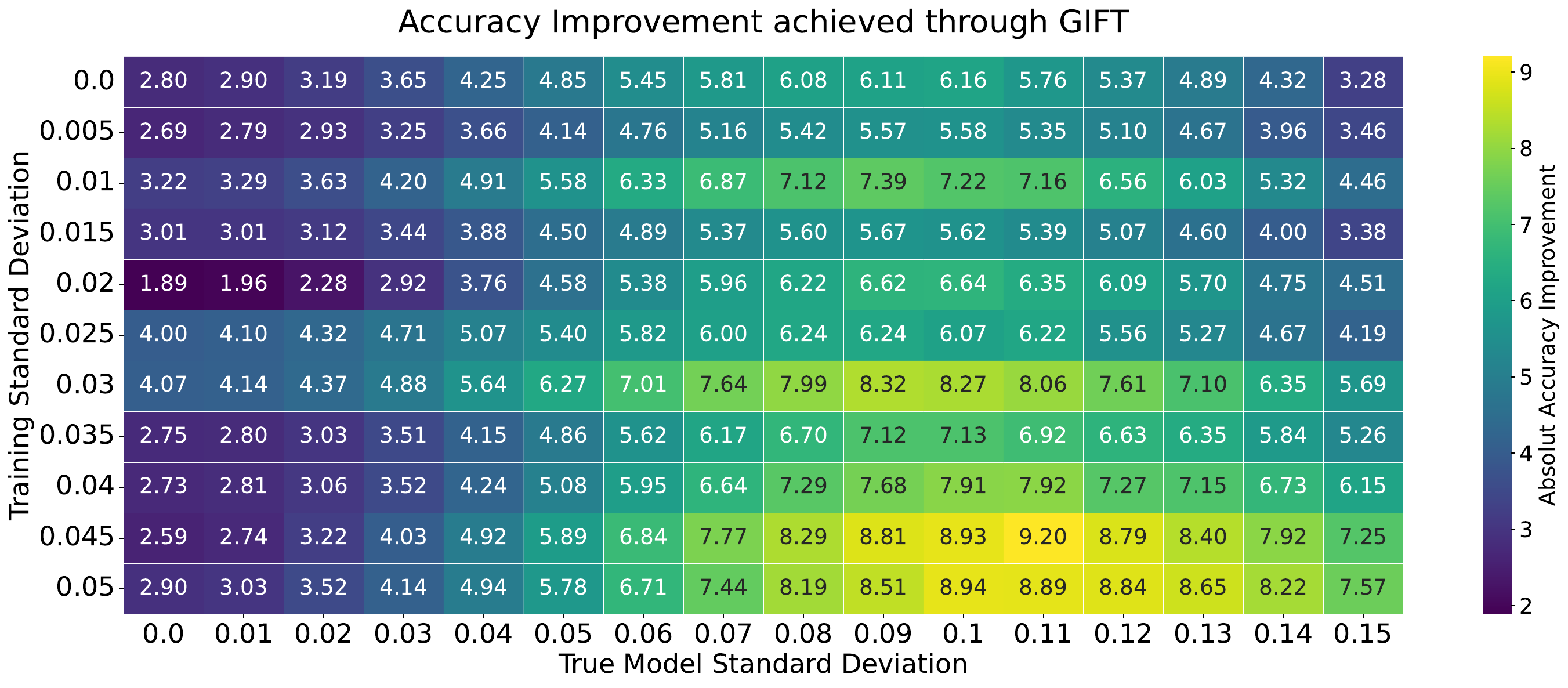}
	\caption{Mean loss and accuracy improvements achieved with \GIFT (\Cref{alg:GIFT}) compared to baseline performance without fine-tuning.
		Results are shown for the deeper network under varying noise standard deviations.
	}
	\label{fig:abs_performance_gift}
\end{figure}

\end{document}